\documentclass{article}

\usepackage{arxiv}

\usepackage{hyperref}       

\usepackage{amsfonts}
\usepackage{amsmath}
\usepackage{amssymb}
\usepackage{amsthm}
\usepackage{algorithm}
\usepackage{algorithmic}

\usepackage{graphicx}
\usepackage{color}
\usepackage{subcaption}
\usepackage{mathtools}   
\usepackage[nosepfour]{numprint}		
\usepackage{booktabs}

\usepackage{csvsimple}

\usepackage{tikz,pgfplots}
\usetikzlibrary{plotmarks}
 \usepackage{pgf-pie}

\usepackage{siunitx}
\usepackage{listofitems} 
\definecolor{tuc}{RGB}{50,120,100} 
\definecolor{tucmath}{RGB}{161,11,112}

\tikzstyle{mynode}=[thick,draw=tuc,fill=tuc!20,circle, minimum size=22]

\pgfplotsset{compat=1.15}
\usetikzlibrary{arrows}

\usepackage{prettyref}
\usepackage{mathrsfs}
\usepackage[utf8]{inputenc}
\usepackage{floatflt} 
\usepackage{float} 
\usepackage{wrapfig}

\usepackage{chngcntr}
\counterwithin{figure}{section}
\counterwithin{table}{section}

\renewcommand{\vec}[1]{\ensuremath \mathbf{\boldsymbol{#1}}}

\newsavebox\dotbox
\sbox{\dotbox}{\(\displaystyle\bigodot\)}

\newcommand{\X}{\ensuremath{\mathcal{X}}}

\renewcommand{\phi}{\varphi}
\newcommand{\w}{\ensuremath{\vec \omega}}
\newcommand{\G}{\ensuremath{\mathcal{N}}} 
\newcommand{\x}{\ensuremath{\vec x}}

\newcommand{\eps}{\ensuremath{\varepsilon}}
\newcommand{\normm}[1]{{\left\vert\kern-0.25ex\left\vert\kern-0.25ex\left\vert #1 
    \right\vert\kern-0.25ex\right\vert\kern-0.25ex\right\vert}}

\newcommand{\N}{\ensuremath{\mathbb{N}}}

\newcommand{\F}{\ensuremath{\mathcal{F}}}
\newcommand{\M}{\ensuremath{\mathcal{M}}}
\newcommand{\I}{\ensuremath{\mathcal{I}}}

\newcommand{\Z}{\ensuremath{\mathbb{Z}}}

\newcommand{\R}{\ensuremath{\mathbb{R}}}
\newcommand{\C}{\ensuremath{\mathbb{C}}}
\newcommand{\E}{\ensuremath{\mathbb{E}}}

\renewcommand{\O}{\ensuremath{\mathcal{O}}}
\renewcommand{\P}{\ensuremath{\mathbb{P}}}

\newcommand{\abs}[1]{\ensuremath{\left\vert#1\right\vert}}
\newcommand{\dx}{\mathrm{d}}
\newcommand{\e}{\mathrm{e}}
\newcommand{\im}{\mathrm{i}}

\newcommand{\mix}{\mathrm{mix}}
\DeclareMathOperator*{\diag}{diag}
\DeclareMathOperator*{\supp}{supp}

\DeclareMathOperator*{\argmin}{argmin}

\renewcommand{\d}{\, \mathrm{d}}

\newcommand{\norm}[1]{\left\lVert {#1} \right\rVert}    
\newcommand{\TT}{\mathcal T}				

\newcount\colveccount
\newcommand*\colvec[1]{
        \global\colveccount#1
        \begin{pmatrix}
        \colvecnext
}
\def\colvecnext#1{
        #1
        \global\advance\colveccount-1
        \ifnum\colveccount>0
                \\
                \expandafter\colvecnext
        \else
                \end{pmatrix}
        \fi
}

\newtheorem{theorem}{Theorem}[section]
\newtheorem{lemma}[theorem]{Lemma}

\newtheorem{corollary}[theorem]{Corollary}
\newtheorem{proposition}[theorem]{Proposition}
\newtheorem{problem}[theorem]{Problem}

\theoremstyle{definition}
\newtheorem{definition}[theorem]{Definition}
\newtheorem{example}[theorem]{Example}
\newtheorem{remark}[theorem]{Remark}

\newenvironment{Theorem}[1][noisnotdefined]{ \ifthenelse{\equal{#1}{noisnotdefined}}{\begin{theorem}}{\begin{theorem}[#1]}\normalfont\slshape}{\end{theorem}}
\newenvironment{Lemma}{ \begin{lemma}\normalfont\slshape}{\end{lemma}}
\newenvironment{Remark}[1][noisnotdefined]{ \ifthenelse{\equal{#1}{noisnotdefined}}{\begin{remark}}{\begin{remark}[#1]}\normalfont\rmfamily}{\bend\end{remark}}
\newenvironment{Example}{ \begin{example}\normalfont\rmfamily}{\bend\end{example}}
\newenvironment{Definition}{ \begin{definition}\normalfont\slshape}{\end{definition}}
\newenvironment{Corollary}{ \begin{corollary}\normalfont\rmfamily}{\end{corollary}}

\numberwithin{equation}{section}

\newcommand{\bend}{\hspace*{0ex} \hfill \hbox{\vrule height
    1.5ex\vbox{\hrule width 1.4ex \vskip 1.4ex\hrule  width 1.4ex}\vrule
    height 1.5ex}}

\long\def\symbolfootnote[#1]#2{\begingroup
\def\thefootnote{\fnsymbol{footnote}}\footnote[#1]{#2}\endgroup}

\newcounter{todocounter}
\newcommand{\todo}[2][noisnotdefined]{
 \marginpar{\fcolorbox{black}{yellow}{\footnotesize\textbf{todo}}
 \ifthenelse{\equal{#1}{noisnotdefined}}{}{\textcolor{black}{\newline\tiny #1}}}
 \textbf{\ifthenelse{\equal{#2}{.}}
   {\fcolorbox{red}{white}{\textcolor{red}{$\maltese$}}}{{\textcolor{red}{#2}}}}
 \refstepcounter{todocounter}}

\newrefformat{def}{Definition \ref{#1}}
\newrefformat{sub}{Subsection \ref{#1}}
\newrefformat{prop}{Proposition \ref{#1}}
\newrefformat{fig}{Figure \ref{#1}}
\newrefformat{rem}{Remark \ref{#1}}
\newrefformat{cor}{Corollary \ref{#1}}

\title{ANOVA-boosting for Random Fourier Features}

\date{June 2025}	

\author{ Daniel Potts \\
	Faculty of Mathematics\\
	Chemnitz University of Technology\\
	D-09107 Chemnitz, Germany \\
	\texttt{daniel.potts@math.tu-chemnitz.de} \\
	\And
	Laura Weidensager \\
	Faculty of Mathematics\\
	Chemnitz University of Technology\\
	D-09107 Chemnitz, Germany \\
	\texttt{laura.weidensager@math.tu-chemnitz.de} \\
}

\renewcommand{\shorttitle}{ANOVA-Boosting for random Fourier Features}

\begin{document}
\maketitle

\begin{abstract}
We propose two algorithms for boosting random Fourier feature models for approximating high-dimensional
functions. These methods utilize the classical and generalized analysis of variance (ANOVA) decomposition
to learn low-order functions, where there are few interactions between the variables. Our algorithms are 
able to find an index set of important input variables and variable interactions reliably.\par
Furthermore, we generalize already existing random Fourier feature models to an ANOVA setting, where terms of different order can be used.
Our algorithms have the advantage of being interpretable, meaning that the influence of every input variable is known in the learned model, even for dependent input variables.  
We provide theoretical as well as numerical results that our algorithms perform well for sensitivity analysis. The ANOVA-boosting step reduces the approximation error of existing methods significantly.
\end{abstract}


\section{Introduction}
Developing predictive models based on empirical data is a current field of research with diverse applications. 
The continuous growth in data collection leads to complex datasets, necessitating the handling of regression or 
classification tasks in high-dimensional spaces. Traditional machine learning techniques such as support vector 
machines, neural networks, and decision trees are commonly used to address these challenges. However, a crucial 
concern alongside prediction accuracy is the interpretability of these models, which is essential for understanding 
the underlying reasoning behind predictions. \par

Many current approaches, although effective with smaller or moderate number of input variables, stop working when confronted with high-dimensional challenges. The main problem to practical computability is often related to high dimension of the multivariate integration or interpolation
problem, known as the curse of dimensionality. A well-known foundation of a dimensional decomposition is the analysis of variance (ANOVA) decomposition, first presented by Hoeffding in the 1940s. Since then, the ANOVA decomposition has been studied a lot in the literature, see for example~\cite{CaMoOw97, RaAl99, LiOw06, KuSlWaWo09, Holtz11, Gu13, Owen23}. \par
However, the
classical ANOVA decomposition is only available for
independent, product-type probability measures of the input density. 
In practice, there could be notable correlations or dependencies among input variables. Therefore, the classical decomposition must be generalized for an arbitrary,
non product type probability measure. Achieving this will require modifying the original
orthogonality conditions. Indeed, inspired by Stone~\cite{St94} and employing a set
of weakened annihilating conditions, Hooker~\cite{Ho07} provided an existential proof of a
unique ANOVA decomposition for dependent variables, referred to as the generalized ANOVA decomposition in this paper, subject to a mild restriction on the probability measure. Afterwards, different approaches for calculating the component functions are studied for example in~\cite{LiRa12, Ra142}. \par

The methodology presented in this paper offers an alternative to traditional machine learning methods by proposing an initial ANOVA-boosting step for random feature methods, but also provides a natural means to assess the importance and influence of attributes on the predicted outcomes. The generalization of the classical ANOVA decomposition, e.g. in~\cite{Ho07, Ra142} forms the basis for our algorithm, which can be applied to possibly dependent input variables. We calculate an approximation by a least squares regression, which  penalizes the non-orthogonality between ANOVA terms. \par

Consider the following standard supervised learning setup. Let $\X=\{\vec x_1,\ldots,\vec x_M\}\subset\R^d$ be a set of discrete samples. We consider the problem of reconstructing a multivariate function $f\colon\R^d \rightarrow \C$ from discrete function samples on the set of nodes $\X$, which are sampled from the density $\mu\colon \R^d\rightarrow \R_{+}$.
We study the scattered-data problem, i.e.~we have given as labels the (possibly noisy) function values $\vec f = (f(\vec x) + \epsilon_{\x} )_{\vec x\in \X}$.
In contrast, in~\cite{RiWe24} the authors give explicit advice on the location of good sampling points for approximating high-dimensional functions as finite sums of lower-dimensional functions.\par

It is natural to express the model output $f(\x)$ as a finite hierarchical expansion in terms of the input variables,
\begin{equation}\label{eq:decomp}
f(\x) = f_{\varnothing} + \sum_{i=1}^d f_i(x_i) + \sum_{1\leq i<j\leq d} f_{\{i,j\}} (x_i,x_j)+\ldots + f_{\{1,\ldots,d\}}(x_1,x_2,\ldots, x_d), 
\end{equation}
where the zero-th order component function $f_\varnothing$ is a constant representing the mean
of $f (\x)$, the first order component function $f_i (x_i )$ gives the independent
contribution to $f (\x)$ by the $i$-th input variable acting alone, the second order component
function $ f_{\{i,j\}} (x_i,x_j)$ gives the pair cooperative contribution to $f (\x)$ by the input
variables $x_i$ and $x_j$, etc. The last term $f_{\{1,\ldots,d\}}(x_1,x_2,\ldots, x_d)$ contains any residual
$d$-th order cooperative contribution of all the input variables. The classical ANOVA decomposition, \cite{CaMoOw97, LiOw06,Holtz11}, of a function is a tool for capturing high-dimensional behaviour by demanding orthogonality with respect to the measure $\mu$ between the terms in~\eqref{eq:decomp} and weak annihilating conditions for functions $f\in L_2(\R^d,\mu)$. \par

In many settings functions may arise naturally as sums of functions, each with a limited variable interaction. Such
low-order structure may also be used to reduce the curse of dimensionality, \cite{CaMoOw97, Wu2011, SchmischkeDiss}. Another approach to capture low-dimensional structures by Gaussian mixtures was done in~\cite{HeBaSt22}.
In this regard, two classes of problems arise: either all of the
input variables $\x = (x_1, x_2,\ldots, x_d)$ are independent or at least some portion of the
variables in $\x$ are correlated. Standard formulations of the ANOVA deal with the case
of independent variables. We, on the other hand, use the extension of~\cite{LiRa12} to also treat correlated
variables.\par

Kernel-based approaches have been extensively used in high-dimensional function approximation since
they often perform well in practice. The random feature model \cite{RaRe07} is a popular technique for approximating
the kernel (and thus the minimizer of kernel regression problems) using a randomized
basis that can avoid the cost of full kernel methods. An alternative perspective to view the random feature model is as a two-layer network
with a randomized but fixed single hidden layer, \cite{RaRe07,LiXiYuSu22}. The random feature model takes the form
$$f^\#(\x) = \sum_{k=1}^N a_{k}\e^{\im \langle \w_k,\x\rangle} =\vec a^\top \e^{\im \langle \vec W,\vec x\rangle}, \quad \w_\vec k \in \R^d,$$ 
where $\x \in \R^d$ is the input data, $\vec W \in \R^{d\times N}$ is a random weight matrix, and $\vec a\in \C^{N}$ is the final weight layer. The columns of the matrix $\vec W $ are independent and identically distributed (i.i.d.) random variables generated by the (user defined) probability density function $\rho(\w)$.
We construct the feature matrix 
\begin{equation}
\vec A=(\e^{\im\langle \w_k,\x\rangle})_{k\in \{1,\ldots, N\},\vec x\in \X}.
\end{equation}
Given the collection of $M=|\X|$ measurements, $\vec f = (f(\vec x) + \epsilon_{\x} )_{\vec x\in \X}$, the random feature regression problem becomes training $\vec a$ by optimizing
$$\min_{\vec a\in \C^N}\norm{\vec f- \vec A\vec a}_2^2+\mathcal R(\vec a)$$
with some penalty function $\mathcal R\colon \C^N\rightarrow \R$. The most common choice for $\mathcal R$ is the ridge penalty
$\mathcal R(\vec a) = \lambda \norm{\vec a}_2^2$, which leads to the random feature ridge regression problem \cite{RuRo17,LiTo19,MeMi22}. We will use another penalty function $\mathcal R$, which incorporates the ANOVA decomposition of the function $f$. 
We propose a random Fourier feature-framework where we explain the ANOVA decomposition within the random Fourier feature (RFF) structure. This specification allows us to associate variances and covariances to input variables, leading to interpretability that is not
present in existing RFF models. Our algorithm aims to boost existing RFF algorithms like SHRIMP~\cite{Xie22} (uses iterative magnitude pruning to select features) or HARFE~\cite{HARFE23} (uses hard threshold pursuit to select features) by introducing a first approximation step, which is demonstrated by numerical examples. \par

This paper is organized as follows. In Section~\ref{sec:ANOVA_independent} we introduce the well-known ANOVA decomposition for independent input variables and relate this decomposition to the Fourier transform of the function $f$. Furthermore, we introduce functions of lower order and show that they occur naturally in function spaces of mixed smoothness. In Section~\ref{sec:dep_input} we generalize the ANOVA decomposition to possibly dependent input variables. We summarize the idea of random Fourier feature algorithms in Section~\ref{sec:RFF} and apply them to the ANOVA setting. The resulting boosting algorithms are summarized in Section~\ref{sec:sensitivity}, where we show how to do sensitivity analysis on an approximation with random Fourier features. The theoretical analysis in Section~\ref{sec:theory} generalizes the theory in~\cite{Ha23,Xie22} for random Fourier features and finally in Section~\ref{sec:num} we show with numerical examples the power of our boosting algorithms. 
 
\subsubsection*{Our main contributions are as follows:}
\begin{itemize}
	\item We propose an ANOVA-boosting, which extends and further develops the sparse random feature approximation from~\cite{Ha23,Xie22, HARFE23} to arbitrary index sets $U\subset \mathcal P([d])$ and give a new connection between the Fourier transform of a function and the ANOVA terms. This leads to random features, which are adapted to the function. 
	\item Generalizing the theory of sparse random Fourier features: In many cases, for example in the target case of functions of low effective dimension, the Fourier transform only exists in distributional sense. The norm on which the existing literature is based on, contains a maximum norm of the Fourier transform, which has to be generalized to the setting of tempered distributions.
	\item Introducing and analysing a first approximation step which calculates the important ANOVA terms for independent or dependent input variables.
	\item We improve the interpretability of previous random feature models by reducing the importance of variables that are only correlated with other variables and do not influence the function.
\end{itemize}

We will distinguish between the two cases where the input variables are independent or dependent. In the first case, we will use the classical ANOVA decomposition, whereas in the latter case we have to generalize this decomposition.

\subsubsection*{Related work}
\begin{itemize}
	\item The authors from~\cite{MaYa20} propose the notion of neural decomposition, which integrates the classical ANOVA and
deep neural networks for dimensionality reduction and
variance decomposition. Similar to our approach for dependent input variables, they show that identifiability for independent input variables can be achieved by
training models subject to constraints on the marginal properties of the decoder networks.
\item The D-MORPH algorithm~\cite{LiRa12} uses orthogonal basis with respect to sampling density $\mu$, which requires knowledge about the sampling density. Our approach is applicable independent of the sampling density $\mu$. Furthermore, instead of calculating the solution of the minimization problem directly by using an SVD, we solve the problem by an iterative algorithm. This work was followed, among other, by~\cite{BoLiBaPlRa22}, where the procedure was generalized to sampling from mixture densities, but also in this case an orthogonal basis is necessary. \\
In~\cite{Ra142} the generalized ANOVA decomposition is constructed by a constructive method by employing multivariate orthogonal polynomials as bases and calculating the expansion
coefficients involved from the solution of linear algebraic equations.
\item The authors in~\cite{ChGa12} also study indices measuring the sensitivity of the
output with respect to dependent input variables, but they are restricted to independent pairs of dependent input variables.
\item We generalize the approximation with random Fourier features, for already existing algorithms see for example~\cite{Ha23, Xie22, HARFE23}. See also~\cite{LiXiYuSu22} for a nice overview of random features for kernel approximation. 
\end{itemize}

\subsubsection*{Definitions and Notation}
In this paper we denote by $[d]$ the set $\{1,\ldots,d\}$ and its power set by $\mathcal{P}([d])$. The $d$-dimensional input variable of the function $f$ is $\vec x$, where we denote the subset-vector by $\vec x_{\vec u}=(x_i)_{i\in \vec u}$ 
for a subset $\vec u\subseteq [d]$. The complement of those subsets is always with respect to $[d]$, i.e., $\vec u^c=[d]\backslash \vec u$.
For an index set $\vec u\subseteq [d]$ we define $|\vec u|$ as the number of elements in $\vec u$.
Define the Fourier transform by
\begin{equation}\label{eq:Fourier_R}
\hat{f}(\w):=\int_{\R^d}f(\vec x)\,\e^{-\im\langle\w,\vec x\rangle}\d \vec x \quad \text{ for } \vec w\in \R^d.
\end{equation}
If $f\in L_2(\R^d)$ with $\hat{f}\in L_1(\R^d)$, the Fourier inversion formula
\begin{equation}\label{eq:fourier_inv}
f(\vec x) = \frac{1}{(2\pi)^d}\int_{\R^d}\hat{f}(\vec \w)\, \e^{\im \langle \w,\vec x\rangle} \d \w
\end{equation}
holds true for almost all $\vec x\in \R^d$.
For $f\notin L_2(\R^d)$ the Fourier transform $\hat f$ is defined only in distributional sense. Let $\mathcal S(\R^d)$ be the Schwartz space of rapidly decreasing functions on $\R^d$. 
A \textit{tempered distribution} is a continous linear functional on $\mathcal S(\R^d)$. If a distribution $T$ arises from a function in the sense that 
$$\langle T_f ,\phi\rangle = \int_{\R^d} f(\x) \phi(\x) \dx \x, \qquad \phi \in \mathcal S(\R^d),$$
then we speak about a \textit{regular distribution}. %
The Dirac distribution $\delta$ is defined by $\langle \delta , \phi\rangle = \phi(\vec 0)$ for all $\phi\in \mathcal S(\R^d)$, which is not regular.
The Fourier transform $\hat T$ of a tempered distribution $T$ is a linear functional on the Schwartz space and defined by,
$$\langle \hat T, \phi \rangle = \langle T, \hat \phi\rangle.$$
A density function or probability density function is a positive function, such that $\mathbb P(\x \in [\vec a,\vec b]) = \int_{\vec a}^{\vec b} \mu(\x)\dx \x$. This function is normalized, such that $\int_{\R^d}\mu(\x) \dx \x = 1$. We will denote this function by $\mu(\x)$ for the sampling distribution and similarly by $\rho(\w)$ for the feature distribution. Furthermore, we denote for $\gamma>0$ some frequently used densities by
\begin{align*}
\mu_{\G}(\x) = \frac{1}{(2\pi\gamma^2)^{d/2}} \e^{-\frac{\norm{\x}^2}{2\gamma}}&\quad \text{ Gaussian},\\
\mu_{\mathcal C}(\x)	 = \prod_{i = 1}^d \frac{1}{\pi \sigma (1+x_{i}^2/\gamma^2)}	&\quad \text{ Cauchy}.
\end{align*}
Let $s>0$. Then we define Sobolev spaces of dominating mixed smoothness by
\begin{equation}\label{eq:Hsmix_R}
H^s_{\mix}(\R^d)\coloneqq\left\{f\colon\R^d\to \C\mid \norm{f}_{H^s_{\mix}(\R^d)}<\infty\right\},
\end{equation}
where the norm is defined by
\begin{equation*}
\norm{f}_{H^s_{\mix}(\R^d)}^2=\int_{\R^d} |\hat{f}(\w) |^2\prod_{i=1}^d(1+|\omega_i|^2)^{s} \d \vec \omega.
\end{equation*}

\section{The ANOVA decomposition for independent input variables}\label{sec:ANOVA_independent}
In this section we study the case of independent input variables, which coincides with the density $\mu$ having tensor product structure.  
For periodic functions there is a connection between the Fourier coefficients of the ANOVA terms, which is used to construct approximation algorithms for high-dimensional functions with low effective dimension in an efficient and fast manner, see~\cite{PoSc19a}. This was the motivation to study the more general setting, seeking a connection between the Fourier transform $\hat f$ and the ANOVA terms, which we will do in the following.\par

The curse of dimensionality comes into play when analysing data in high-dimensional spaces. A frequently used concept is the following, \cite{CaMoOw97, LiOw06,Holtz11}. See also~\cite[Chapter 8.4]{Mo22} or \cite[Appendix]{Owen23} for a general introduction to functional decompositions. 
\begin{definition}\label{def:anova-terms}
Let $f$ be in $L_2(\R^d,\mu)$. For a tensor product measure 
\begin{equation}\label{eq:mu_prod}
\mu(\vec x) = \prod_{i=1}^d \mu_i(\vec x_i)
\end{equation}
and for a subset $\vec u\subseteq[d]$ we define the \textbf{ANOVA (Analysis of variance) terms} recursively by 
\begin{align}\label{eq:anova-terms}
f_{\varnothing} &= \int_{\R^d} f(\x) \mu(\x) \dx \x \notag\\
f_{\vec u}(\vec x_{\vec u}) &=\int_{\R^{d-|\vec u|}}f(\vec x) \mu_{\vec u^c}(\vec x_{\vec u^c})\d \vec x_{\vec u^c}-\sum_{\vec v\subset \vec u}f_{\vec v}(\vec x_{\vec v}).
\end{align}
The\textbf{ ANOVA decomposition} with respect to $\mu$ of a function ${f\colon \R^d\to \C}$ is then given by
\begin{equation}\label{eq:anova-decomp}
f(\vec x)=f_{\varnothing}+\sum_{i=1}^d f_{\{i\}}(x_i)+\sum_{i\neq j=1}^d f_{\{i,j\}}(x_i,x_j)+\cdots+f_{[d]}(\vec x)=\sum_{\vec u\subseteq[d]}f_{\vec u}(\vec x_{\vec u}).
\end{equation}
The terms~\eqref{eq:anova-terms} are the unique decomposition~\eqref{eq:anova-decomp}, 
such that 
\begin{align}
\langle f_{\vec u},f_{\vec v} \rangle_\mu \coloneqq \int_{\R^d} f_{\vec u}(\x_{\vec u})f_{\vec v}(\x_{\vec u})\mu(\x) \dx \x &= 0 \quad \text{ for } \vec v\neq \vec u\subseteq[d],\label{eq:orth}\\
\int_{\R} f_{\vec u}(\x_{\vec u}) \mu_j(x_j)\dx x_j &= 0 \quad \text{ for } j\in  \vec u. \label{eq:zero_mean}
\end{align}
\end{definition} 
Note that in general, $f_{\vec u}\notin L_2(\R^{|\vec u|})$, but $f_{\vec u}\in L_2(\R^{|\vec u|},\mu_{\vec u})$. Furthermore, every density has the property that $\mu\in L_1(\R^d)$, which implies that $\mu_i \in L_1(\R)$, such that the one-dimensional Fourier transforms $\hat{\mu}_i$ exists.
For a better readability of the following proofs, we introduce the notation
\begin{equation}\label{eq:def_E}
E(\x,\w,\mu, \vec u) \coloneqq \prod_{i\in \vec u}\left(\e^{\im \omega_i x_i } - \hat{\mu}_i( -\omega_i) \right)\prod_{i\in \vec u^c}\hat{\mu}_i( -\omega_i).
\end{equation} 
In terms of the Fourier transform, the ANOVA terms~\eqref{eq:anova-terms} can then be described by the following.

\begin{lemma}\label{lem:ANOVA-terms_f}
Let the sampling distribution $\mu$ have a product structure~\eqref{eq:mu_prod}. Then the ANOVA decomposition \eqref{eq:anova-terms} for a function $f\in L_2(\R^d,\mu) \cap L_2(\R^d)$, is given by
\begin{equation}\label{eq:anova_hat}
f_{\vec u}(\vec x_{\vec u}) 
= \frac{1}{(2\pi)^d}\int_{\R^d} \hat f(\w)E(\x,\w,\mu, \vec u)   \d \w.
\end{equation}

\end{lemma}
\begin{proof}
The Fourier transform of the tensor product density $\mu$ can be decomposed as 
\begin{equation*}
\hat{\mu} (\w) = \prod_{i\in [d]} \hat{\mu}_i(\omega_i).
\end{equation*} 
We prove~\eqref{eq:anova_hat} inductively over $|\vec u|$. First, observe that
\begin{align*}
f_\varnothing &= \int_{\R^d} f(\vec x) \mu(\x) \d \vec x 
= \int_{\R^d}\frac{1}{(2\pi)^d}\int_{\R^d}\hat f(\w)\, \e^{\im \langle \w,\vec x\rangle} \d \w\,\mu(\x) \d \vec x\\
&=\frac{1}{(2\pi)^d}\int_{\R^d}\hat f(\w)\int_{\R^d}\, \e^{\im \langle \w,\vec x\rangle} \mu(\x) \d \x\d \w 
=\frac{1}{(2\pi)^d}\int_{\R^d}\hat f(\w) \hat{\mu}(-\w)\d \w \\
&= \frac{1}{(2\pi)^d}\int_{\R^d}\hat f(\w) E(\x,\w,\mu, \varnothing)\d \w,
\end{align*}
where $E(\x,\w,\mu, \varnothing) = \prod_{i\in [d]} \hat{\mu}_i(-\omega_i)$ does not depend on $\x$.
The induction step follows using Definition~\ref{def:anova-terms},
\begin{align*}
&f_{\vec u}(\vec x_{\vec u})
=\int_{\R^{d-|\vec u|}}\frac{1}{(2\pi)^d}\int_{\R^d}\hat f(\w)\, \e^{\im \langle \w,\vec x\rangle} \d \w \mu_{\vec u^c}(\vec x_{\vec u^c})\d \vec x_{\vec u^c}-\sum_{\vec v\subset \vec u}f_{\vec v}(\vec x_{\vec v})\\
&\quad=\frac{1}{(2\pi)^d}\int_{\R^d}\hat f(\w)\e^{\im \langle \w_{\vec u},\vec x_{\vec u}\rangle}\int_{\R^{d-|\vec u|}}\e^{\im \langle \w_{\vec u^c},\vec x_{\vec u^c}\rangle}  \mu_{\vec u^c}(\vec x_{\vec u^c})\d \vec x_{\vec u^c}\d \w
-\frac{1}{(2\pi)^d}\sum_{\vec v\subset \vec u}\int_{\R^d} \hat f(\w)E(\x,\w,\mu, \vec v)\d \w\\
&\quad=\frac{1}{(2\pi)^d}\int_{\R^d}\hat f(\w)\e^{\im \langle \w_{\vec u},\vec x_{\vec u}\rangle} \hat{\mu}_{\vec u^c}(-\w_{\vec u^c})\d \w-\frac{1}{(2\pi)^d}\sum_{\vec v\subset \vec u}\int_{\R^d} \hat f(\w) \prod_{i\in \vec v} \left(\e^{\im \omega_i x_i}- \hat{\mu}_i(-\omega_i)\right) \prod_{i \in \vec v^c} \hat{\mu}_i(-\omega_i) \d \w\\
&\quad=\frac{1}{(2\pi)^d}\int_{\R^d}\hat f(\w)\hat{\mu}_{\vec u^c}(-\w_{\vec u^c}) \left(\e^{\im \langle \w_{\vec u},\vec x_{\vec u}\rangle}-\sum_{\vec v\subset \vec u} \prod_{i\in \vec v} \left(\e^{\im \omega_i x_i}- \hat{\mu}_i(-\omega_i)\right) \prod_{i \in \vec u\backslash v} \hat{\mu}_i(-\omega_i)\right) \d \w\\
&\quad=\frac{1}{(2\pi)^d}\int_{\R^d} \hat f(\w) \prod_{i\in \vec u} \left(\e^{\im \omega_i x_i}- \hat{\mu}_i(-\omega_i)\right) \prod_{i \in \vec u^c} \hat{\mu}_i(-\omega_i) \d \w\\
&\quad=\frac{1}{(2\pi)^d}\int_{\R^d} \hat f(\w) E(\x,\w,\mu,\vec u) \d \w.
\end{align*}
This shows~\eqref{eq:anova_hat}. 
\end{proof}

\begin{Example}
Suppose 
\begin{equation}\label{eq:f_2d}
f\colon \R^2 \rightarrow \R \qquad f(\x) = g_1(x_1) \, g_2(x_2) =  \frac{|x_1|}{(1+x_1^2)^2} \cdot \max\left(1-\abs{x_2},0\right) \in H_{\mix}^{3/2-\epsilon}(\R^2).
\end{equation}
The functions $g_1(x_1)$ and $g_2(x_2)$ have a kink, such that the derivatives of order $2$ are not in $L_2(\R)$. Furthermore, the function $f$ is
of tensor product structure, which means that the ANOVA decomposition is
\begin{align*}
f_1(x_1) &= \left(g_1(x_1)- \overline{f_1}\right) \cdot \overline{f_2},  &\overline{f_1}& \coloneqq \int_{\R} g_1(x_1) \mu_1(x_1) \dx x_1,\\
f_2(x_2) &= \left(g_2(x_2)  - \overline{f_2}\right) \cdot \overline{f_1},&\overline{f_2}&\coloneqq\int_{\R} g_2(x_2) \mu_2(x_2) \dx x_2, \\
f_{1,2}(\x) &= \left(g_1(x_1) - \overline{f_1}\right)\cdot \left(g_2(x_2)  - \overline{f_2}\right),
\end{align*}
where the constants with respect to standard Gaussian samples $\mu(\x)  = \mu_{\G}(\x)= \tfrac{1}{\sqrt{2\pi}^d}\e^{-\norm{\x}^2/2} $ and uniform samples on $[-1,1]^2$ are summarized here:
\begin{center}
\begin{tabular}{  |c|c |c| } 
  \hline
  $\mu$ & Gaussian & Uniform \\ 
  \hline
  $f_{\varnothing}$ &$0.0792$ & $0.125$  \\
	 $\overline{f_1}$ &$0.2148$ & $0.25$  \\
	 $\overline{f_2}$ &$0.3687$ & $0.5$  \\
  \hline
\end{tabular}
\end{center}
The ANOVA decomposition is plotted in Figure~\ref{fig:ANOVA_terms} for these two cases.
This example shows that tensor product type function can be easily decomposed in the ANOVA terms independent of the sampling density $\mu$. Furthermore, the ANOVA decomposition depends on the sampling density $\mu$. For more complicated functions this can have much more influence. 
\end{Example}
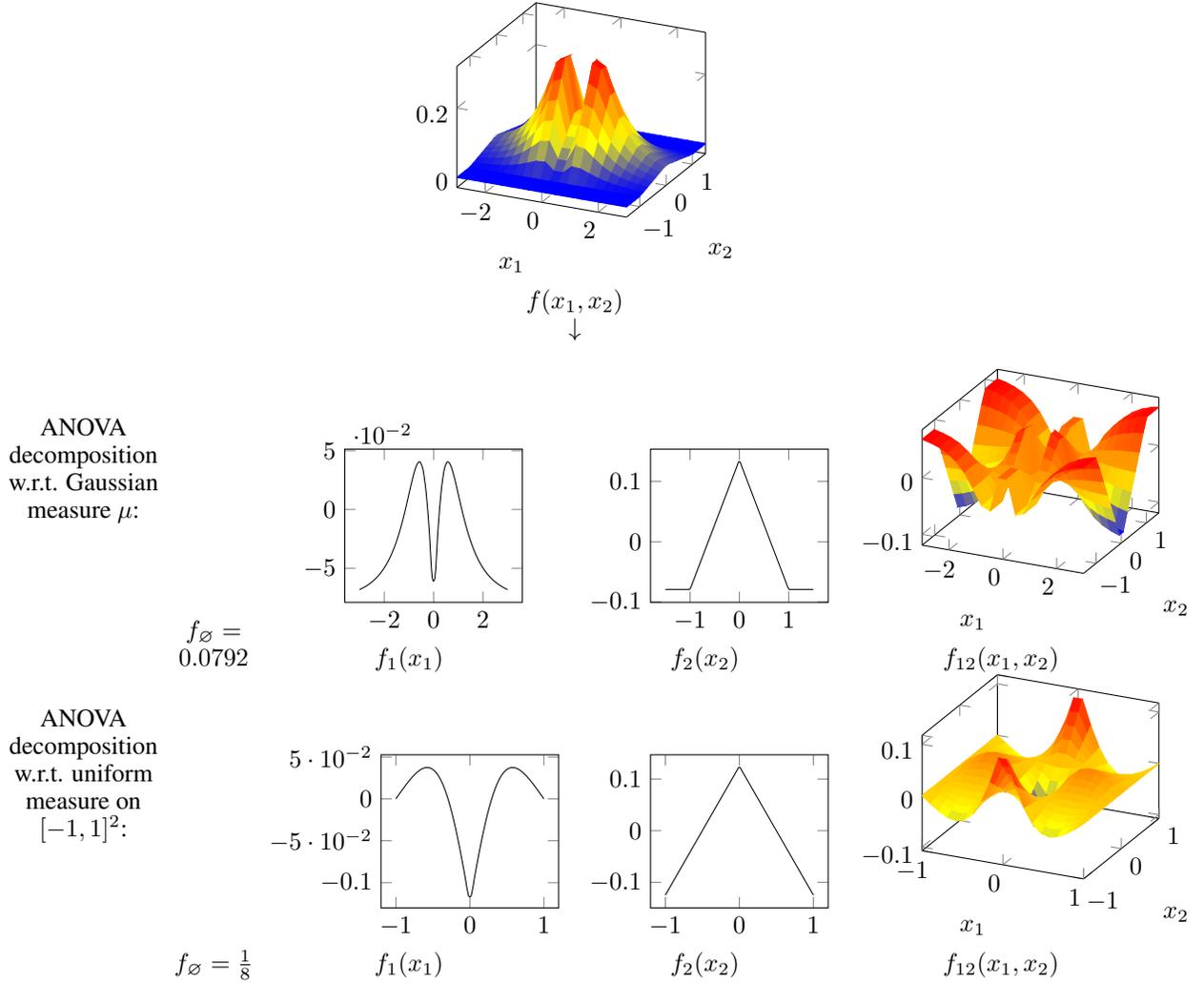
\begin{figure}[tb]
\centering
\begin{minipage}[c]{0.3\textwidth}
\begin{tikzpicture}
\begin{axis}[scale only axis,width = 0.7\textwidth,
,xlabel=$x_1$,ylabel=$x_2$]
\addplot3 [
    domain=-3:3,
    domain y = -1.5:1.5,
    samples = 20,
    samples y = 20,
    surf,shader=flat
    ] {(abs(x)/(1+x^2)^2) *max(0,(1-abs(y)))};
 
\end{axis}
 
\end{tikzpicture}
\centering
$f(x_1,x_2)$\\
$\downarrow$\\
\vspace*{10pt}
\end{minipage}\\
\begin{minipage}[c]{0.14\textwidth}
\centering
\vspace*{-150pt}
ANOVA decomposition w.r.t. Gaussian density $\mu$:
\end{minipage}%
\begin{minipage}[b]{0.08\textwidth}
\centering
$f_\varnothing = 0.0792$\\
\end{minipage}%
\begin{minipage}[b]{0.25\textwidth}
\begin{tikzpicture}
\begin{axis}[scale only axis,width = 0.6\textwidth]
\addplot [
    domain=-3:3,
    samples = 50,
		smooth
    ] {((abs(x)/(1+x^2)^2) - 0.2148)* 0.368746 };
 
\end{axis}
 
\end{tikzpicture}
\centering
$f_1(x_1)$
\end{minipage}%
\begin{minipage}[b]{0.25\textwidth}
\begin{tikzpicture}
\begin{axis}[scale only axis,width = 0.6\textwidth]
\addplot [
    domain=-1.5:1.5,
    samples = 100,
		smooth
    ] { 0.2148* (max(0,(1-abs(x)))-0.368746) };
\end{axis}
\end{tikzpicture}
\centering
$f_2(x_2)$
\end{minipage}%
\begin{minipage}[b]{0.25\textwidth}
\begin{tikzpicture}
\begin{axis}[scale only axis,width = 0.8\textwidth,
,xlabel=$x_1$,ylabel=$x_2$]
\addplot3 [
    domain=-3:3,
    domain y = -1.5:1.5,
    samples = 20,
    samples y = 20,
    surf,shader=flat
    ] {((abs(x)/(1+x^2)^2)-0.2148) * (max(0,(1-abs(y)))-0.368746)};
 
\end{axis}
\end{tikzpicture}
\centering
$f_{12}(x_1,x_2)$
\end{minipage}\\
\begin{minipage}[c]{0.14\textwidth}
\centering
\vspace*{-150pt}
ANOVA decomposition w.r.t. uniform density on $[-1,1]^2$:
\end{minipage}%
\begin{minipage}[b]{0.08\textwidth}
\centering
$f_\varnothing = \tfrac 18$
\end{minipage}%
\begin{minipage}[b]{0.25\textwidth}
\begin{tikzpicture}
\begin{axis}[scale only axis,width = 0.6\textwidth]
\addplot [
    domain=-1:1,
    samples = 50,
		smooth
    ] {((abs(x)/(1+x^2)^2) - 1/4)* 0.5 };
\end{axis}
\end{tikzpicture}
\centering
$f_1(x_1)$
\end{minipage}%
\begin{minipage}[b]{0.25\textwidth}
\begin{tikzpicture}
\begin{axis}[scale only axis,width = 0.6\textwidth]
\addplot [
    domain=-1:1,
    samples = 100,
		smooth
    ] {1/4* (max(0,(1-abs(x)))-0.5) };
 
\end{axis}
\end{tikzpicture}
\centering
$f_2(x_2)$
\end{minipage}%
\begin{minipage}[b]{0.25\textwidth}
\begin{tikzpicture}
\begin{axis}[scale only axis,width = 0.8\textwidth,
,xlabel=$x_1$,ylabel=$x_2$]
\addplot3 [
    domain=-1:1,
    domain y = -1:1,
    samples = 20,
    samples y = 20,
    surf,shader=flat
    ] {((abs(x)/(1+x^2)^2)-1/4) * (max(0,(1-abs(y)))-0.5)};
 
\end{axis}
\end{tikzpicture}\\
\centering
$f_{12}(x_1,x_2)$

\end{minipage}\\

\caption{The ANOVA decomposition of the function~\eqref{eq:f_2d}. }
\label{fig:ANOVA_terms}
\end{figure}

\subsubsection*{Sensitivity analysis}
The ANOVA decomposition is the basis for sensitivity analysis, which is the study of how the uncertainty in the output of a mathematical model can be divided and allocated to different sources of uncertainty in its inputs. A measure describing the proportion of how much the variables $\x_{\vec u} $ contribute to the variance of the function $f$ itself are the variances
\begin{equation}\label{eq:sigma}
\sigma^2(f_{\vec u}) = \int_{\R^{|\vec u|}}|f_{\vec u}(\x_{\vec u})|^2 \mu_{\vec u}(\x_{\vec u}) \dx \x_{\vec u},
\end{equation}
where $\sum_{\vec u\subseteq [d]} \sigma^2(f_{\vec u}) =  \sigma^2(f)$. This provides an explanation for the name \textbf{analysis of variance} decomposition to this function decomposition. 
Sobol indices, first introduced by~\cite{So95}, are an often used tool for sensitivity analysis, namely
\begin{equation}\label{eq:sobol}
S_{\vec u} = \frac{\sigma^2(f_{\vec u})}{\sigma^2(f)}.
\end{equation}

\subsubsection*{Relation to the case of periodic functions}
The ANOVA decomposition in terms of the Fourier transform $\hat f$ investigated in Lemma~\ref{lem:ANOVA-terms_f} is a generalization of the periodic case. To see this, let $\mu_i(x_i) = \frac{1}{2\pi }\, 1_{[- \pi , \pi]}$, the density belonging to the uniform density on the torus.  
For the Fourier transform of this density we calculate
\begin{equation}\label{eq:mu_hat}
\hat{\mu}_i(\omega_i) = \frac{\sin(\pi \omega_i)}{\pi\omega_i}
=\begin{cases}
0 \text{ if } \omega_i \in \Z\backslash 0,\\
1 \text{ if } \omega_i= 0.
\end{cases}
\end{equation}
For periodic functions the Fourier coefficients are defined by
$$c_{\vec k}(f) = \int_{-\pi }^\pi f(\x) \e^{-\im \langle \vec k,\x\rangle}\d \vec x.$$
It is possible to extend the definition~\eqref{eq:Fourier_R} to include periodic functions by viewing them as tempered distributions. This makes it possible to see a connection between the Fourier series and the Fourier transform for periodic functions that have a convergent Fourier series. If $f$ is a periodic function with period $1$, that has convergent Fourier series, then:
$$\hat f(\omega) = \sum_{k \in \Z} c_k(f) \delta\left(\omega -k\right),$$
where $c_k(f)$ are the Fourier coefficients of $f$ and $\delta$ is the Dirac delta distribution. The corresponding multivariate case is
$$\hat f(\w) =  \sum_{\vec k \in \Z^d} c_{\vec k}(f) \delta\left(\w -\vec k\right),$$
Applying Lemma~\ref{lem:ANOVA-terms_f} to this setting, we have
\begin{align*}
f_{\vec u}(\vec x_{\vec u}) 
&= \frac{1}{(2\pi)^d}\int_{\R^d} \hat{f}(\w)E(\x,\w,\mu, \vec u)   \d \w
=   \frac{1}{(2\pi)^d}\sum_{\vec k \in \Z^d} c_{\vec k}(f) \delta\left(\w -\vec k\right) E(\x,\w,\mu, \vec u)   \\
&= \frac{1}{(2\pi)^d} \sum_{\vec k \in \Z^d} c_{\vec k}(f) E(\x,\vec k,\mu, \vec u)   
=   \frac{1}{(2\pi)^d}\sum_{\stackrel{\vec k \in \Z^d}{\supp \vec k =\vec u} } c_{\vec k}(f)\e^{\im \langle \vec k_{\vec u},\vec x_{\vec u}\rangle} .
\end{align*}
The last equality follows from~\eqref{eq:mu_hat} and the definition of the term $E$ in~\eqref{eq:def_E}. This connection was shown in~\cite{PoSc19a} and was starting point for efficient algorithms.

\subsection{Functions of low order}
Often, high-dimensional functions that arise from important physical systems are of low
order, meaning the function is dominated by a few terms, each depending on only a
subset of the input variables, say $q$ out of the $d$ variables where $q \ll d$. For that reason, we 
formalize the notion of low order functions by extending the definition
from~\cite{Ha23}.
\begin{definition}[Functions of low order]
Fix $d,q\in \N$ with $q\leq d$. A function $f\colon \R^d \rightarrow \C$ is an \textbf{order-$q$} function, if  
$$f(\x) =  \sum_{\vec u \in U_q } f_{\vec u}(\x_{\vec u}) \quad \text{ with } \quad U_q = \{\vec u\in [d] \mid |\vec u|\leq q\} .$$
\end{definition}
Low order functions arise naturally in the physical world and are used as a form of the reduced complexity
model for such systems. \par 
The following gives a bound for the variances of the ANOVA terms, which is guided by a decomposition of the frequency domain of $f$. The proof can be found in Appendix~\ref{sec:proofs}.
\begin{Lemma}\label{lem:sigma_bound}
Let $f\in L_2(\R^d)$ with $\hat f \in L_1(\R^d)$, then the variances of the ANOVA terms $f_{\vec u}$ defined in~\eqref{eq:anova_hat} are bounded by
$$\sigma^2(f_{\vec u}) \leq \frac{1}{(2\pi)^{2d}} \norm{\hat{\mu}}_{L_1(\R^d)}\int_{\R^d}|  \hat f(\w)|^2
| E(\vec 0,\w,\mu,\vec u)| \d \w. $$
\end{Lemma}

\begin{Example}
Consider again the standard Gaussian distribution $\mu = \mu_{\G}$, where
$\hat{\mu}_{\G}(\w) = \e^{-\frac{\norm{\w}^2}{2}}$.  
By the functions $ |E(\vec 0,\w,\mu,\vec u)|$ appearing in Lemma~\ref{lem:sigma_bound} we decompose the frequency domain $\R^d$ concerning the different ANOVA indices $\vec u\subseteq [d]$. We plot the two-dimensional example in Figure~\ref{fig:decomp_FT}. One can see that (this is in general the case, not only for this example) 
$$\lim_{k\rightarrow \infty}|E(\vec 0,k \w,\mu, \vec u)|=
\begin{cases}
1& \text{ if } \supp \w = \vec u,\\
0& \text{ otherwise }. 
\end{cases}$$
The decomposition is the analogue of the discrete decomposition of the Fourier series into ANOVA terms, see~\cite[Fig.1]{PoSc19a}. %
\begin{figure}[htb]%
\centering
\begin{minipage}[c]{0.24\textwidth}
\begin{tikzpicture}
\begin{axis}[scale only axis,width = 0.8\textwidth,
,xlabel=$\omega_1$,ylabel=$\omega_2$]
\addplot3 [
    domain=-5:5,
    domain y = -5:5,
    samples = 20,
    samples y = 20,
    surf,shader=flat
    ] { (exp(-x^2)) * (exp(-y^2))};
\end{axis}
\end{tikzpicture}\\
\centering
$\vec u=\varnothing$
\end{minipage}
\begin{minipage}[c]{0.24\textwidth}
\begin{tikzpicture}
\begin{axis}[scale only axis,width = 0.8\textwidth,
,xlabel=$\omega_1$,ylabel=$\omega_2$]
\addplot3 [
    domain=-5:5,
    domain y = -5:5,
    samples = 20,
    samples y = 20,
    surf,shader=flat
    ] { (1-exp(-x^2)) * (exp(-y^2))};
\end{axis}
\end{tikzpicture}\\
\centering
$\vec u=\{1\}$
\end{minipage}
\begin{minipage}[c]{0.24\textwidth}
\begin{tikzpicture}
\begin{axis}[scale only axis,width = 0.8\textwidth,
,xlabel=$\omega_1$,ylabel=$\omega_2$]
\addplot3 [
    domain=-5:5,
    domain y = -5:5,
    samples = 20,
    samples y = 20,
    surf,shader=flat
    ] { (exp(-x^2)) * (1-exp(-y^2))};
\end{axis}
\end{tikzpicture}\\
\centering
$\vec u=\{2\}$
\end{minipage}%
\begin{minipage}[c]{0.24\textwidth}
\begin{tikzpicture}
\begin{axis}[scale only axis,width = 0.8\textwidth,
,xlabel=$\omega_1$,ylabel=$\omega_2$]
\addplot3 [
    domain=-5:5,
    domain y = -5:5,
    samples = 20,
    samples y = 20,
    surf,shader=flat
    ] { (1-exp(-x^2)) * (1-exp(-y^2))};
\end{axis}
\end{tikzpicture}\\
\centering
$\vec u=\{1,2\}$
\end{minipage}
\caption{$2$-dimensional decomposition of the frequency domain into ANOVA terms for Gaussian distribution $\mu_{\G}$. Plotted are the functions $ |E(\vec 0,\w,\mu_{\G},\vec u)|$ for $\vec u\subseteq\{1,2\}$.}
\label{fig:decomp_FT}
\end{figure}
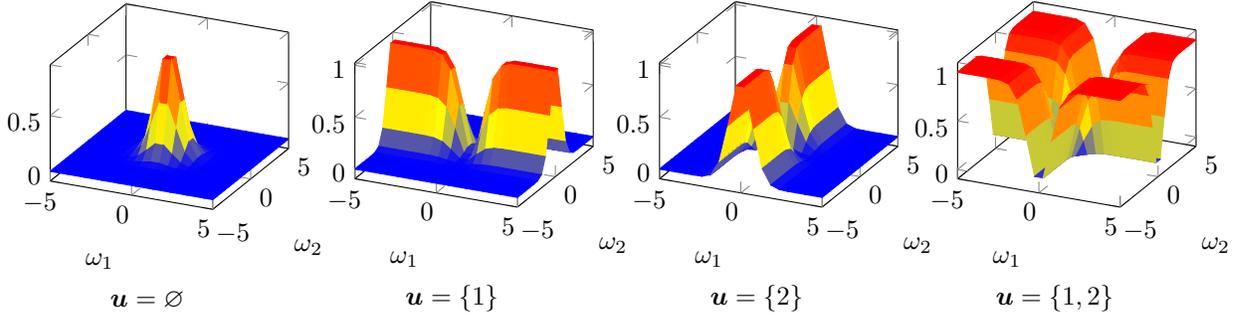%
 \end{Example}%
To study functions of low-order $q$, define 
$$\TT_qf \coloneqq \sum_{|\vec u|\leq q} f_{\vec u}.$$
In \cite[Corollary 2.32]{SchmischkeDiss} the author delivers error estimates for the truncation error $\norm{f-\TT_qf}_{L_2(\R^d,\mu)}$ for 
functions $f$ in function spaces with product and order-dependent weights, which builds on a transformation of periodic functions and an ANOVA decomposition based on Fourier coefficients on the torus. 
However, the estimates there are related to the smoothness of the transformed function on the torus, see also~\cite{LiPo22} for details of the idea of transformations from $\R^d$ to the torus and the transformation of the smoothness thereby. Since the transformation can destroy the smoothness, in the following theorem we give a bound on the truncation error $\norm{f-\TT_qf}_{L_2(\R^d,\mu)}$ relative to the norm $\norm{f}_{H^s_\mix(\R^d)}$. The proof can be found in Appendix~\ref{sec:proofs}.

\begin{Theorem}\label{thm:error_f-Tqf}
Let the measures $\mu_i$ be either symmetric and have positive Fourier transform or fulfill for fixed $s>\frac 12$ the mild condition
\begin{equation}\label{eq:condition}
\frac{|1-\hat{\mu_i}(-\omega_i)| + |\hat{\mu_i}(-\omega_i)| }{(1+|\omega_i|^2)^{s}}\leq 1
\end{equation}
for all $i\in [d]$ and all $\omega_i\in \R$. Define the constant
$$c_{\mu,s} = \max_{i\in[d]}\max_{\omega\in \R} \frac{1-\hat{\mu_i}(-\omega)}{(1+|\omega|^2)^{s}}.$$
Then for $f\in H^s_\mix(\R^d)$ the truncation error is bounded by
\begin{align*}
\norm{f-\TT_{q}f}_{L_2(\R^d,\mu)}^2 &\leq \frac{c_{\mu,s}^q}{(2\pi)^{2d}} \norm{\hat{\mu}}_{L_1(\R^d)}\norm{f}^2_{H^s_{\mix}(\R^d)}.
\end{align*}
Furthermore, if the Fourier transforms $\hat{\mu_i}$ are differentiable, we have
\begin{equation}\label{eq:constant2}
c_{\mu,s} = \max_{i\in [d]}\max_{\omega\in \R} \frac{-\hat{\mu}_i'(\omega)}{2s\omega\,\left(1+\omega^2\right)^{s-1}}.
\end{equation}
\end{Theorem}

We specifically point out that in the previous result the truncation error $\norm{f-\TT_{q}f}_{L_2(\R^d,\mu)}^2$ is bounded by a constant $ \O(2^{-q})$ and $\norm{f}_{H^s_\mix(\R^d)}$, which is determined by the smoothness of $f$. In general, this norm can increase with increasing $d$, but on the other hand the factor $(2\pi)^{-2d}\norm{\hat{\mu}}_{L_1(\R^d)}$ decays exponentially with increasing $d$. We want to point out, that our estimates are restricted to the case $\hat \mu\in L_1(\R^d)$, which is not the case for the uniform distribution on a cube. In the latter case one should use the the Fourier series instead of the Fourier transform to describe the ANOVA decomposition, see the periodic example in~\eqref{eq:mu_hat}.  
\begin{Example}
Let us have a look at Gaussian samples $\mu_i(x_i) =\mu_{\G}(x_i)= \frac{1}{\gamma \sqrt{2\pi}} \e^{-\tfrac{   x_i^2}{2\gamma^2 }}$ with variance $\gamma$, which means
$\hat{\mu}_i(\omega_i) = \e^{-\gamma^2\omega_i^2/2}$. Then
\begin{align*}
c_{\mu,s} &= \frac{\gamma\omega_i}{2s\omega_i (1+\omega_i^2)^{s-1}} \e^{-\gamma^2\omega_i^2/2} =\frac{\gamma}{2s(1+\omega_i^2)^{s-1}} \e^{-\gamma^2\omega_i^2/2}\leq \frac{\gamma}{2s},\\
\norm{\hat{\mu}_i}_{L_1(\R)} &= \frac{\sqrt{2\pi}}{\gamma}, 
\end{align*}
such that 
$$
\norm{f-\TT_{q}f}_{L_2(\R^d,\mu_\G)}^2 \leq (2\pi)^{-\tfrac 32d}(2s)^{-(q+1)}\norm{f}^2_{H^s_{\mix}(\R^d)}.
$$
Let us have a look at Cauchy distributed samples $\mu_i(x_i) =\mu_{\mathcal C}(x_i)= \frac{1}{\pi \gamma (1+x_i^2/\gamma^2) }$ with variance $\gamma$, which means
$\hat{\mu}_i(\omega_i) = \e^{-\gamma |\omega_i|}$. Then
\begin{align*}
c_{\mu,s} &= \sup_{\omega_i>0}\frac{\gamma\omega_i}{2s\omega_i (1+\omega_i^2)^{s-1}} \e^{-\gamma\omega_i} =\frac{\gamma}{2s(1+\omega_i^2)^{s-1}} \e^{-\gamma^2\omega_i^2/2}\leq \frac{\gamma}{2s},\\
\norm{\hat{\mu}_i}_{L_1(\R)} &= \frac{\sqrt{2}}{\gamma}, 
\end{align*}
such that 
$$
\norm{f-\TT_{q}f}_{L_2(\R^d,\mu_{\mathcal C})}^2 \leq 2^{-d}\pi^{-2d}(2s)^{-(q+1)}\norm{f}^2_{H^s_{\mix}(\R^d)}.
$$
\end{Example}

\section{The generalized ANOVA decomposition for correlated input variables}\label{sec:dep_input}
The main assumption of the ANOVA decomposition is that the input parameters $x_i, i=1,\ldots, d$, are independent.
This is unrealistic in many cases. Clearly, the correlation structure of random variables heavily influences the
composition of component functions as well as global sensitivity analysis.

However, when the dependence is present
among variables, the variance contribution of an individual variable
$x_i$ consists of not only the contribution resulting from the variable itself, but also contains the dependent contribution resulting from the
dependence between variable $x_i$ and other variables. So far, the literature \cite{ LiRa12, Ra142} discusses the variance
contributions with dependent variables, and makes a distinction
between the independent contribution and dependent contribution of the variables.\par

We now consider possibly dependent input variables $x_i$, i.e.~an arbitrary non-product type probability density function $\mu\colon \R^d \rightarrow \R $, that has 
marginal probability density functions 
\begin{equation}
\label{eq:mu_u}\mu_{\vec u}(\x_\vec u) \coloneqq \int_{\R^{d-|\vec u|}} \mu(\x)\dx \vec x_{\vec u^c}, \end{equation}
where $\varnothing \neq \vec u \subseteq [d] $. 

In the case for independent variables, the ANOVA decomposition~\eqref{eq:anova-decomp} is unique by demanding the condition~\eqref{eq:orth}. For dependent variables, it is in general not possible to find an orthogonal decomposition. First, there is a mild condition to the measure $\mu$ needed, to construct an ANOVA decomposition of the form~\eqref{eq:anova-decomp}:
Assume that for every $\vec u\subseteq[d]$ the support of $\mu_{\vec u}$ is \textbf{grid-closed} \cite{Ho07}. The grid closure implies that for any point $\x_{\vec u}\in \supp \mu_{\vec u} $ we can move in each coordinate direction and find another point in the support of $\mu_{\vec u}$. This is a mild regularity requirement, which is fulfilled by common probability distributions. The grid closure excludes only degenerated distributions like $\mu_{\{1,2\}} = \vec 1_{\{x_1 = x_2\}}$, where it anyway is not possible to distinguish between the input variables $x_1$ and $x_2$. 

Second, we have to replace the condition~\eqref{eq:orth} of the classical ANOVA decomposition in the setting for independent variables to a milder condition, \textbf{hierarchical orthogonality} condition, 
\begin{equation}\label{eq:sca_prod_gen}
\int_{\R^d} f_{\vec u}(\x_{\vec u}) f_{\vec v}(\x_{\vec v}) \mu(\x) \dx \x = 0 \quad \text{ for all } \vec v \subset \vec u, 
\end{equation}
see~\cite{Ho07, Ra142} for more details.
The \textbf{weak annihilating conditions} are 
\begin{equation}\label{eq:zero_mean_gen}
\int_{\R} f_{\vec u}(\x_{\vec u}) \mu_{\vec u}(\x_{\vec u}) \dx x_i = 0 \quad \text{ for all } i \in \vec u, \vec u \subseteq [d].
\end{equation}
Demanding these conditions, leads to a unique generalized ANOVA decomposition. Indeed, if \eqref{eq:zero_mean_gen} is fulfilled, also \eqref{eq:sca_prod_gen} is true, see~\cite{Ho07}.
Every square-integrable multivariate
function $f\in L_2(\R^d,\mu)$ admits a unique, finite, hierarchical expansion 
\begin{equation}\label{eq:anova_gen}
f(\x) = \sum_{\varnothing \neq \vec u \subseteq [d]} f_{\vec u}(\x_{\vec u}),
\end{equation}
referred to as the \textbf{generalized ANOVA decomposition}. The existence and uniqueness of the decomposition in~\eqref{eq:anova_gen} under mild conditions has been proven in~\cite{ChGa12, Ho07, St94}. That is, for the existence the support of every $\mu_{\vec u}$ has to be grid-closed and the uniqueness follows by demanding~\eqref{eq:zero_mean_gen}. Note that the generalized ANOVA decomposition matches the classical ANOVA decomposition~\eqref{eq:anova-decomp}, if the input variables are independent, i.e.~if the density $\mu$ is a tensor product. \par

There exist many methods for calculating the global sensitivity indices~\eqref{eq:sobol} of a function of independent variables. In contrast, only a few methods,
such as those presented in~\cite{ChGa12, Ho07, KuTaAn12, LiRa12, Ra142} are available for models with dependent or correlated
input. In all literature the Sobol indices~\eqref{eq:sobol} are generalized to the following.
\begin{Definition}\label{def:sobol_indices}
The Sobol indices for an ANOVA term $f_{\vec u}$ measuring the contribution of $\x_{\vec u}$ into the model, denoted by $S_{\vec u, \text{var}}, S_{\vec u, \text{cor}}$ and $S_{\vec u} $ are given by 
\begin{align*}
S_{\vec u, \text{var}} &= \frac{\sigma^2(f_{\vec u})}{\sigma^2(f)}\\
S_{\vec u, \text{cor}} &= \frac{\sum_{\stackrel{\varnothing \neq \vec v \subseteq [d]}{\vec v\cap \vec u\neq \varnothing, \vec v\not\subseteq \vec u}} \langle f_{\vec u},f_{\vec v}\rangle_{\mu}}{\sigma^2(f)} \\
S_{\vec u} &=  S_{\vec u, \text{var}} + S_{\vec u, \text{cor}}.
\end{align*}
The first two indices $S_{\vec u, \text{var}}$ and $S_{\vec u, \text{cor}}$ represent the normalized versions of the \textbf{variance contributions} and \textbf{covariance contributions} from $f_{\vec u}$ to $\sigma^2(f)$. The third index, $S_{\vec u}$, referred to as the \textbf{total global sensitivity index} is the sum of variance and covariance contributions. 
\end{Definition}
When the random variables are independent, the covariance contributions to the total sensitivity index $S_{\vec u, \text{cor}}$ vanish for all $\vec u\subseteq [d]$, leaving only one sensitivity index for the classical ANOVA decomposition.

\subsection{ANOVA decomposition of the frequency domain}\label{sec:tempered_distributions}
In case of periodic functions and independent input variables there is a connection between the ANOVA terms $f_{\vec u}$ and the Fourier coefficients~\cite{PoSc19a} or the wavelet coefficients~\cite{LiPoUl21}. This connection leads to efficient algorithms. In the following we study a similar connection for functions on $\R^d$ and the Fourier transform. 

Let $U\subseteq \mathcal P([d])$ be a subset of all ANOVA terms and let $f\in L_2(\R^d)\cap L_2(\R^d,\mu)$ have the generalized ANOVA decomposition $f(\x) = \sum_{\vec u \in U}f_{\vec u}(\x_{\vec u})$. The functions $f_{\vec u}$ depend only on the variables $\x_{\vec u}$, i.e.~for all $\phi \in \mathcal S(\R^{d})$,
\begin{align*}
\langle \hat T_f ,\phi\rangle  &= \langle T_f ,\hat \phi\rangle = \sum_{\vec u\in U} \langle f_{\vec u}(\x_{\vec u}),\hat \phi\rangle 
= \sum_{\vec u\in U} \int_{\R^d} f_{\vec u} (\x_{\vec u}) \hat \phi(\x)\d \x 
= \sum_{\vec u\in U} \int_{\R^d} f_{\vec u} (\x_{\vec u}) \int_{\R^d} \phi(\w) \d \w \d \x
\\
&= \sum_{\vec u\in U} \int_{\R^{|\vec u|}} f_{\vec u} (\x_{\vec u})   \int_{\R^d} \phi(\w) \int_{\R^{|\vec u^c|}} \e^{-\im \langle \x_{\vec u^c},\w_{\vec u^c}\rangle} \d \x_{\vec u^c} \, \e^{-\im \langle \x_{\vec u},\w_{\vec u}\rangle} \d \w \d \x_{\vec u}
\\
&= \sum_{\vec u\in U} \int_{\R^{|\vec u|}} f_{\vec u} (\x_{\vec u}) \int_{\R^d} \phi(\w) \delta(\w_{\vec u^c})\e^{-\im \langle \x_{\vec u},\w_{\vec u}\rangle} \d \w \d \x_{\vec u}\\
&= \sum_{\vec u\in U} \int_{\R^{|\vec u|}} f_{\vec u} (\x_{\vec u}) \int_{\R^d} \phi(\w_{\vec u},\vec 0) \e^{-\im \langle \x_{\vec u},\w_{\vec u}\rangle} \d \w_{\vec u}  \d \x_{\vec u} \\
&= \sum_{\vec u\in U} \langle f_{\vec u}(\x_{\vec u}),\hat \phi_{\vec u}\rangle  \qquad \text{ where } \phi_{\vec u}(\w_\vec u) \coloneqq \phi(\w_{\vec u},\vec 0) \in  \mathcal S(\R^{|\vec u|})\\
&= \sum_{\vec u\in U} \langle \hat T_{f_\vec u},\phi_{\vec u}\rangle.
\end{align*}%
Since $\phi\in S(\R^{d})$, the functions $\phi(\w_{\vec u},\vec 0) \in  \mathcal S(\R^{|\vec u|})$ are also in a Schwarz space of lower dimension and the notation $\phi(\w_{\vec u},\vec 0)$ means that we insert in the function $\phi$ zeros for all $\omega_i$ with $i \notin \vec u$. %
The previous calculations also decompose 
the frequency domain $\R^d$ into parts which belong to the different ANOVA terms, in the sense that for the ANOVA term $\vec u$ the corresponding frequencies $\w$ have to fulfill $\w_{\vec u^c} = \vec 0$. For that reason let us define the frequency decomposition
\begin{equation}\label{eq:Q_u}
Q_{\vec u}\coloneqq \{\w \in \R^d\mid \w_{\vec u^c} = \vec 0, \omega_i \neq 0 \text{ for } i\in \vec u\}.
\end{equation}
An illustration for the three-dimensional case can be found in Figure~\ref{fig:frequency_decomp}. In general $f_{\vec u}\notin L_2(\R^{|\vec u|})$, but $f_{\vec u}\in L_2(\R^{|\vec u|}, \mu_{\vec u})$.
 
\begin{figure}[htb]
\centering
\begin{minipage}[t]{0.5\textwidth}
	\centering
		\includegraphics[width=\textwidth]{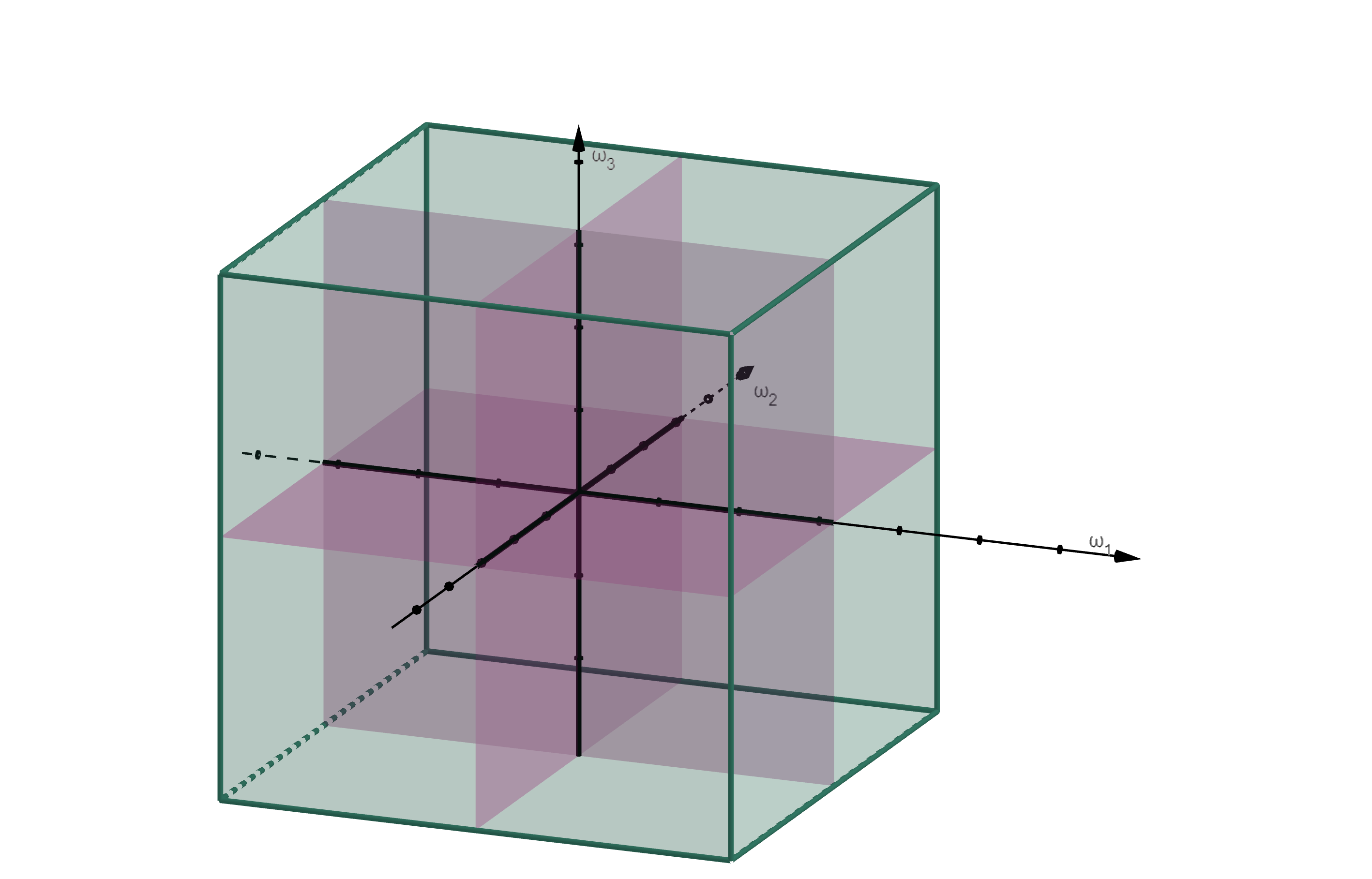}
		\centering 
		$=$
\end{minipage}\\
\begin{minipage}[t]{0.22\textwidth}
	\includegraphics[width=\textwidth]{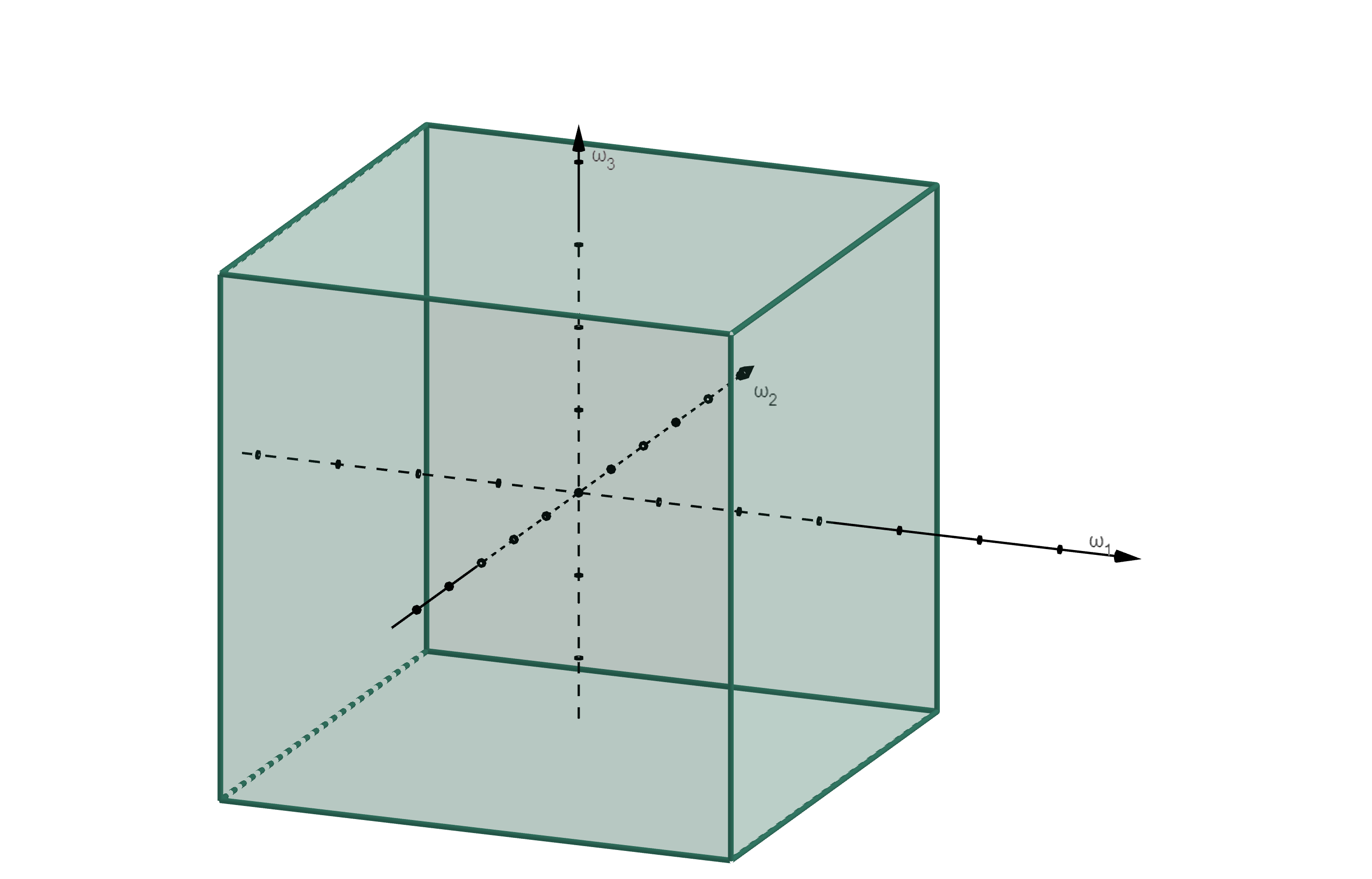}
\centering
$Q_{\{1,2,3\}}$
\end{minipage}%
		\begin{minipage}{0.038\textwidth}
				\vspace*{-50pt}
			$+$
		\end{minipage}%
\begin{minipage}[t]{0.22\textwidth}
\includegraphics[width=\textwidth]{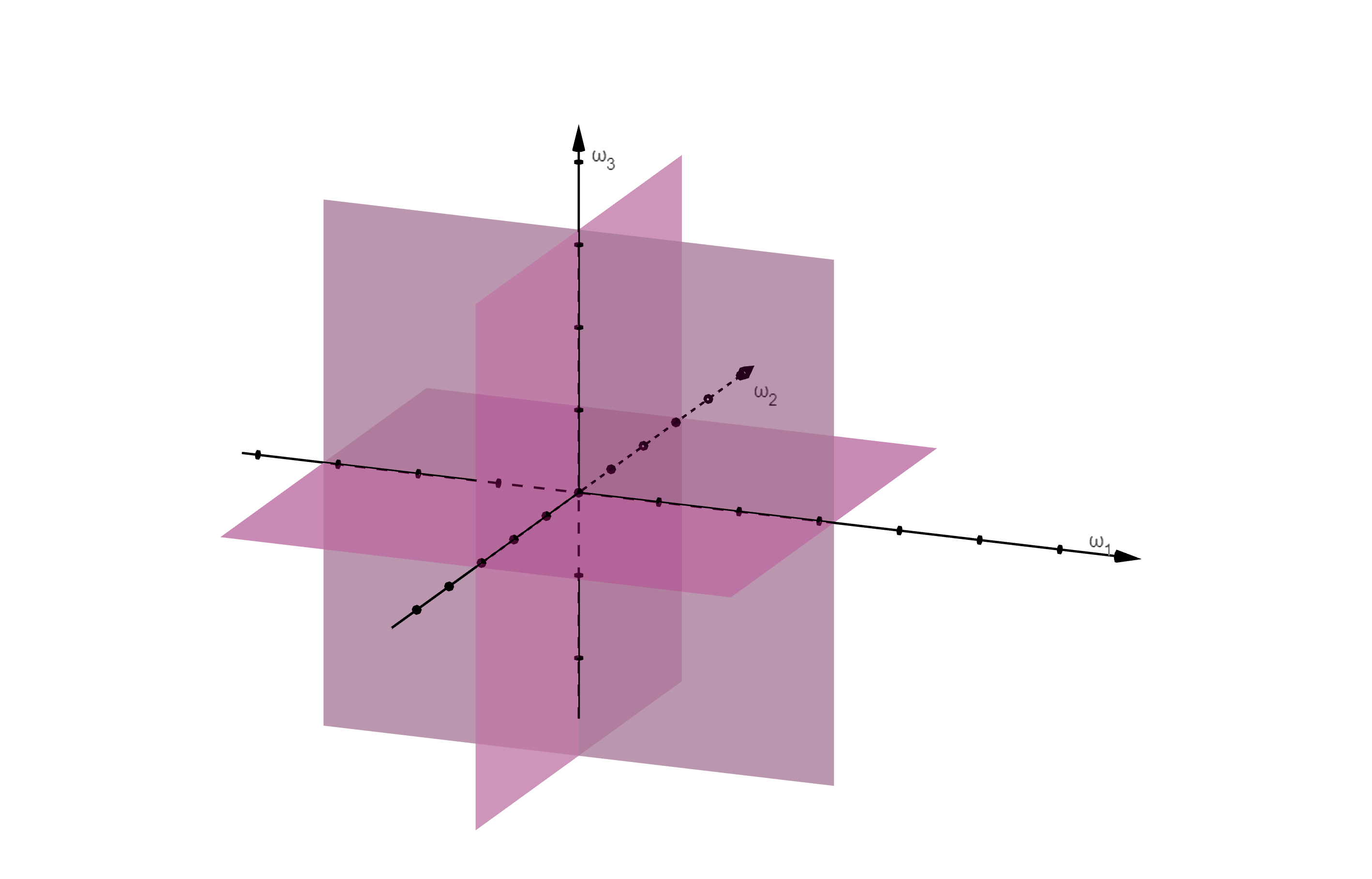}
\centering
$Q_{\{1,2\}},  Q_{\{1,3\}}, Q_{\{2,3\}}$
\end{minipage}%
		\begin{minipage}{0.038\textwidth}
		\vspace*{-50pt}
			$+$
		\end{minipage}%
\begin{minipage}[t]{0.22\textwidth}
\includegraphics[width=\textwidth]{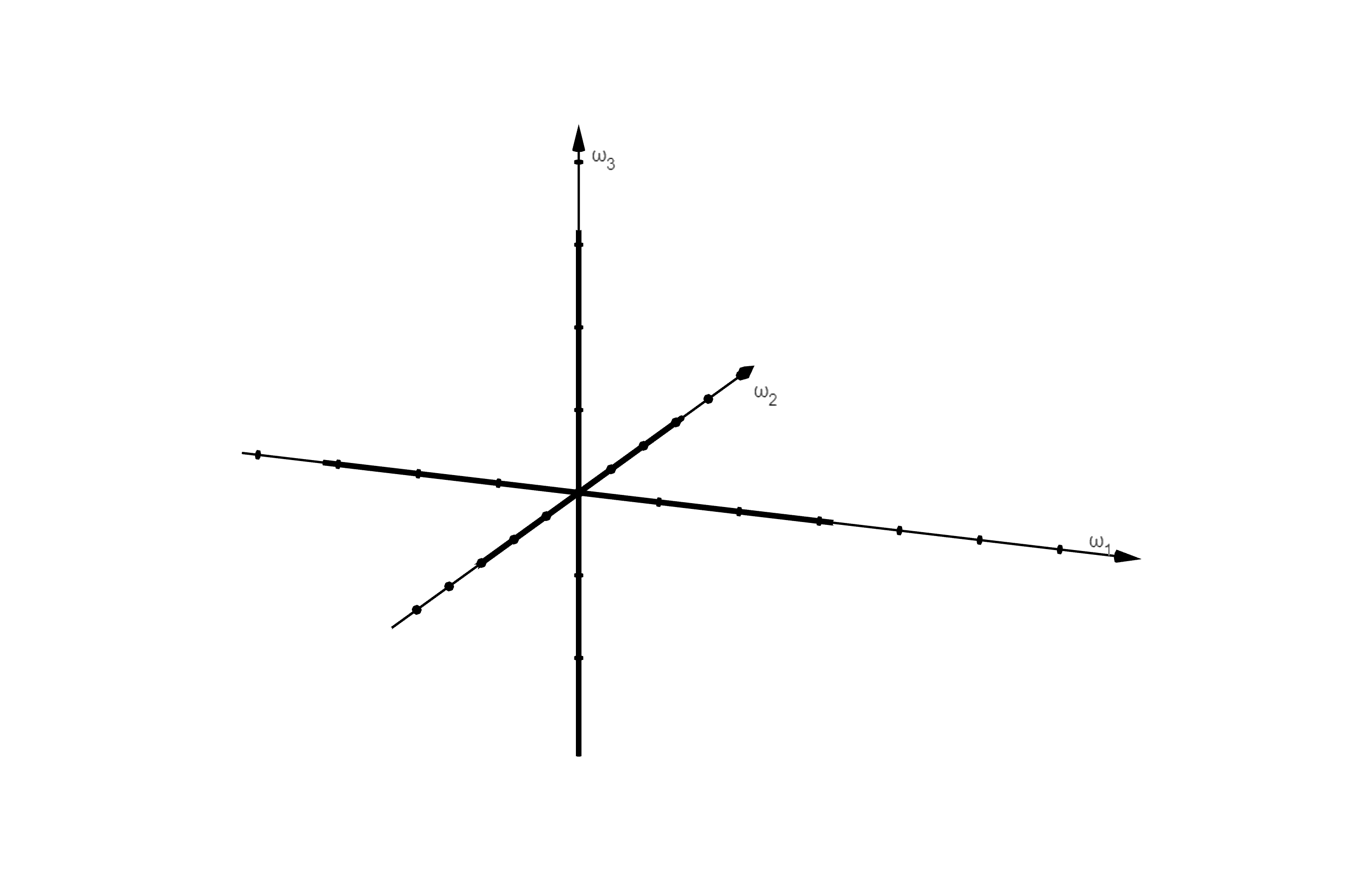}
\centering
$Q_{\{1\}} ,Q_{\{2\}},Q_{\{3\}}$
\end{minipage}%
		\begin{minipage}{0.038\textwidth}
				\vspace*{-50pt}
			$+$
		\end{minipage}%
\begin{minipage}[t]{0.22\textwidth}
\includegraphics[width=\textwidth]{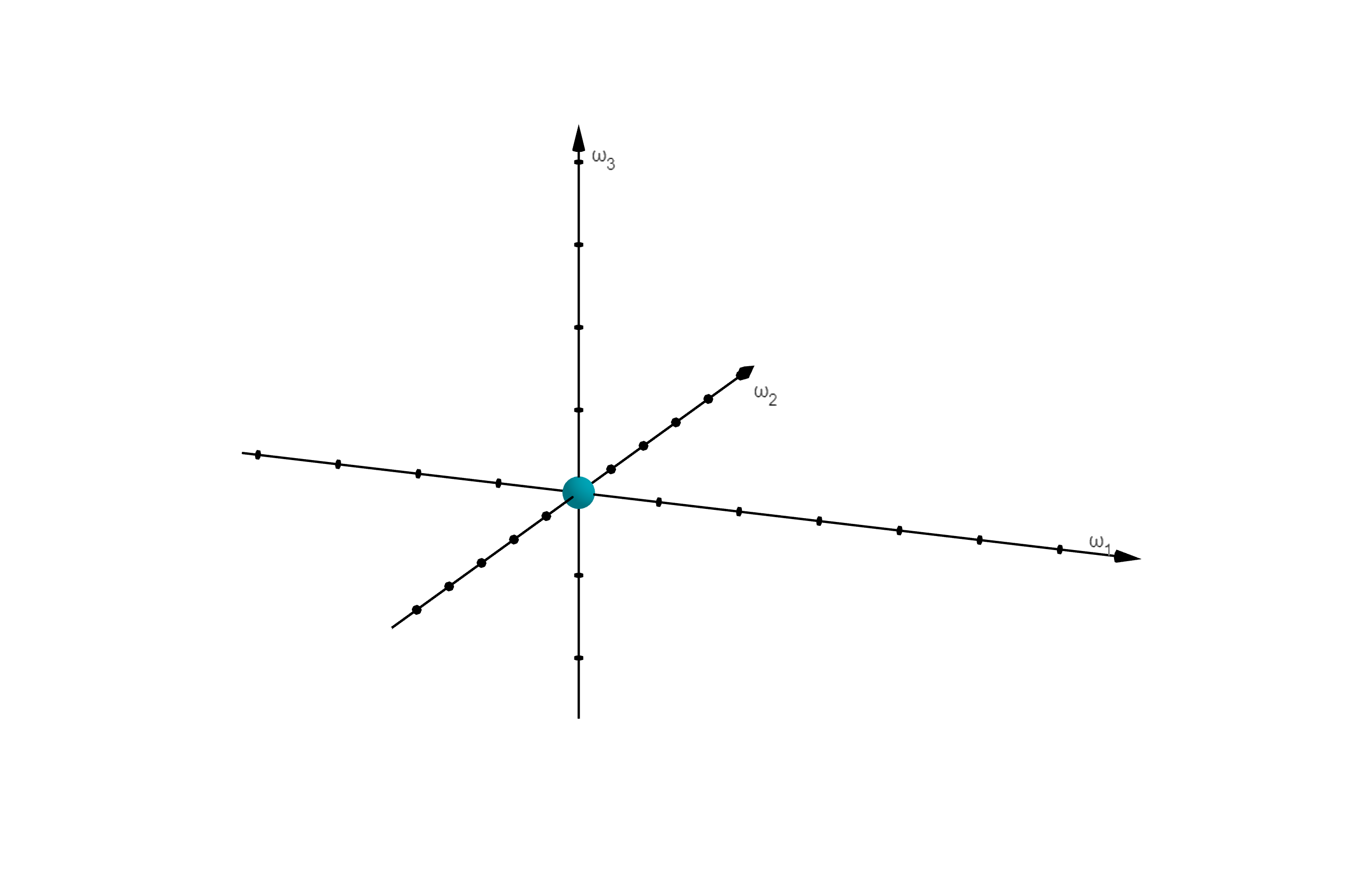}
\centering
$Q_{\varnothing}$
\end{minipage}%
\caption{Decomposition of the frequency domain $\R^3$ into the lower dimensional parts. The lower dimensional subsets $Q_{\vec v}$ are not part of the higher-dimensional $Q_{\vec u}$ for $\vec v\subset \vec u$. }
\label{fig:frequency_decomp}
\end{figure}
\section{Random Fourier features}\label{sec:RFF}
Kernel-based approaches have been extensively used in data-based applications,
including image classification and high-dimensional function approximations since
they often perform well in practice. The random feature model is a popular technique for approximating
the kernel using a randomized 
basis that can avoid the cost of full kernel methods. An alternative perspective is to view the random feature model  as a non-linear randomized function
approximation. 
The theoretical foundation of random Fourier features builds on Bochner's characterization of positive definite functions, see also the pioneering work for random features~\cite{RaRe07}.

\begin{Theorem}[Bochner's theorem \cite{Bo23}]
A continuous and shift-invariant function $\kappa \colon \R^d \rightarrow \R$ is positive definite if and only if $\kappa$ is the Fourier transform of a non-negative measure.
\end{Theorem}
If a shift-invariant kernel $\kappa$ is properly scaled with $\kappa(\vec 0)= 1$, Bochner’s theorem guarantees that its Fourier
transform $\rho(\w)$ is indeed a probability distribution, which means that
\begin{equation}\label{eq:kernel}
\kappa(\x-\x') = \int_{\R^d} \e^{\im \langle \w,\x-\x'\rangle }\rho(\w) \dx \w = \E_{\w \sim \rho}\left(\e^{\im \langle \w,\x\rangle} \e^{-\im \langle \w,\x'\rangle}\right) ,
\end{equation}
see also~\cite[Chapter 4.4]{PlPoStTa23}.
According to~\eqref{eq:kernel}, the random Fourier feature model makes use of the standard Monte Carlo sampling scheme to approximate $\kappa(\x-\x')$. In particular, one uses the approximation
\begin{equation}\label{eq:kernel2}
k(\x-\x') = \E_{\w \sim \rho}\left(\e^{\im \langle \w,\x\rangle}\cdot  \e^{-\im \langle \w,\x'\rangle}\right) \approx \vec A(\x) \cdot \vec A(\x')^*,
\end{equation}
with the explicit feature mapping 
$$
\vec A(\x) =  \left(\e^{\im \langle \w_1,\x\rangle}, \cdots, \e^{\im \langle \w_N,\x\rangle} \right)\in \C^{N},
$$
where $\{\w_{k}\}_{k=1}^N$ are sampled from $\rho(\w)$ independently of the training set $\X$.
Consequently, the original kernel matrix $\vec K = \left(\kappa(\x-\x')\right)_{\x\in \X,\x'\in \X}$ is then approximated by
$$\vec K \approx \frac 1N \vec A \vec A^*, \quad \vec A = \left(\e^{\im \langle \w_k,\x\rangle} \right)_{\x\in \X, k=1,\ldots,N}\in \C^{|\X|\times N}.$$

In the case of $N \ll |\X|$, this is a low rank approximation of the kernel which is for a big amount of samples $|\X|$ computational more feasible than the kernel matrix $\vec K$ itself. \par

We will focus on two families of feature distributions, Gaussian features and the Sobolev-type features with i.i.d coordinates (associated to the Gaussian kernel and Laplace-type kernel, respectively):
\begin{align}
\rho^\sigma_{\G}(\w) &\coloneqq \left(\frac{1}{\sigma\sqrt{2\pi}}\right)^d\exp{\left(-\norm{\w}^2/(2\sigma^2)\right)},\quad \quad && \text{Gaussian density}, \sigma>0 \label{eq:Gauss}\\
\rho^{s,\sigma}_{\Pi}(\w) &\coloneqq  c_{\rho_\Pi^s} \,\prod_{i\in [d]}\frac{1}{\sigma \left(1+\omega_i^2/\sigma^2\right)^s}, && \text{tensor-product density } (d\geq 2, s>\tfrac 12,\sigma>0 ),\label{eq:rho_s}
\end{align}
with the constant $c_{\rho_\Pi^s} \coloneqq  \left(\frac{\Gamma(s)}{\sqrt{\pi}\,\Gamma(s-\tfrac 12)}\right)^d$ 
chosen to ensure the associated densities have unit mass. 
One special case of the tensor-product density, is the tensor-product \textbf{Cauchy} distribution 
\begin{equation}\label{eq:Cauchy}
\rho^{\sigma}_{C}(\w) \coloneqq  \prod_{i = 1}^d \frac{1}{\pi \sigma (1+w_{i}^2/\sigma^2)}.
\end{equation}

From the neural network point of view, this is a two-layer network
with a randomized but fixed single hidden layer. Given an unknown function
$f \colon  \R^d \rightarrow \C$, the random Fourier feature model takes the form 
$$f^\#(\x) = \sum_{j = 1}^Na_j \e^{\im \langle\w_j,\vec x\rangle},$$ 
where $\x\in \R^d$ is the
input data, $(\w_j)_{j=1}^N$ are the random weights and $\vec a = (a_j)_{j=1}^N\in \C^N$ is the final weight layer.

Existing algorithms differ in how they select features $\w_j$ and weights $a_j$. In most cases, the features $\w_j$ are independent and identically distributed random variables generated by the (user defined) probability density function $\rho(\w)$. Then, for the random Fourier feature model, the output layer $\vec a$ is
trained (training data-dependent or independent), while the hidden layer (the weights $\w_j$) are fixed. \par

Suppose we are given a probability density $\rho$ used to sample the entries of the random weights $\w$.
Let us recall the definition for bounded $\F(\rho)$-norm functions, see also~\cite{Ha23, Xie22, HARFE23}.

\begin{Definition}\label{def:rho-norm}
Let $\rho\colon\R^d\rightarrow \R$ be a density function. A function $f\colon \R^d\rightarrow \R$ has finite $\F(\rho)$-norm with respect to $\e^{\langle \w, \cdot\rangle}$ if it belongs to the class
\begin{equation}\label{eq:F_rho_old}
\F(\rho) \coloneqq \left\{f(\vec x) = \int_{\R^d} \hat{f}(\w) \e^{\im \langle \w, \x\rangle}\d \w\mid  \norm{f}_{\F(\rho)}\coloneqq \sup_{\w\in \R^d} \left\lvert \frac{\hat{f}(\w)}{\rho(\w)}\right\rvert <\infty\right\}.
\end{equation}
\end{Definition}
Choosing a Gaussian distribution $\rho$ like in \cite{Xie22,Ha23} is a very restrictive condition to the function space $\F(\rho)$, since the Fourier transform $\hat f$ has to decay faster than exponentially. For instance, in Sobolev spaces the decay of the Fourier transform is polynomially.
For functions in $\F(\rho)$ generalization error bounds for random feature ridge
regression from~\cite{Xie22,Ha23} achieve the rate $\O(M^{-1})$, provided the number of data samples grows with $M$ and satisfies certain statistical assumptions.\par

For the definition of $\F(\rho)$, one has to choose either $\rho$, such that $\supp \rho = \R^d$, or one has to ensure that $\hat f(\w) = 0$ on $\R^d\backslash \supp \rho$. This is a drawback of this definition used in the literature. Furthermore, when more information is known about the target function $f$, the rates and complexity
bounds improve (especially with respect to the dimension). This helps mitigate issues
with the approximation of functions in high-dimensions. Especially, if the function is of low order. This is what we want to study in this Section. 
If the function $f$ is of low order, say $q$, the Fourier transform $\hat f$ is not defined as a function, 
but only in distributional sense, which makes the norm in Definition~\ref{def:rho-norm} infinite. This also concerns the $\rho$-norm for function of lower dimension, defined in~\cite{Ha23} if the effective dimension of the function $f$ does not equal $q$, the sparsity of the drawn random features.
For that reasons we introduce in the following ANOVA truncated random Fourier features.

\subsection{Random  Fourier Features and ANOVA}\label{sec:RF_ANOVA}
In this section we relate the generalized ANOVA decomposition~\eqref{eq:anova_gen} to the random feature approximation. If nothing is known about the function $f$, it might be appropriate to use random Fourier features from a $d$-dimensional density $\rho$. But if the function is of low order or sparse in the ANOVA-terms, it is useful to decompose the density $\rho$ having the ANOVA decomposition of the function $f$ in mind. Shortly, we draw $n_{\vec u}$ random Fourier features supported on $\vec u$ for every ANOVA term $\vec u$ in the set $U\subset \mathcal P([d])$. The random Fourier features comprise all ANOVA terms of the downward closure
\begin{equation}\label{eq:under_U}
\underline U \coloneqq \left\{ \vec v \subseteq [d] \mid \exists \vec u\in U \text{ such that } \vec v\subseteq \vec u\right\}.
\end{equation}
Assume the function $f$ has a sparse ANOVA decomposition, i.e.~for small $\epsilon >0$, 
$$\norm{f-\TT_{ \underline{ U}} f}_{L_2(\R^d,\mu)} \leq \epsilon, \qquad \text{ where }  \TT_{ \underline{ U}} f(\x)= \sum_{\vec u \in  \underline{ U} } f_{\vec u}(\x_{\vec u}) .$$
For such functions it is reasonable to reduce the dimension of the random Fourier features, where the remaining main task is to find the index set $U$. In this paper we aim to develop Algorithms~\ref{alg:1} and \ref{alg:2} for finding this index set $U$.
\begin{Definition}[ANOVA-truncated random Fourier features]\label{def:ANOVA-features}
Let $U\subset \mathcal P([d])$ be a set of ANOVA indices and the functions $\rho_{\vec u}\colon \R^{|\vec u|}\rightarrow \R$ be probability distributions. A collection of $N = \sum_{\vec u\in U} n_{\vec u}$ weight vectors $\w_1,\ldots, \w_N$ is called a set of \textbf{ANOVA-truncated random Fourier features}, if it is generated as follows:
For each index $\vec u\in U$ draw $n_{\vec u}$ realizations $\vec z_1,\ldots,\vec z_{n_{\vec u}}$ from $\rho_{\vec u}$ and construct $|\vec u|$-sparse features $\w_k$ by setting $\supp(\w_k) = \vec u$ and $(\w_k)_{\vec u} = \vec z_k$.
Define the notation $\I_{\vec u} = \{\w \in \I \mid  \w \in Q_{\vec u} \} $, where $Q_{\vec u}$ is defined in~\eqref{eq:Q_u} and $\mathcal I = \{\w_1,\ldots,\w_N\}$.
\end{Definition}
Note that the algorithms~\cite{Ha23, HARFE23, Xie22} are restricted to index sets $U = \{\vec u \subseteq [d] \mid |\vec u| = q\}$. Concerning the interpretability of the results this is a disadvantage, since it can happen that non-important input variables gain significance, see the following example. 
For that reason, we use random Fourier features of different dimension up to order $q$, and not only random Fourier features of order $q$.

\begin{Example}\label{ex:compare_harfe}
Introduce a  function of the form
\begin{equation}\label{eq:fried}
f(x_1,\ldots, x_{20}) = f_{1,2}(x_1,x_2) + f_3(x_3) +f_4(x_4) +f_5(x_5),
\end{equation}
this includes the Friedmann function considered in~\cite[Fig.4]{HARFE23}. The Fourier transform $T_{\hat f}$ is supported on 
$$Q_\varnothing \cup Q_{\{1\}} \cup Q_{\{2\}} \cup Q_{\{1,2\}} \cup Q_{\{3\}} \cup Q_{\{4\}} \cup Q_{\{5\}}. $$
Choosing $q$-sparse random Fourier features with $U = \{\vec u\subset [d]\mid|\vec u|=1\}$ and without demanding some orthogonality to the ANOVA terms, means that for example the ANOVA term $f_1$ can be described by a sum $\sum_{k=2}^{20}\sum_{\w \in \I_{\{1,k\}}} \e^{\im (\omega_1x_1  +\omega_kx_k)}$.  
This leads to the problem, that coefficients $a_{\w}$ for $\w\in \I_{\{1,k\}}$ for some $k\in \{2,\ldots, 20\}$ describe the ANOVA term $f_1$ and are non-zero, despite the fact that the variable $x_k$ does not play a role in the function. Then, analyzing the
histogram based on the occurrence rate (as a percentage) of the input variables obtained from the HARFE
model like in~\cite[Fig.4]{HARFE23}, leads to non-zero weights for non-necessary variables.
\end{Example}

In applications the function $f$ is unknown, even if a 
formula for $f$ is available, the component functions $f_{\vec u}$ can not be calculated analytically using for example the coupled equations like in~\cite{Ra142}. In typical applications, the function $f$ is only available
by sampling points $\x\in \X$ from modelling or experiments. Therefore, a practical
numerical method is needed to construct each unique component function. Similar to~\cite{LiRa12}, we minimize the squared error under
the hierarchical orthogonality condition~\eqref{eq:sca_prod_gen}.
For a set $U\subseteq \mathcal P([d])$ and ANOVA-truncated random Fourier features $\w\in \I$ drawn according to Definition~\ref{def:ANOVA-features},
we construct the random feature matrix 
\begin{equation}\label{eq:A}
\vec A = \left(\vec A_{\vec u}\right)_{\vec u\in U} \quad \text{ with } \quad \vec A_{\vec u} = \left(\e^{\im \langle \w_{\vec u},\x_{\vec u}\rangle}\right)_{\x\in \X,\w_{\vec u}\in \I_{\vec u}}.
\end{equation}
Employing the generalized ANOVA decomposition, every term $f_{\vec u}$ of the function $f$ by a sum
$$ f_{\vec u}(\x_{\vec u}) \approx \sum_{\w_{\vec u}\in \I_{\vec u}}a_{\w} \e^{\im \langle \w_{\vec u},\x_{\vec u}\rangle},$$
where the related random Fourier features $\w$ are in $Q_{ \vec u}$. 
To find a suitable vector $\vec a=(a_{\w})_{\w \in \I}$, we use a regularization, which is similar to defining the cost function like the D-MORPH algorithm~\cite{LiRa12} does: 
A solution vector $\vec a$ should fulfill simultaneously $\norm{\vec A \vec a -\vec f}_2 = 0$ and the hierarchical orthogonality~\eqref{eq:sca_prod_gen}. Since the regularization by forcing the integrals~\eqref{eq:zero_mean_gen} to be zero is not numerically feasible, we use a discretization of the integrals $\langle f_{\vec u},f_{\vec v}\rangle_{\mu}$ instead:
\begin{equation}\label{eq:sca_prod_X}
\langle f_{\vec u} ,f_{\vec v} \rangle_{\X} \coloneqq \frac{1}{M}\sum_{\vec x\in \X}f_{\vec u} (\x_{\vec u})\overline{f_{\vec v}(\x_{\vec v})}. 
\end{equation}
In contrast to~\cite{LiRa12}, we do not want to enforce the hierarchical orthogonality~\eqref{eq:sca_prod_gen} by calculating an SVD, but rather by penalizing by using a regularization. The solution vector $\vec a^\#$ can then be obtained by minimizing 
\begin{align}\label{eq:lsqr_reg}
\vec a^\# &= \argmin_{\vec a} \norm{\vec A\vec a-\vec f}^2 + \lambda \norm{\vec a }^2_{\hat{\vec W}},\notag\\
& = \argmin_{\vec a} \norm{\begin{pmatrix}\vec A\\ \sqrt{\lambda \hat{\vec W}}\end{pmatrix}\vec a-\begin{pmatrix}\vec f\\ \vec 0\end{pmatrix}}_2^2 \\
\text{ where }\norm{\vec a }^2_{\hat{\vec W}}&=\vec a^* \hat{\vec W}\vec a=\sum_{\vec u\in U}\vec a_{\vec u} \hat{\vec W}_{\vec u}\vec a_{\vec u}, 
\end{align}
where we introduce the weight matrix $\hat{\vec W}$ by
\begin{align}\label{eq:W_hat}
\hat{\vec W} &\coloneqq \diag\left(\hat{\vec W}_{\vec u}\right)_{\vec u\in U},\notag\\
\hat{\vec W}_{\vec u} 
&= \frac{1}{M^2}\vec A_{\vec u}^* \begin{pmatrix}\vec A_{\vec v_1} \vec A_{\vec v_2}\cdots \end{pmatrix} \begin{pmatrix}\vec A^*_{\vec v_1}\\ \vec A^*_{\vec v_2}\\ \ldots\end{pmatrix}\vec A_\vec u 
= \frac{1}{M^2} \vec A_{\vec u}^* \left(\sum_{\vec v\subset \vec u} \vec A_{\vec v}\vec A_{\vec v}^* \right) \vec A_\vec u \in \C^{n_{\vec u} \times n_{\vec u}},
\end{align}
where $\vec v_i$ runs through all subsets $\vec v \subset \vec u$ and $\varnothing $ is also a subset $\vec v$ of $\vec u$.
The following lemma shows that this regularization coincides with the hierarchical orthogonality~\eqref{eq:sca_prod_gen}.
\begin{Lemma}\label{lem:lsqr_reg}
The regularization term $\vec a_{\vec u}^*\hat{\vec W}_{\vec u}\vec a_{\vec u}$ in~\eqref{eq:lsqr_reg} with the weight matrices $\hat{\vec W}_{\vec u}$, defined in \eqref{eq:W_hat} penalizes the non-orthogonality of the
terms $f_{\vec u}$ and $f_{\vec v}$ where $\vec v\subset \vec u$ with the discrete scalar product~\eqref{eq:sca_prod_X} in the sense that
$$\vec a_{\vec u}^*\hat{\vec W}_{\vec u} \vec a_{\vec u} \geq \sum_{\vec v\subset \vec u} \frac{1}{\norm{\vec a_{\vec v}}^2}|\langle f_{\vec u},f_{\vec v}\rangle_{\X} |^2.$$
 \end{Lemma}
\begin{proof}
The Cauchy-Schwarz inequality gives for all vectors $\vec b,\vec c$ and matrix $\vec A$ with suitable size that 
$$|\langle \vec A \vec b ,\vec c\rangle |^2\leq  \norm{\vec A \vec b }^2\norm{\vec c}^2.$$
Applying this to our setting yields for
\begin{align*}
\left|\langle f_{\vec u} ,f_{\vec v} \rangle_{\X}\right|^2 
&= \frac{1}{M^2} \left|\langle \vec A_{\vec u} \vec a_{\vec u},\vec A_{\vec v} \vec a_{\vec v}  \rangle\right|^2
= \frac{1}{M^2} \left|\langle \vec A_{\vec v}^*\vec A_{\vec u} \vec a_{\vec u}, \vec a_{\vec v}  \rangle\right|^2\\
&\leq \frac{1}{M^2} \norm{\vec A_{\vec v}^*\vec A_{\vec u} \vec a_{\vec u}}^2\norm{\vec a_{\vec v}}^2  
= \frac{\norm{\vec a_{\vec v}}^2 }{M^2}\left| \langle \vec A_{\vec v}^*\vec A_{\vec u} \vec a_{\vec u}, \vec A_{\vec v}^*\vec A_{\vec u} \vec a_{\vec u} \rangle\right|\\
&= \frac{\norm{\vec a_{\vec v}}^2 }{M^2} \vec a_{\vec u}^* \vec A_{\vec u}^*\vec A_{\vec v}\vec A_{\vec v}^* \vec A_\vec u\vec a_{\vec u}.
\end{align*}
Hence, the regularization term $\vec a_{\vec u} \hat{\vec W}_{\vec u} \vec a_{\vec u}$ has the following connection to the discrete orthogonality of the terms $f_{\vec u}$ and $f_{\vec v}$, if $\vec a_{\vec v} \neq \vec 0$,
\begin{align*}
\vec a_{\vec u}^*\hat{\vec W}_{\vec u} \vec a_{\vec u}
&=\vec a_{\vec u}^* \left( \frac{1}{M^2}\vec A_{\vec u}^* \left(\sum_{\vec v\subset \vec u} \vec A_{\vec v}\vec A_{\vec v}^* \right) \vec A_\vec u \right) \vec a_{\vec u}
\geq \sum_{\vec v\subset \vec u} \frac{1}{\norm{\vec a_{\vec v}}^2}|\langle f_{\vec u},f_{\vec v}\rangle_{\X} |^2.
\end{align*}
This finishes the proof.
\end{proof}
For the one-dimensional terms $\vec u = \{i\}$ we obtain equality in the previous lemma, the weight matrix $\hat{\vec W}_{\vec u}$ contains only $\vec A_{\varnothing}$ and is in this case equal to
$$\hat{\vec W}_{\{i\}} = \vec 1_{M\times M}$$ 
and
\begin{align*}
\vec a_{\{i\}}^*\hat{\vec W}_{\{i\}}\vec a_{\{i\}} 
&=\frac{1}{M^2}\vec a_{\{i\}}^*\vec A_{\{i\}}^*\begin{pmatrix}1&\cdots &1  \\ &\vdots&\\ 1&\cdots& 1 \end{pmatrix}\vec A_{\{i\}}\vec a_{\{i\}}
= \left|\frac{1}{M}\sum_{\x\in \X} f_i(x_i) \right|^2 = \frac{1}{f_{\varnothing}^2}|\langle f_{i}, f_{\varnothing}\rangle_{\X}|^2
=\left|\sum_{\x\in \X}f_i(x_i)\right|^2.
\end{align*}
For the two-dimensional case $\vec u = \{i,j\}$ we have
$$\hat{\vec W}_{\{i,j\}} = \vec 1_{M\times M} + \vec A_{\{i\}}\vec A_{\{i\}}^* +\vec A_{\{j\}}\vec A_{\{j\}}^*.$$ 
To solve the regularized least squares problem~\eqref{eq:lsqr_reg}, we have to construct the matrix $\vec A$ and the matrix $\hat{\vec W}$, which is a block diagonal matrix, with blocks belonging to every $\vec u\in U$, so the square root has to be calculated for every block separately only. The actual minimization problem is then solved by an iterative least squares algorithm.

\section{Sensitivity analysis}\label{sec:sensitivity}
The aim of sensitivity analysis is to study how the output of a mathematical model or system can be divided and allocated to different input variables. Or, speaking in the ANOVA setting, to compare the variances $\sigma^2(f_{\vec u}) $ of the different ANOVA terms. This simplifies the model: Overly complex models may complicate analysing the inputs. By performing sensitivity analysis, users can better understand what factors don't actually matter and can be removed from the model.\par

Suppose a set of ANOVA indices $U\subseteq \mathcal P([d])$ and ANOVA-truncated random Fourier features $\w\in \I$ according to Definition~\ref{def:ANOVA-features}. Let some approximation to $f$ be a sum $f^\#$ of the form
\begin{equation}\label{eq:f_sharp}
f^\#(\vec x)=\sum_{\vec u\in U}\sum_{\w\in \I_{\vec u}} a^\#_{\w} \e^{\im\langle\w_{\vec u},\x_{\vec u}\rangle },
\end{equation}
with some coefficients $a^\#_\w$ (which can be the output of some optimization algorithm). In the following we will show two methods for performing sensitivity analysis for a function of kind~\eqref{eq:f_sharp}. The first approach in Section~\ref{sec:gsi1} can be applied for independent input variables and exploits the tensor product structure of the corresponding sampling density $\mu$. We start with $q$-sparse random Fourier features and include successively needed ANOVA-terms of order smaller than $q$. In contrast to that, in Section~\ref{sec:gsi2} we will first start by incorporating all ANOVA-terms up to order $q$ and omit these with low variance.

\subsection{Sensitivity analysis for independent input variables }\label{sec:gsi1}
In the case of a sum $f^\#(\x) = \sum_j a_j e_j(x)$ with basis functions $e_j$ being orthonormal in $L_2(\R^d,\mu)$, the variances of the ANOVA terms can be calculated easily from the coefficients $a_j$, see for example~\cite{SchmischkeDiss,LiPoUl21} for the exponential basis or the wavelet basis on the torus, respectively. In this paper we want to study the problem of unknown sampling density $\mu$, such that an orthogonal basis is not available. \\ 
However, analogously to Lemma~\ref{lem:ANOVA-terms_f} the recursive definition~\eqref{eq:anova-terms} splits the function $f^\#$ into the terms
\begin{align*}
f^\#_{\varnothing} &= \int_{\R^d}  \sum_{\w\in \I} a^\#_{\w} \e^{\im\langle\w,\x\rangle }\mu(\x) \d \x= \frac{1}{(2\pi)^d}\int_{\R^d}  \sum_{\w\in \I} a^\#_{\w} \e^{\im\langle\w,\x\rangle }\mu(\x) \d \x 
= \frac{1}{(2\pi)^d}\sum_{\w\in \I} a^\#_{\w}\hat{\mu}(-\w)\\
f^\#_{\vec u}(\vec x_{\vec u}) &= \frac{1}{(2\pi)^d}\sum_{\w\in \I}a^\#_{\w} \prod_{i\in \vec u} \left(\e^{\im \omega_i x_i}- \hat{\mu}_i(-\omega_i)\right) \prod_{i \in \vec u^c} \hat{\mu_i}(-\omega_i) = \frac{1}{(2\pi)^d}\sum_{\w\in \I}a^\#_{\w} E(\x,\w,\mu,\vec u), 
\end{align*}
with the terms $E$ defined in~\eqref{eq:def_E}.

Instead of calculating the integrals for the ANOVA decomposition~\eqref{eq:anova-terms}, we derive advantage from the fact that the points are sampled from the density $\mu$. Then the Monte-Carlo approximation of the integrals can be calculated using the RFF matrix $\vec A$ and a coefficient vector $\vec a^\#$.\par

We consider the setting where the set of ANOVA indices $U$ is completely arbitrary. We propose to start with $U = \{u \in [d] \mid |u| = q\}$ and refining this set iteratively, we will give more details later in this section.
Since the functions $\e^{\im \langle \w,\cdot \rangle}$ are not orthogonal for random drawn frequencies $\w$, the decomposition~\eqref{eq:f_sharp} is not the unique ANOVA decomposition of $f^\#$. But we notice, that for $\vec u\in U$ with $\{\vec v\in U \mid \vec u\subset \vec v \} = \varnothing$, the ANOVA term $f^\#_{\vec u}$ is completely contained in the sum $\sum_{\w\in \I_{\vec u}} a^\#_{\w} \e^{\im\langle \w_{\vec u},\x_{\vec u}\rangle }$. The main idea of our procedure is to calculate, if such an ANOVA term is really necessary or if the indices can be reduced to even lower dimensional sparse random Fourier features.
\begin{lemma}\label{lem:gsi}
Let $f^\#$ from~\eqref{eq:f_sharp} be the output of a trained random feature model.
Fix a subset $\vec u\in U$ with no superset in $U$, i.e~$\{\vec v\in U \mid \vec u\subset \vec v \} = \varnothing$ and denote $g(\vec x_{\vec u})\coloneqq \sum_{\w\in \I_{\vec u}} a^\#_{\w} \e^{\im\langle \w_{\vec u},\x_{\vec u}\rangle } $. 
Replacing in the iterative formula~\eqref{eq:anova-terms} every integral by a Monte-Carlo approximation on the sample points $\X$, gives the terms
\begin{align*}
g^{\text{MC}}_{\varnothing} &= \frac 1M \sum_{\w\in \I_{\vec u}}\vec a^{ \#}_{\w} \sum_{\x^{(j)} \in \X} \e^{\im \langle \vec x^{(j)}, \vec w\rangle},\\
 g^{\text{MC}}_{\vec v}(\vec x_{\vec v}) &= \frac 1M \sum_{\w\in \I}\vec a^{ \#}_{\w} \sum_{\x^{(j)} \in \X} \e^{\im \langle \vec x^{(j)}_{\vec u\backslash \vec v}, \vec w_{\vec u\backslash \vec v}\rangle} \prod_{i\in \vec v} \left(\e^{\im x_i\omega_i} - \e^{\im \vec x_{i}^{(j)}\omega_i}\right).
\end{align*} 
\end{lemma}
\begin{proof}
We use induction over $|\vec v|$, and begin with $\vec v = \varnothing$,
\begin{align*}
g_{\varnothing} &= \int_{\R^d}  f^\#(\vec x) \mu(\x) \d \x \approx \frac{1}{M}\sum_{\vec x\in \X} \sum_{\w\in \I} a^{ \#}_{\w} \e^{i\langle\w,\x\rangle } =: g^{\text{MC}}_{\varnothing}.
\end{align*}
The tensor product structure of the density $\mu$ in \eqref{eq:mu_prod} is the basis to approximate an $|\vec v|$-dimensional integral with respect to $\mu_{\vec v}$ by
$$\int_{\R^{|\vec v|}}g(\vec x_{\vec v},\vec x_{\vec v^c} ) \mu_{\vec v}(\vec x_{\vec v})\d \vec x_{\vec v}
\approx \frac 1M \sum_{\x^{(j)}\in \X} g(\vec x^{(j)}_{\vec v},\vec x_{\vec v^c}  ), $$
which is a function that depends on the variables $\x_{\vec v^c}$ and not on the variables $\x_{\vec v}$. 
For the induction step we use the iterative definition~\eqref{eq:anova-terms} of the ANOVA-terms,
\begin{align*}
g_{\vec v}(\vec x_{\vec v})&=\int_{\R^{|\vec u|-|\vec v|}}g(\vec x_{\vec u}) \mu(\vec x_{\vec v^c})\d \vec x_{\vec v^c}-\sum_{\vec v'\subset \vec v}g_{\vec v'}(\vec x_{\vec v'})\\
&\approx \frac 1M \sum_{\vec x^{(j)}\in \X}  g(\vec x_{\vec v},\vec x^{(j)}_{\vec u\backslash \vec v}) - \sum_{\vec v'\subset \vec v} g^{\text{MC}}_{\vec v'}(\vec x_{\vec v})\\
&= \frac 1M  \sum_{\w\in \I_{\vec u}}  a^{ \#}_{\vec w} \e^{\im \langle\vec x_{\vec v},\w_{\vec v}\rangle} \sum_{\vec x^{(j)}\in \X} \left( \e^{\im \langle\w_{\vec u\backslash \vec v}, \x^{(j)}_{\vec u\backslash \vec v}\rangle}\right) \\
&\quad -  \frac 1M \sum_{\vec v'\subset \vec v} \sum_{\w \in \I_{\vec u}}a^{ \#}_{\vec w}  \sum_{\x^{(j)} \in \X} \e^{\im \langle \vec x^{(j)}_{\vec u\backslash \vec v'}, \w_{\vec u\backslash \vec v'}\rangle} \prod_{i\in \vec v'} \left(\e^{\im x_i\omega_i} - \e^{\im \vec x_{i}^{(j)}\omega_i}\right)\\
 &= \frac 1M  \sum_{\w\in \I_{\vec u}} a^{ \#}_{\vec w} \sum_{\vec x^{(j)}\in \X} \e^{\im \langle\w_{\vec u\backslash \vec v}, \x^{(j)}_{\vec u\backslash \vec v}\rangle}  \left(  \e^{\im \langle\vec x_{\vec v},\w_{\vec v}\rangle}   - \sum_{\vec v'\subset \vec v}  \e^{\im \langle \vec x^{(j)}_{\vec v\backslash \vec v'}, \w_{\vec v\backslash \vec v'}\rangle} \prod_{i\in \vec v'} \left(\e^{\im x_i\omega_i} - \e^{\im \vec x_{i}^{(j)}\omega_i}\right)\right)\\
 &= \frac 1M  \sum_{\w\in \I_{\vec u}} a^{ \#}_{\vec w} \sum_{\vec x^{(j)}\in \X} \e^{\im \langle\w_{\vec u\backslash \vec v}, \x^{(j)}_{\vec u\backslash \vec v}\rangle}  \left(   \prod_{i\in \vec v} \left(\e^{\im x_i\omega_i} - \e^{\im \vec x_{i}^{(j)}\omega_i}\right)\right).
\end{align*}
This finishes the proof.
\end{proof}

\begin{Remark}
The case where $\vec v = \vec u$ in the previous Lemma~\ref{lem:gsi} is of special interest. In this case we calculate, 
\begin{align}\label{eq:gsi_gMC}
\left(f_{\vec u}^\#(\x_{\vec u}) \right)_{\x \in \X} &\approx 
\frac 1M  \sum_{\w\in \I_{\vec u}} a^\#_{\vec w} \sum_{\vec x^{(j)}\in \X} \left(   \prod_{i\in \vec u} \left(\e^{\im x_i\omega_i} - \e^{\im \vec x_{i}^{(j)}\omega_i}\right)\right)\notag\\
\frac{\sigma^2(f_{\vec u}^\#(\x_{\vec u}))}{\sigma^2(f)} &\approx \frac{\norm{f_{\vec u}^\#(\x_{\vec u})}^2_{\ell_2(\X)}}{\sigma^2(\vec f)} .
\end{align}
Introducing three-dimensional tensors $\e^{\im \x^{(j)}_i \omega_i} -\e^{\im \x^{(\tilde j)}_i \omega_i}\in \C^{|\X|\times |\X| \times n_{\vec u}}$ this 
is calculated numerically by a tensor-vector multiplication, followed by a summation, point-wise squaring and a summation. We approximate the variance of $f$ by the variance of the given data vector $\vec f$. 
A splitting of the given data into test data and validation data gives the possibility to use the prediction on the validation set to estimate the variances of the ANOVA terms on validation data.
\end{Remark}

The variance of $f^\#_{\vec u}$ is a good approximation to the variance of $f_{\vec u}$ if the error $\norm{f_\vec u - f^\#_{\vec u}}_{L_2(\R^{|\vec u|},\mu_{\vec u})}$ is small, see
\begin{align}\label{eq:error_gsi}
|\sigma^2(f_{\vec u}) - \sigma^2(f^\#_{\vec u})| 
&= \left|\int_{\R^{|\vec u|}} \left( |f_{\vec u}(\x_\vec u)|^2 - |f^\#_{\vec u}(\x_{\vec u})|^2\right) \mu_{\vec u}(\x_{\vec u} )\dx \x_{\vec u}\right|\notag\\
&=\left|\int_{\R^{|\vec u|}} \left( |f_{\vec u}(\x_\vec u)| - |f^\#_{\vec u}(\x_{\vec u})|\right) \left( |f_{\vec u}(\x_\vec u)| + |f^\#_{\vec u}(\x_{\vec u})|\right)\mu_{\vec u}(\x_{\vec u}) \dx \x_{\vec u}\right|\notag\\
&\leq \norm{f_\vec u - f^\#_{\vec u}}_{L_2(\R^{|\vec u|},\mu_{\vec u})}  \norm{f_\vec u + f^\#_{\vec u}}_{L_2(\R^{|\vec u|},\mu_{\vec u})} \notag\\
&\leq \norm{f_\vec u - f^\#_{\vec u}}_{L_2(\R^{|\vec u|},\mu_{\vec u})}  \norm{f_\vec u }_{L_2(\R^{|\vec u|},\mu_{\vec u})} \left(2 +  \norm{f_\vec u - f^\#_{\vec u}}_{L_2(\R^{|\vec u|},\mu_{\vec u})}  \right).
\end{align}
The algorithms in the exsiting literature~\cite{Ha23,Xie22,HARFE23} used $q$-sparse random Fourier features, but they numerically verified that choosing $q$ equal to the real effective dimension of the function $f$ leads to best approximation results. Furthermore, the norm $\F(\rho)$ from Definition~\ref{def:rho-norm} is not finite, the Fourier transform $\hat f$ is defined only in distributional sense.

\begin{Definition}\label{def:anti-dc}
A set $U$ is called \textbf{anti downward closed} if for every $\vec u\in U$ there is no index $\vec v\in U$, which is a subset of $\vec u$. 
\end{Definition}
We propose to choose an anti downward closed set $U$. Furthermore, we propose to start with $q$-sparse random Fourier features and customize the random features to the ANOVA decomposition of the function $f$ iteratively. This works as follows. We start with $q$-sparse random Fourier features as proposed in the literature so far by choosing $ U = \{\vec u\in \mathcal P([d])\mid |\vec u| = q\} $. Then we draw in total $N$ random Fourier features and learn a first approximation, described by the parameter vector $\vec a^\#$. 
This is Stage I of Algorithm~\ref{alg:1}. \par
Then, using the variance estimations~\eqref{eq:gsi_gMC}, Algorithm~\ref{alg:1} decides in Stage II for every $\vec u\in U$, if it keeps this ANOVA index or if it omits this ANOVA index and uses instead all indices $\vec u$ of order $q-1$ which are contained in $\vec u$. This procedure is done $q$ times, to reduce the ANOVA index set $U$ to the really necessary variable interactions. If the function has only a low amount of non-zero ANOVA-terms this leads to a huge decrease of non-necessary parameters in the model, which is the starting point for the iterative pruning steps. We summarize this in Algorithm~\ref{alg:1}. \par

\begin{algorithm}[tb]
\caption{ANOVA-boosting for independent input variables}
	\vspace{2mm}
	\begin{tabular}{ l l l }
		\textbf{Input:}
		&	$\X = (\x^{(i)})_{i=1}^M\in \R^d$ & sampling nodes \\
		& $\vec f = (f(x^{(i)}))_{i=1}^M$ & function values at sampling nodes\\
		& $ q $ & maximal superposition dimension\\
		& $\varepsilon$ & ANOVA threshold\\
		& $N$ & number of total random Fourier features \\
		& $\lambda$ & regularization parameter\\
		 & $\rho_{\vec u} \text{ for } |\vec u|\leq q$  & feature distributions\\
	\end{tabular}\\
	\textbf{Stage I: Initialization} 
	\begin{algorithmic}[1]
			\STATE{$U = \{\vec u\subseteq[d] \mid |\vec u|=q\}$.}
			\STATE{$n =\left\lfloor \frac{N}{|U|} \right\rfloor$.}
			\STATE{For every $\vec u\in U$ draw $n $ $q$-sparse features $\w\in \I_{\vec u}$ from $\rho_{\vec u}$ and construct the matrix 
			$$\vec A=[\vec A_{\vec u}]_{\vec u \in U}\in \C^{M\times N}, \quad \vec A_{\vec u} = (\e^{\im \langle\w_{\vec u},\x_{\vec u}\rangle})_{\vec x\in \X_{\text{train}},\w \in \I_{\vec u}}\in \C^{M\times N}.$$
			First approximation: $\vec a^\# = \vec A^* (\vec A\vec A^*+ \lambda \vec I)^{-1}\vec f$, $f^\#(\x)=\sum_{k=1}^N a_k^\# \e^{\im \langle\w_k,\x \rangle}$.}%
			
			\textbf{\hspace{-17pt}Stage II: ANOVA-boosting}
			\FOR{$t = 1, \ldots, q$}
			\STATE{For every $\vec u\in U$ calculate the variances $\sigma^2_{\text{MC}}( f^\#_{\vec u})$ using \eqref{eq:gsi_gMC}. }
			\STATE{$U \leftarrow \{\vec u \in U\mid \sigma^2_{\text{MC}}( f^\#_{\vec u})\geq \epsilon\}$. }
			\STATE{$U_{t} = \{\vec v\in [d] \mid |\vec v| = t-1, \nexists \vec u\in U \text{ with } \vec v\subset \vec u \}$. }
			\STATE{$U \leftarrow  U \cup  U_{t}$.}
			\STATE{Draw $n = \left\lfloor \frac{N}{|U|} \right\rfloor$ $|\vec u|$-sparse features $\w\in \I_{\vec u}$ from $\rho_{\vec u}$ for every $\vec u\in U$ (keep already drawn features).}
			\STATE{Construct the matrix $\vec A$ and update the approximation $\vec a^\# = \vec A^* (\vec A\vec A^*+ \lambda \vec I)^{-1}\vec f$, $f^\#(\x)=\sum_{k=1}^N a^\#_k \e^{\im \langle\w_k,\x \rangle}$. }
			\ENDFOR
	\end{algorithmic}
	\begin{tabular}{ l l l }
		\textbf{Output:} 
		$U$
	\end{tabular}
	\label{alg:1}
\end{algorithm}  

In Figure~\ref{fig:anova_part} we illustrate this procedure for an example function, which can be written in the form
\begin{equation}\label{eq:example_plot}
f \colon\R^7 \rightarrow \R, \quad f(\x)= f_{\{1,2,3\}}(x_1,x_2,x_3) + f_{\{1\}}(x_1) + f_{\{1,3\}}(x_1,x_3) f_{\{5\}}(x_5) + f_{\{6,7\}}(x_6,x_7).
\end{equation}
\begin{figure}[htb!]
\centering
\begin{tikzpicture}[x=2.2cm,y=1.4cm]
	\node[mynode] (11) at (1,3) {$x_1$};
	\node[mynode] (12) at (1,2) {$x_2$};
	\node[mynode] (13) at (1,1) {$x_3$};
	\node[mynode] (14) at (1,0) {$x_4$};
	\node[mynode] (15) at (1,-1) {$x_5$};
	\node[mynode] (16) at (1,-2) {$x_6$};
	\node[mynode] (17) at (1,-3) {$x_7$};
	
	\node[mynode,draw=tucmath,fill=tucmath!20] (21) at (2,3.5) {\tiny{$\{1,2,3\}$}};
	\node[mynode] (22) at (2,2.5) {\tiny{$\{1,2,4\}$}};
	\node[mynode] (23) at (2,1.5) {\tiny{$\{1,2,5\}$}};
	\node[mynode] (24) at (2,-1.5) {\tiny{$\{4,5,6\}$}};
	\node[mynode] (25) at (2,-2.5) {\tiny{$\{4,5,7\}$}};
	\node[mynode] (26) at (2,-3.5) {\tiny{$\{5,6,7\}$}};
	\draw[thick,-to] (11) -- (21);
	\draw[thick,-to] (12) -- (21);
	\draw[thick,-to] (13) -- (21);
	\draw[thick,-to] (11) -- (22);
	\draw[thick,-to] (12) -- (22);
	\draw[thick,-to] (14) -- (22);
	\draw[thick,-to] (11) -- (23);
	\draw[thick,-to] (12) -- (23);
	\draw[thick,-to] (15) -- (23);
	
	\draw[thick,-to] (14) -- (24);
	\draw[thick,-to] (15) -- (24);
	\draw[thick,-to] (16) -- (24);
	\draw[thick,-to] (14) -- (25);
	\draw[thick,-to] (15) -- (25);
	\draw[thick,-to] (17) -- (25);
	\draw[thick,-to] (15) -- (26);
	\draw[thick,-to] (16) -- (26);
	\draw[thick,-to] (17) -- (26);

	\node[mynode,draw=tucmath,fill=tucmath!20] (31) at (3,3.5) {\tiny{$\{1,2,3\}$}};
	\node[mynode] (32) at (3,2.5) {\tiny{$\{1,4\}$}};
	\node[mynode] (33) at (3,1.5) {\tiny{$\{1,5\}$}};
	\node[mynode] (34) at (3,0.5) {\tiny{$\{1,6\}$}};
	\node[mynode] (35) at (3,-1.5) {\tiny{$\{5,6\}$}};
	\node[mynode] (36) at (3,-2.5) {\tiny{$\{5,7\}$}};
	\node[mynode,draw=tucmath,fill=tucmath!20] (37) at (3,-3.5) {\tiny{$\{6,7\}$}};
	
	\draw[thick,-to] (21) -- (31);
	\draw[thick,-to] (22) -- (32);
	\draw[thick,-to] (23) -- (33);
	\draw[thick,-to] (2,0) -- (32);
	\draw[thick,-to] (2,0) -- (33);
	\draw[thick,-to] (2,0) -- (34);
	\draw[thick,-to] (26) -- (35);
	\draw[thick,-to] (24) -- (35);
	\draw[thick,-to] (26) -- (36);
	\draw[thick,-to] (25) -- (36);
	\draw[thick,-to] (2,0) -- (37);
	\draw[thick,-to] (26) -- (37);
	\draw[thick,-to] (2,0) -- (36);
	\draw[thick,-to] (2,0) -- (35);
	
	\node[mynode,draw=tucmath,fill=tucmath!20]  (41) at (4,1.5) {\tiny{$\{1,2,3\}$}};
	\node[mynode] (42) at (4,0.5) {\tiny{$\{4\}$}};
	\node[mynode,draw=tucmath,fill=tucmath!20]  (43) at (4,-0.5) {\tiny{$\{5\}$}};
	\node[mynode,draw=tucmath,fill=tucmath!20]  (44) at (4,-1.5) {\tiny{$\{6,7\}$}};
	
	\draw[thick,-to] (31) -- (41);
	\draw[thick,-to] (32) -- (42);
	\draw[thick,-to] (33) -- (43);
	\draw[thick,-to] (35) -- (43);
	\draw[thick,-to] (36) -- (43);
	\draw[thick,-to] (37) -- (44);
\draw[thick,-to] (3,-0.5) -- (42);
	\draw[thick,-to] (3,-0.5) -- (43);
	
	\node[mynode,draw=tucmath,fill=tucmath!20] (51) at (5,1) {\tiny{$\{1,2,3\}$}};
	\node[mynode,draw=tucmath,fill=tucmath!20] (52) at (5,0) {\tiny{$\{5\}$}};
	\node[mynode,draw=tucmath,fill=tucmath!20] (53) at (5,-1) {\tiny{$\{6,7\}$}};
		\draw[thick,-to] (41) -- (51);
	\draw[thick,-to] (43) -- (52);
	\draw[thick,-to] (44) -- (53);

	\path (23) --++ (24) node[midway,scale=1.5] {$\vdots$};
	\path (34) --++ (35) node[midway,scale=1.5] {$\vdots$};

  \node[above=15,align=center,black] at (1,4) {input\\[-0.2em]variables};
	\node[above=15,align=center,black] at (2,4) {first\\[-0.2em]approx.};
  \node[above=15,align=center,black] at (3,4){second\\[-0.2em]approx.};
	 \node[above=15,align=center,black] at (4,4) {third\\[-0.2em]approx.};
  \node[above=15,align=center,black] at (5,4) {found terms}; 
	\node[align=center,black] at (1,-4.2) {$|U|:$};
	\node[align=center,black] at (2,-4.2) {35};
		\node[align=center,black] at (3,-4.2) {19};
		\node[align=center,black] at (4,-4.2) {4};
		\node[align=center,black] at (5,-4.2) {3};
	\end{tikzpicture}%
	
	\caption{Example procedure of finding the ANOVA-sparse random Fourier features: Consider a $7$-dimensional input function of the form~\eqref{eq:example_plot}, starting at $q=3$. The terms with approximated variance $\sigma^2(f_{\vec u})$ bigger than the threshold $\epsilon$ are highlighted in magenta for each approximation step. The result is an anti downward closed set $U$. At the bottom we give the number of indices in the index set $U$, which is used in the respective step.  }
	\label{fig:anova_part}

	\end{figure}
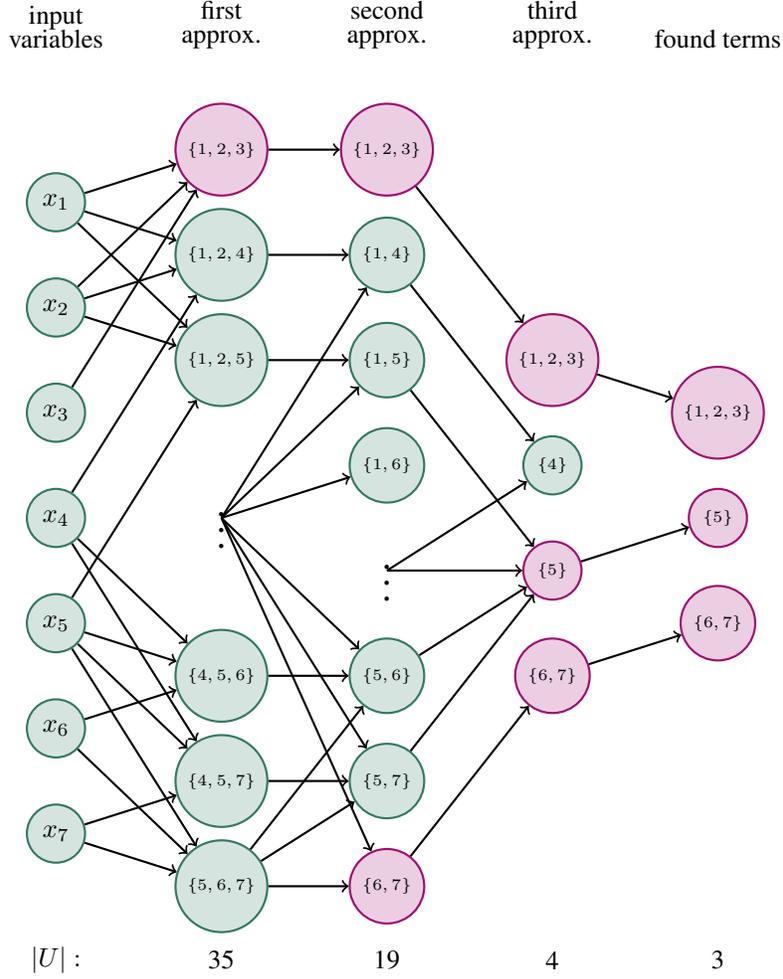%
We start with all three-dimensional terms, i.e.~$U = \{\vec u\subset [d]\mid |\vec u| = 3\}$. A sensitivity analysis of the first approximation $f^\#$ shows that only the three-dimensional term $\{1,2,3\}$ has variance bigger than a threshold $\epsilon$. The other three-dimensional terms can be replaced by all two-dimensional terms, where the terms $\{1,2\}$ and $\{2,3\}$ are not needed, because they are contained in the term $\{1,2,3\}$. A second approximation shows that only the variances of $f^\#_{\{1,2,3\}}$ and $f^\#_{\{6,7\}}$ are bigger than a threshold. In the third approximation only the additional one-dimensional terms $\{4\}, \{5\}$ are needed, since the other ones are contained in the higher-dimensional terms. After the third approximation only the important terms in an anti downward closed set $U$ remain for the next approximation step to reduce the over-parametrized model to an under-parametrized model, for example using SHRIMP or HARFE.

\subsection{Sensitivity analysis for correlated input variables}
\label{sec:gsi2}
If we do not have information about the sample density $\mu$, we do not have the tensor product structure as for independent input variables. 
We want to come back to the regularized least squares~\eqref{eq:lsqr_reg}. As shown in Lemma~\ref{lem:lsqr_reg}, this regularization ensures the hierarchical
orthogonality~\eqref{eq:sca_prod_gen} of the ANOVA terms. In this setting we demand that the 
sum $\sum_{\w\in \I_{\vec u}} a_{\w}^\# \e^{\im \langle \w_{\vec u},\x_{\vec u}\rangle}$ should approximate the ANOVA term $f_{\vec u}$ for every $\vec u\in U$.
Therefore, we start with $U = \{\vec u\subseteq [d]\mid |\vec u|\leq q\}$. Then the solution vector $\vec a^\#$ of~\eqref{eq:lsqr_reg} seperates the ANOVA terms in the sense that 
$$f_{\vec u}(\x_{\vec u}) \approx \sum_{\w\in \I_{\vec u}} a_{\w}^\# \e^{\im \langle \w_{\vec u},\x_{\vec u}\rangle}.$$
Let the RFF matrix $\vec A$ be split like in~\eqref{eq:A} and denote the vectors $\vec a^{\#}_{\vec u} = \left(a^{\#}_\w\right)_{\w \in \I_{\w}}$. Then we approximate the Sobol indices, see Definition~\ref{def:sobol_indices} by
\begin{align}\label{eq:sobol_num1}
S^{\text{MC}}_{\vec u,\text{var}} &=\frac{\norm{\vec A_{\vec u}\vec a^\#_{\vec u}}_2^2}{\sigma^2(\vec f)} \\
S^{\text{MC}}_{\vec u,\text{cor}} &=\frac{\sum_{\stackrel{\varnothing \neq \vec v \subseteq [d]}{\vec v\cap \vec u\neq \varnothing, \vec v\not\subseteq \vec u}} \langle \vec A_{\vec v} \vec a^\#_{\vec v}, \vec A_{\vec u} \vec a^\#_{\vec u} \rangle}{\sigma^2(\vec f)} \label{eq:sobol_num2}\\
S^{\text{MC}}_{\vec u} &=S^{\text{MC}}_{\vec u,\text{var}} +S^{\text{MC}}_{\vec u,\text{cor}}.\label{eq:sobol_num3}
\end{align}
The procedure is summarized in Algorithm~\ref{alg:2}. Note, that this procedure can also be applied to samples from tensor product sampling densities $\mu$. In that case we expect the indices $S_{\vec u,\text{cor}}$ to be zero.

 \begin{algorithm}[tb]
\caption{ANOVA-boosting for possibly dependent input variables}
	\vspace{2mm}
	\begin{tabular}{ l l l }
		\textbf{Input:}
		&	$\X = (\x^{(i)})_{i=1}^M\in \R^d$ & sampling nodes \\
		& $\vec f = (f(x^{(i)}))_{i=1}^M$ & function values at sampling nodes\\
		& $ q $ & maximal superposition dimension\\
		& $\varepsilon$ & ANOVA threshold\\
		& $N$ & total number of random Fourier features\\
		& $\lambda$ & regularization parameter\\
		& $\rho_{\vec u} \text{ for } |\vec u|\leq q$  & feature distributions\\
	\end{tabular}\\
	\textbf{Stage I: Initialization}
	\begin{algorithmic}[1]
			\STATE{$U = \{\vec u\subseteq[d] \mid |\vec u|\leq q\}$}
\[\]
			\textbf{\hspace{-17pt}Stage II: ANOVA-boosting}
			\FOR{$t = q, \ldots, 1$}
			\STATE{$n =\left\lfloor\frac{N}{|U|}\right\rfloor$.}
			\STATE{For every $\vec u\in U$ draw $n $ $q$-sparse features $\w\in \I_{\vec u}$ from $\rho_{\vec u}$ and construct the matrix 
			$$\vec A=[\vec A_{\vec u}]_{\vec u \in U}\in \C^{M\times N}, \quad \vec A_{\vec u} = (\e^{\im \langle\w_{\vec u},\x_{\vec u}\rangle})_{\vec x\in \X_{\text{train}},\w \in \I_{\vec u}}\in \C^{M\times N}.$$
			The solution vector $\vec a^\#$ is solution of minimization problem~\eqref{eq:lsqr_reg} by an iterative least squares algorithm, $f^\#(\x)=\sum_{k=1}^N a^\#_k \e^{\im \langle\w_k,\x \rangle}$.}
			\STATE{For every $\vec u\in U$ calculate the Sobol indices $S^{\text{MC}}_{\vec u,\text{var}}$, $S^{\text{MC}}_{\vec u,\text{cor}}$ and $S^{\text{MC}}_{\vec u}$ using \eqref{eq:sobol_num1} to \eqref{eq:sobol_num3}. }
			\STATE{$U \leftarrow \{\vec u \in U\mid S^{\text{MC}}_{\vec u,\text{var}}>\epsilon \text{ or }|\vec u|<t \}.$ }
			\ENDFOR
			\STATE{make anti downward closed set $U$ (see Definition~\ref{def:anti-dc}): $$U \leftarrow U\backslash \{\vec u \in U\mid \exists \vec v \in U \text{ with } \vec u\subset \vec v  \}.$$}%
			\STATE{Draw $n = \left\lfloor \frac{N}{|U|} \right\rfloor$ $|\vec u|$-sparse features $\w\in \I_{\vec u}$ from $\rho_{\vec u}$ for every $\vec u\in U$ (keep already drawn features).}
	\end{algorithmic}
	\begin{tabular}{ l l l }
		\textbf{Output:} 
		$U $
	\end{tabular}
	\label{alg:2}
\end{algorithm}

\subsubsection*{A good choice for the index set $U$}
We want to discuss two procedures done in Algorithm~\ref{alg:2}: The first one is the for-loop, which is a similar proceeding as in Algorithm~\ref{alg:1}. Another possibility would be to do just one step of approximation and omit all indices $\vec u \in U$ with $S^{\text{MC}}_{\vec u,\text{var}}$ smaller than the threshold $\epsilon$ independent of the order $|\vec u|$. But it turned out in numerical tests, to be beneficial to use the loop, because otherwise the algorithm would not be able to detect the correct ANOVA terms. The variances of terms $f_{\vec u}$ of order less than $q$ are not estimated well enough when using ANOVA-truncated random Fourier features belonging to all $|U |= \sum_{i = 0}^q\binom{d}{i}$ terms of order smaller or equal $q$. \par

The second procedure is, that in line $8$ of Algorithm~\ref{alg:2} we shrink the index set $U$, such that there are no two sets contained, which 
are subsets $\vec u \subset \vec v$, which is necessary to receive an anti downward closed set $U$, see Definition~\ref{def:anti-dc}. This is the 
better choice, since the ANOVA terms $\vec v \subseteq \vec u$ are already contained in the sum $\sum_{\w\in \I_{\vec u}}a_{\w}\e^{\im \langle \vec x_{\vec u},\w_{\vec u}\rangle}$. This is made clearer by the following example. Assume a two-dimensional function $f = f_{\varnothing} +f_{\{1\}} +f_{\{2\}} + f_{\{1,2\}}$, which we approximate by the sum
$$f^\# = \sum_{\w\in \I_{\{1,2\}}\subset Q_{\{1,2\}}} a^\#_{\w} \e^{\im \langle \x_{\vec u},  \w_{\vec u}\rangle}.$$  
The approximation $f^\#$ has non-zero ANOVA terms $f^\#_{\varnothing}$, $f^\#_{\{1\}}$ and  $f^\#_{\{2\}}$, since the weak annihilating condition~\eqref{eq:zero_mean_gen} would require in the case $f^\# = f^\#_{\{1,2\}}$ that
\begin{align*}
0&=\int_{\R} f^\#_{\{1,2\}}(x_1,x_2) \mu_{\{1,2\}}(x_1,x_2)\, \dx x_1
=\sum_{\w\in \I_{\{1,2\}}} a^\#_{\w} \int_{\R} \e^{\im \left( x_1\omega_1+ x_2\omega_2\right)} \mu(x_1,x_2)\, \dx x_1\\
&=\sum_{\w\in \I_{\{1,2\}}} a^\#_{\w}\e^{\im  x_2\omega_2} \int_{\R} \e^{\im x_1\omega_1} \mu(x_1,x_2)\, \dx x_1,
\end{align*}
which can not be true for arbitrary density $\mu$, and $x_2$ and is also not demanded in the minimization of the RFF algorithms. This also applies to larger dimension $d$ and other index sets $\vec u$, that the ANOVA-terms $f^\#_{\vec v}$ are non-zero for $\vec v\subseteq \vec u$, if we draw random Fourier features from the set $Q_{\vec u}$. Thus, it is beneficial to use an anti-downward closed subset $U$. Numerical tests also indicate that the RFF algorithms yield better results in this case.

\subsubsection*{Complexity of the ANOVA-boosting algorithms}
The main focus of the ANOVA-boosting lies on the accuracy. Nevertheless, we will briefly sum up the complexity of our proposed algorithms.\par
In Algorithm~\ref{alg:1}, the calculation of the variances $\sigma^2_{\text{MC}}( f^\#_{\vec u})$ using \eqref{eq:gsi_gMC} is the crucial step and not fast. The complexity of this step is $\sum_{\vec u\in U} \O((|\vec u|+1)M^2n_{\vec u}) = \O((q+1)M^2N)$. If the number $M$ of samples is large, this complexity can be reduced to $M_{\text{val}}$ by calculating the variances just at $M_{\text{val}}$ validation points $\X_{\text{val}}$, since our aim is just to find the important variables and variable interactions by calculating the Sobol indices. The complexity of calculating the vector $\vec a^\# = \vec A^* (\vec A\vec A^*+ \lambda \vec I)^{-1}\vec f$ is the same as for the algorithm SHRIMP. \par
In Algorithm~\ref{alg:2}, the calculation of the Sobol indices 
has only complexity $\O(M)$. We solve the regularized system~\eqref{eq:lsqr_reg} 
with an iterative solver, where the difference compared to Algorithm~\ref{alg:1} 
is the construction of the matrix $\sqrt{ \hat{\vec W}}$, which has complexity 
$\mathcal O(M^2N|U|)$. Again, the complexity of solving the system iteratively 
remains the same.

\section{Theoretical analysis}\label{sec:theory}
In this section, we state error estimates for the approximation of a high-dimensional function $f$ of low effective dimension. Remember that the ANOVA index set $U \subset \mathcal P([d])$ should be chosen by an ANOVA-boosting algorithm. We propose to use an anti downward closed index set $U$, see Definition~\ref{def:anti-dc}, for the final random Fourier feature algorithm after the ANOVA-boosting.

For such an anti downward closed index-set $U$, point evaluations of $\hat T_{f_{\vec u}}(\w_{\vec u})$ are possible in $Q_{\vec u}$ as the following lemma shows.  
\begin{Lemma}\label{lem:new}
Let $f\in L_2(\R^d,\mu) $ and $U$ be an anti downward closed index set, which describes the ANOVA decomposition $f$ well by $f = f_{\underline U}$, where the downward closure $\underline U$ is defined in~\eqref{eq:under_U}. Let furthermore the function $f$ fulfill that
\begin{equation}\label{eq:assumption}
\int_{\R^{|\vec u|}} f(\x)\mu_{\vec u^c}(\x_{\vec u^c}) \d \x_{\vec u^c} \in L_2(\R^{|\vec u|}) \qquad \text{ for all }\vec u\in U,
\end{equation}
which means that the non-zero terms of highest order are in $L_2$ and is weaker than $f\in L_2(\R^d)$. 
Then, for $\vec u \in U$ we have that $\hat T_{\sum_{\vec v \subseteq \vec u }f_{\vec v}} \in L_2(Q_{\vec u})$
in the sense, that 
\begin{equation}\label{eq:hat T_u}\hat T_{\sum_{\vec v \subseteq \vec u }f_{\vec v}}(\w_{\vec u}) = \hat T_{g_\vec u}(\w_{\vec u}) + \sum_{\vec v\subset \vec u} \delta(\w_{\vec v^c}) h_{\vec v}(\w_{\vec v}),\end{equation}
where $\hat T_{g_\vec u}$ is a regular distribution on $\R^{|\vec u|}$.

\end{Lemma}
%
%
%
\begin{proof}
By integrating~\eqref{eq:decomp} with respect to $\mu_{\vec u^c}$ on $\R^{|\vec u^c|}$, we receive the equations
\begin{align*}
\int_{\R^{|\vec u^c|}} f(\x) \mu_{\vec u^c}(\x_{\vec u^c})\d\x_{\vec u^c} &= \sum_{\vec v\subseteq \vec u}f_{\vec v}(\x_{\vec v}) + \sum_{\stackrel{\vec v \neq \varnothing }{\vec v\cap\vec u\neq \varnothing}} \int_{\R^{|\vec v\cap \vec u^c|}} f_{\vec v}(\x_\vec v )\mu_{\vec v\cap \vec u^c}(\x_{\vec v\cap \vec u^c})\dx \x_{\vec v\cap \vec u^c},\\
\intertext{which is equivalent to}
 \sum_{\vec v\subseteq \vec u}f_{\vec v}(\x_{\vec v})&= \int_{\R^{|\vec u^c|}} f(\x) \mu_{\vec u^c}(\x_{\vec u^c})\d\x_{\vec u^c} - \sum_{\stackrel{\vec v \neq \varnothing }{\vec v\cap\vec u\neq \varnothing}} \int_{\R^{|\vec v\cap \vec u^c|}} f_{\vec v}(\x_\vec v )\mu_{\vec v\cap \vec u^c}(\x_{\vec v\cap \vec u^c})\dx \x_{\vec v\cap \vec u^c}.
\end{align*}
We will show in the following, that the function $\sum_{\vec v\subseteq \vec u} f_{\vec v}(\x_{\vec v})$ is a sum of lower dimensional terms, which do not depend on all variables $\x_{\vec u}$ and a function in $L_2(\R^{|\vec u|})$, such that~\eqref{eq:hat T_u} holds true. 
The first term is in $L_2(\R^{|\vec u|})$ due to the assumption~\eqref{eq:assumption}.
For the classical ANOVA decomposition the terms in the second sum are all zero. For the generalized ANOVA decomposition we receive a sum of lower dimensional terms except for the cases where $\vec v\supset \vec u$, where
\begin{align*}
\int_{\R^{|\vec v\cap \vec u^c|}} f_{\vec v}(\x_\vec v )\mu_{\vec v\cap \vec u^c}(\x_{\vec v\cap \vec u^c})\dx \x_{\vec v\cap \vec u^c} 
= \int_{\R^{|\vec v\backslash \vec u|}} f_{\vec v}(\x_\vec v )\mu_{\vec v\backslash \vec u}(\x_{\vec v\backslash \vec u})\dx \x_{\vec v\backslash \vec u}\eqqcolon h_{\vec u}(\x_{\vec u}). 
\end{align*}
Since we assume that $f = f_{\underline U}$, there do not exist non-zero ANOVA terms $f_\vec v$ with $\vec v\supset \vec u$, such that 
$h_{\vec u}(\x_{\vec u}) = 0$. Thus, the sum $\sum_{\vec v\subseteq \vec u}f_{\vec v}(\x_{\vec v})$ is a sum of a function in $L_2(\R^{|\vec u|})$ and lower dimensional terms, such that the Fourier transform has the form~\eqref{eq:hat T_u}. %
\end{proof}
With this lemma, for every $\vec u\in U$, also a sum of the form $\sum_{\vec v\in U_{\vec u}} f_{\vec v}$ with $\vec u\in U_{\vec u} \subset \mathcal P(\vec u) $ can be written as a sum of a function in $L_2(\R^{|\vec u|})$ and lower dimensional terms. Since the Fourier transform for functions in $L_2$ is defined only in $L_2$, we interpret ${g_{\vec u}}$ in the previous lemma always as the continuous representative, which we can evaluate pointwise.
The index set $\I$ collects all drawn random Fourier features by $\I = (\w_{k})_{k=1}^N$. The random Fourier features  supported on $\vec u$ are in the index set $\I_{\vec u}$. 

\par
We improve the analysis for the generalization error for the approximation by sparse random Fourier features. Following~\cite{Ha23,Xie22,HARFE23}, we go through the finite-sum approximation 
\begin{equation}\label{eq:f_star}
f^{\star}(\x) = \sum_{\vec u \in U} f^\star_{\vec u}(\x_{\vec u}) =\sum_{\vec u \in U}\sum_{\w\in \I_{\vec u}} a_{\w}^{\star} \, \e^{\im\langle \w_{\vec u},\x_{\vec u}\rangle},
\quad \quad a_{\w}^{\star} = \frac{\hat{T}_{\sum_{\vec v\in U_{\vec u}}f_{\vec v}}(\w_{\vec u})}{n_{\vec u} \,(2\pi)^d \rho_{\vec u}(\w_{\vec u})}, \quad \quad \sum_{\vec u\in U} U_{\vec u} = \underline U.
\end{equation}
This is motivated by the Monte Carlo approximation of the integral

\begin{align*}
\langle \hat T_f ,\phi\rangle
&= \sum_{\vec u\in \underline U} \int_{\R^{|\vec u|}} f_{\vec u} (\x_{\vec u}) \int_{\R^d} \phi(\w_{\vec u},\vec 0) \e^{-\im \langle \x_{\vec u},\w_{\vec u}\rangle} \d \w_{\vec u}  \d \x_{\vec u} 
= \sum_{\vec u\in \underline U}  \int_{\R^{|\vec u|}}  \hat T_{f_{\vec u}}(\w_{\vec u})   \phi(\w_{\vec u},\vec 0) \d \w_{\vec u}\\
&=\sum_{\vec u\in  U}  \int_{\R^{|\vec u|}}  \hat T_{\sum_{\vec v\in U_{\vec u}}}f_{\vec v}(\w_{\vec v})   \phi(\w_{\vec u},\vec 0) \d \w_{\vec u}.
\end{align*}
Note that $f^\star$ is not known in practice, because $\hat{T}_{f}$ is not known. Furthermore, $\E_{\w} \left[f^\star_{\vec u}(\vec x_{\vec u})\right] = \sum_{\vec v\in U_{\vec u}}f_{\vec v}(\vec x_{\vec v})$ for fixed $\vec x\in \R^d$ and $\E_{\w} \left[f^\star(\vec x)\right] = \sum_{\vec v\in \underline{U}}f_{\vec v}(\vec x_{\vec v})$.
\par

The function $f^\star$ allows the following error splitting,
\begin{equation}\label{eq:error_split}
\norm{f-f^\#}_{L_2(\R^d,\mu)} \leq \norm{f-\TT_{\color{blue}\underline U} f}_{L_2(\R^d,\mu)} + \norm{\TT_{\color{blue}\underline U} f - f^\star}_{L_2(\R^d,\mu)} + \norm{f^\star -f^\#}_{L_2(\R^d,\mu)} 
\end{equation}
The first error is bounded in Theorem~\ref{thm:error_f-Tqf}. Furthermore, it is known that in many real world problems the underlying function is of low effective dimension, see~\cite{CaMoOw97, DePeVo10, KuSlWaWo09}. \par

Our procedure is as follows:
\begin{itemize}
	\item In Lemma~\ref{lem:error_f-fstar} we generalize the error $\norm{\TT_{\color{blue}\underline U} f - f^\star}_{L_2(\R^d,\mu)}$ to the ANOVA setting by using ANOVA-sparse random feature instead of only $q$-sparse random Fourier features.
		\item We want to perform sensitivity analysis to calculate an index set $U$, which is adapted to the function $f$. In~\eqref{eq:error_gsi} we show that we have a good approximation of the variances, if the approximation error is small, so our procedure finds the important terms.
	\item Once we have fixed a good anti downward closed ANOVA index set $U$, we use SHRIMP or HARFE, so  the approximation bounds from there are applicable for the error $\norm{f^\star -f^\#}_{L_2(\R^d,\mu)}$. The used norm $\F(\rho)$ from~\eqref{eq:F_rho_old} can be replaced by our norm~\eqref{eq:normm}. Furthermore, there is also an error bound for $\norm{f^\star-f^\#}_{L_2(\R^d,\mu)}$ for the ridge regression in~\cite[Theorem 5.19.]{We24}.
\end{itemize}

Let us define for an anti downward closed index set $U$ the $\F(\rho)$-norm by
\begin{equation}\label{eq:normm}
\normm{f}_{\F(\rho)}^2 = \sum_{\vec u\in U}\frac{N}{n_{\vec u}\, (2\pi)^d} \, \left(\sup_{\w \in Q_{\vec u}}\frac{|\hat T_{\sum_{\vec v\in U_{\vec u}}f_{\vec v}}(\w_{\vec v})|}{\rho_{\vec u}(\w_{\vec u})}\right)^2,
\end{equation}
where $Q_{\vec u}$ is defined in \eqref{eq:Q_u}. Note that, the evaluation of $\hat T_{\sum_{\vec v\in U_{\vec u}}}(\w_{\vec u})$ for $\w \in Q_{\vec u}$ is possible due to Lemma~\ref{lem:new}.
For the approximant $f^\star$ we have the following.

\begin{Lemma}\label{lem:error_f-fstar}
Fix $\delta,\epsilon >0$. Let $f$ and $\mu$ fulfil the assumptions of Lemma~\ref{lem:new}. Consider the random feature approximation $f^\star$ from~\eqref{eq:f_star}. If the total number of features $N = \sum_{\vec u\in  U} n_{\vec u}$ satisfies the bound
$$N \geq \frac{1}{\epsilon^2}\left(1+ \sqrt{2\log(1/\delta)}\right)^2,$$
then with probability at least $1-\delta$ with respect to the draw of weights $\w_j$ the following holds
$$\norm{T_{\color{blue}\underline U}f-f^\star}_{L_2(\R^d,\mu)}\leq \epsilon \normm{f}_{\F(\rho)}.$$
\end{Lemma}
\begin{proof}
The proof follows similar arguments like~\cite[Lemma 1]{Ha23}, but applied to our setting of ANOVA truncated random Fourier features. 
Since the probability is zero, that for frequencies $\w_{\vec u}$ drawn from the density $\rho_{\vec u}$, that $\w_{\vec u}\notin Q_{\vec u}$, the coefficients $a_\w^\star$, defined in~\eqref{eq:f_star} are bounded by
\begin{equation}\label{eq:bound_c_w}
|a_{\w}^\star|\leq \frac{1}{n_{\vec u}\,(2\pi)^d} \, \sup_{\w \in Q_{\vec u}}
\left|\frac{\hat{T}_{\sum_{\vec v\in U_{\vec u}}}(\w_{\vec u})}{\rho_{\vec u}(\w_{\vec u})}\right|.
\end{equation}
Define the random variable 
$$v(\w_1, \ldots, \w_N) = \norm{f-f^\star}_{L_2(\R^d,\mu)} = \left(\int_{\R^d} |\E_{\w}(f^\star(\x)) - f^\star(\x)|^2\mu(\x)\dx \x\right)^{1/2}.$$
To apply McDiarmid’s inequality from Theorem~\ref{thm:McDiarmid}, we show that $v$ is stable to perturbation. In particular,
let $f^\star$ be the random feature approximation using random weights $(\w_1,\ldots,\w_k,\ldots,\w_N)$
and let $\tilde f^\star $ be the random feature approximation using random weights $(\w_1,\ldots,\tilde\w_k,\ldots,\w_N)$
with $\supp \tilde{\w}_k = \vec u$, then 
\begin{align*}
|v(\w_1,\ldots,\w_k,\ldots,\w_N) &-v(\w_1,\ldots,\tilde\w_k,\ldots,\w_N) |
\leq \norm{f^\star - \tilde f^\star}_{L_2(\R^d,\mu)}\\
&= \norm{a^\star_{\w_k}\e^{\im\langle (\w_k)_{\vec u}, \x_{\vec v}\rangle} - a^{\star}_{\tilde \w_k}\e^{\im\langle (\tilde\w_k)_{\vec u}, \x_{\vec u}\rangle}}_{L_2(\R^d,\mu)} \\
&\leq \frac{2}{n_{\vec u}} \left(\sup_{\w \in Q_{\vec u}}
\left|\frac{\hat{T}_{\sum_{\vec v\in U_{\vec u}}}(\w_{\vec u})}{\rho_{\vec u}(\w_{\vec u})}\right|\right) \eqqcolon \Delta_k.
\end{align*}
Summing over the $\Delta_k$ yields,
\begin{align*}
\sum_{k=1}^N \Delta_k^2 &\leq  \sum_{\vec u\in U}\frac{4}{n_{\vec u}} \left(\sup_{\w \in Q_{\vec u}} \left|\frac{\hat{T}_{\sum_{\vec v\in U_{\vec u}}}(\w_{\vec u})}{\rho_{\vec u}(\w_{\vec u})}\right|\right)^2\\
&\leq  \frac{4\normm{f}^2_{\F(\rho)}}{N} .
\end{align*}
To estimate the expectation of $v$, we bound the expectation of the second moment. By noting that the variance of an average of i.i.d. random variables is the average of the variances of each variable and by using the relation between the variance and the
un-centered second moment, we have that
\begin{align*}
\E_{\w} (v^2) &=\E_{\w}\norm{\E_{\w}(f^\star)-f^\star}_{L_2(\R^d,\mu)}\\
&= \E_{\w}\norm{\sum_{\vec u\in U} \sum_{\w \in \I_{\vec u}}a^\star_{\w}\e^{\im \langle \w_{\vec v},\x_{\vec v}\rangle}}^2_{L_2(\R^d,\mu)} - \norm{\E_{\w}\left(\sum_{\vec u\in U} a^\star_{\w}\e^{\im \langle \w_{\vec v},\x_{\vec v}\rangle}\right)}^2_{L_2(\R^d,\mu)} \\
 &\leq\sum_{\vec u\in U}\frac{1}{n_{\vec u}} \left(\sup_{\w \in Q_{\vec u}}\left| \frac{\hat{T}_{\sum_{\vec v\in U_{\vec u}}}(\w_{\vec u})}{\rho_{\vec u}(\w_{\vec u})}\right|\right)^2\leq
\frac{\normm{f}_{\F(\rho)}^2}{N}.
\end{align*}
By Jensen’s inequality, the expectation of $v$ is bounded by
\begin{align*}
\E_{\w} (v)\leq \left(\E_{\w} (v^2)\right)^{1/2}\leq \frac{\normm{f}_{\F(\rho)}}{\sqrt{N}}.
\end{align*}
Applying McDiarmids inequality from Theorem~\ref{thm:McDiarmid}, yields
$$\P\left(v \geq \frac{\normm{f}_{\F(\rho)}}{\sqrt{N}}\right)\leq \exp\left(-\frac{2t^2}{\sum_{k}\Delta_k^2}\right) =\exp\left(-\frac{2t^2N}{ 4\normm{f}^2_{\F(\rho)} }\right).  $$
Setting $t$ and $N$ to 
\begin{align*}
t &=\normm{f}_{\F(\rho)}\sqrt{\frac 2N \log(1/\delta)}\\
N &\geq \frac{1}{\epsilon^2}\left(1+ \sqrt{2\log(1/\delta)}\right)^2
\end{align*}
enforces that $v\leq \epsilon \normm{f}_{\F(\rho)} $ with probability at least $1-\delta$. This completes the proof.
\end{proof}

In the following Corollary we summarize all error estimates with the error splitting~\eqref{eq:error_split}. The main point is that the $\normm{f}_{\F(\rho)}$ depends on the ANOVA index set $U$, such that the error bound is smaller the better the index set $U$ is adapted to the function $f$.
\begin{Corollary}
Let $U\subset \mathcal P([d])$ with $q = \max_{\vec u\in U}|\vec u|$, the total number of drawn random Fourier features be $N$ and $M$ samples $\x\in \X$ distributed according to $\mu$. Then, with high probability with respect to the drawn samples and to the drawn features the following holds true.
If $f\in H^s_\mix(\R^d)$ and $\mu$ fulfils the conditions of Theorem~\ref{thm:error_f-Tqf},
\begin{align*}
\norm{f-f^\#}_{L_2(\R^d,\mu)} &\lesssim  c_{\mu,s}^q \norm{f}_{H^s_\mix(\R^d)} + \left(N^{-1/2}  +\sqrt{N^{-1} + M^{-1/2}}\right)\normm{f}_{\F(\rho)},
\intertext{and if the truncation $\TT_{\underline U}f$ is a good approximation of $f$ with error smaller than $\epsilon$, }
\norm{f-f^\#}_{L_2(\R^d,\mu)} &\lesssim  \epsilon +  \left(N^{-1/2}  +\sqrt{N^{-1} + M^{-1/2}}\right)\normm{f}_{\F(\rho)},
\end{align*}
where $\normm{f}_{\F(\rho)}$ is defined in~\eqref{eq:normm}.
\end{Corollary}
\begin{proof}
The errors in the splitting~\eqref{eq:error_split} are bounded by Theorem~\ref{thm:error_f-Tqf} for the error between $f$ and the ANOVA truncation and in Lemma~\ref{lem:error_f-fstar} for the error of the finite-sum approximation $f^\star$. Finally,~\cite[Theorem 5.19.]{We24} states a result for the generalization error of the random Fourier feature approximation.  
\end{proof}
\section{Numerical results}\label{sec:num}
We illustrate in this section our sensitivity analysis on test examples. First, we summarize in Section~\ref{sec:RFFs} two random Fourier feature algorithms from the literature, which we then use for the numerical experiments. We suggest to use one of our algorithms to find a good index set $U$ and then to adapt a random Fourier feature algorithm to ANOVA-sparse random Fourier features.\par    
We present numerical results of Algorithm~\ref{alg:1}, where we use independent input variables. Further, we illustrate in Section~\ref{sec:num2} the 
approximation procedure of Algorithm~\ref{alg:2} through several numerical applications.
The implementation of the ANOVA-boosting algorithms can be found at~\url{https://github.com/Laura1110/ANOVA_RFF}. 

\subsection{Algorithms for RFF}\label{sec:RFFs}
In the literature there are several algorithms for approximating high-dimensional sparse additive
functions. Here we will summarize two of them. Since we assume limited data availability,
we wish to have a sparse representation of the function $f$ by learning the coefficient
vector $\vec a$ with a sparsity constraint.
\begin{itemize}
	\item In \cite{Xie22} the non-linear \textbf{SHRIMP} algorithm was proposed. There the authors propose to use $q$-sparse frequencies, which means that $|\supp \w| = q$ for all random feature weights, where the non-zero components are sampled from the Gaussian distribution $\G(0,\tfrac{ 1}{\sqrt q})$. The algorithm begins with a strong over-parametrization $N \gg M$. The first solution vector $\vec a$ is calculated by
$$\vec a = \vec A^* (\vec A\vec A^*+ \lambda \vec I)^{-1}\vec f.$$
Then, using iterative magnitude pruning (IMP) and selecting the best model via a validation set, the  algorithm output is the solution vector $\vec a^\#$.\par
Iterative Magnitude Pruning is used for compressing over-parametrized neural networks. The IMP procedure prunes features on their magnitude and then retrains the pruned sub-network in each pruning iteration. In every iteration the pruning rate is $p<1$. After calculating the MSE error of a validation set in every iteration one can choose the final model by choosing the smallest validation error.\par
We want to generalize this algorithm to ANOVA-sparse random Fourier features as defined in Definition~\ref{def:ANOVA-features}. In fact this is a generalization of choosing a tensor product density or a $q$-sparse density to a density $\rho$, which can have an arbitrary ANOVA decomposition. 

\item In \cite{HARFE23} the authors solve the sparse random feature regression problem by a
greedy algorithm named hard-ridge random feature expansion (\textbf{HARFE}), which uses
a hard thresholding pursuit (HTP) like algorithm to solve the random feature ridge
regression problem. Specifically, they learn the vector $\vec a$ from the following minimization problem
$$\min_{\vec a} \norm{\vec A\vec a -\vec f}_2^2 + \lambda \norm{\vec a}^2_2 \quad \text{ sucht that } \vec a \text{ is s-sparse}.$$
The idea is to solve for
the coefficients using a much smaller number of model terms. The subset $S$ given by
the indices of the $s$ largest entries of one gradient descent step applied on the vector $\vec a$
is a good candidate for the support set of $\vec a$. The HTP algorithm iterates between these
two steps and leads to a stable and robust reconstruction of sparse vectors depending
on the restricted isometry property (RIP) constant of the matrix $\vec A$, which characterizes matrices which are nearly orthonormal, at least when operating on sparse vectors. 

\end{itemize}
As the numerical test for the RFF suggest, choosing the sparsity $q$ of the ANOVA terms equal to effective dimension of the function $f$, gives the lowest approximation errors. We suggest to first calculate a good ANOVA index set $U$ for the function $f$ with Algorithm~\ref{alg:1} or \ref{alg:2}, use this to draw ANOVA-sparse random Fourier features adapted to the function $f$ and apply afterwards an algorithm for sparse random features, for example SHRIMP or HARFE.\\
The best chances of 
improving the previous algorithms have functions with ANOVA terms of different orders: In this case drawing 
ANOVA random features adapted to the function $f$ will decrease the approximation error significantly. 

\subsection{Numerical results for independent input variables}
To test the performance of the ANOVA-boosting for independent input variables, we test Algorithm~\ref{alg:1} on synthetic functions:
\begin{align}\label{eq:fT1}
f_{T1}(\vec x) &= x_4^2 +x_2x_3 +x_1x_2 +x_4&\\
f_{T2}(\vec x) &= \sin(x_1) + 7 \sin^2(x_2) +0.1 x_3^4\, \sin(x_1) & \text{Ishigami function}\label{eq:fT2}\\
f_{T3}(\x) &= 10 \sin(\pi \, x_1 x_2) + 20 \, (x_3- \tfrac 12)^2 + 10 \, x_4 + 5 x_5&\text{Friedmann function}\label{eq:fT3}
\end{align}
We randomly draw $M$ points from $\G(0,I_d)$ (functions $f_{T1},f_{T2}$) or uniformly on $[0,1]^d$ (function $f_{T3}$) and 
use $N=5M$ random Fourier features in the initialization step. We additionally draw $M$ test samples from the same distribution 
to validate the approximation error using the MSE,
$$\text{MSE} = \frac{1}{|\X_{\text{test}}|} \sum_{\x \in \X_{\text{test}}} |f(\x) - f^\#(\x)|^2.$$
We compare the performance of the SHRIMP/HARFE algorithm and the ANOVA-boosted 
SHRIMP/ HARFE and denote the resulting algorithms as \textbf{ANOVA-S} and \textbf{ANOVA-H}, respectively. The random Fourier features were distributed i.i.d. according to Gaussian distribution $\G(\vec 0,\tfrac 1q \I_d)$ or 			
according to Cauchy distribution $\sim\prod_{i\in [d]}(1+w_i^2)^{-1}$ with variance $\sigma = \tfrac 1q$. In every case we did the approximation 
$10$ times and show the mean. In every case we have chosen the regularization parameter $\lambda = 10^{-6}$ and the cut-off parameter $\epsilon = 0.01$. 
In Table~\ref{tab:RMSEs} we summarize the results. Note that for the HARFE algorithm we used the exponential function $\e^{-\im \langle  \w,\x\rangle}$ in contrast to the authors of the numerical tests in~\cite{HARFE23}, who used the cosine function. Possibly further research could study why the numerical results are better with the cosine functions, despite the theoretical results are mostly stated for exponential functions.  \par

For summarizing the results, our procedure detects the important ANOVA terms, if enough samples are available. The random feature algorithms benefit from the first ANOVA-boosting in Algorithm~\ref{alg:1}, where we could improve the accuracy by factor up to $10^2$. But in any case the approximation error is smaller for the ANOVA-boosted algorithms or at least comparable. Furthermore, our procedure improves previous algorithms by being interpretable by showing clearly which ANOVA terms are zero and which input variables are necessary for the learned final model.\par
Despite constructed for dependent input variables, Algorithm~\ref{alg:2} can also be applied to independent input variables. For the example functions in this section, 
the approximated Sobol indices from Algorithm~\ref{alg:2} are slightly worse than from Algorithm \ref{alg:1}. Both algorithms find the correct index set $U$, but sometimes add indices $\vec u$, with $f_{\vec u}$, which can lead to a slightly bigger approximation error compared to the case where one detects the correct index set $U$. In applications, it may depend on the sampling points $\X$ and the drawn features and the ANOVA decomposition of the function $f$, which algorithm performs better, i.e.~which algorithm detects the index set $U$ with reasonable precision.
\npdecimalsign{.}
\nprounddigits{4}
\begin{table}[hbt]\centering
\begin{tabular}{ccccc|n{2}{4}n{2}{4}n{2}{4}n{2}{4}}
\toprule
		 function  & $d$ & $q$& $M $& $\rho$&  SHRIMP & \text{ANOVA-S}  & \text{HARFE} & \text{ANOVA-H}\\
\midrule
	\begin{filecontents*}{numerik_indep.csv}
d,q,M,distrrf, f, s, as, h, ah
5,2,300,1,1,1.8323502878321267e-6,1.4059979513843489e-6,0.3614776873132769,0.3004535196772401
10,3,500,1,1,0.00046390116526516393,1.4580898886162292e-6,2.377210641733861,0.2150763617602094
5,2,300,2,1,2.0259443999808708e-5,1.4160448549224847e-6,0.5826338180916677,0.9080995712657897
10,3,500,2,1,0.015208578897634963,1.7513890434922899e-6,3.4951862054782743,0.8933611224067304

5,2,500,1,2,0.008174431923501079,2.6587169530812363e-5,0.1377722059687619,0.6070939595170838
5,2,500,2,2,0.0024879723235188994,0.002809999959437781,0.5712130505451104,0.3840590862064409
10,2,1000,1,2,0.005453068439927534,2.6213039347336053e-5,0.6649966364743987,0.6909659900147486
10,2,1000,2,2,0.1212833975854477,0.006328984473856133,1.1395393037752477,0.8896460885180394

5,3,500,1,3,0.0032474465414000175,0.0026571328450201953,5.3181127070856204,0.47664332313986735
5,3,500,2,3,0.0001354703258473133,0.00017423821712845474,3.6377706217817667,0.908396176633493
10,3,200,1,3,0.3808205715442198,0.009828104272074485,5.817784121573397,2.054043588730657
10,3,200,2,3,0.5432534941358964,0.013327233217801946,3.9775562803468105,2.6889563564651846
\end{filecontents*}
\csvreader[head to column names,filter=\equal{\f}{1}\and \equal{\d}{5}\and \equal{\distrrf}{1}]{numerik_indep.csv}{}%
	{$f_{T1}$&$\d$ & $\q$ &$\M$	&$\rho_\G$ & \s &\as&\h&  \ah\\}%
	\csvreader[head to column names,filter=\equal{\f}{1}\and \equal{\d}{5}\and \equal{\distrrf}{2}]{numerik_indep.csv}{}%
	{& & &	&$\rho_{\mathcal C}$ & \s & \as&\h&  \ah\\\cmidrule{2-9}}%
	\csvreader[head to column names,filter=\equal{\f}{1}\and \equal{\d}{10}\and \equal{\distrrf}{1}]{numerik_indep.csv}{}%
	{&$\d$ & $\q$ &$\M$	&$\rho_\G$ & \s& \as&\h&  \ah\\}%
	\csvreader[head to column names,filter=\equal{\f}{1}\and \equal{\d}{10}\and \equal{\distrrf}{2}]{numerik_indep.csv}{}%
	{& & &	&$\rho_{\mathcal C}$ & \s & \as&\h&  \ah\\\midrule}%
	\csvreader[head to column names,filter=\equal{\f}{2}\and \equal{\d}{5}\and \equal{\distrrf}{1}]{numerik_indep.csv}{}%
	{$f_{T2}$&$\d$ & $\q$ &$\M$	&$\rho_\G$ & \s & \as&\h&  \ah\\}%
	\csvreader[head to column names,filter=\equal{\f}{2}\and \equal{\d}{5}\and \equal{\distrrf}{2}]{numerik_indep.csv}{}%
	{& & &	&$\rho_{\mathcal C}$ & \s & \as&\h&  \ah\\\cmidrule{2-9}}%
	\csvreader[head to column names,filter=\equal{\f}{2}\and \equal{\d}{10}\and \equal{\distrrf}{1}]{numerik_indep.csv}{}%
	{&$\d$ & $\q$ &$\M$	&$\rho_\G$ & \s & \as&\h&  \ah\\}%
	\csvreader[head to column names,filter=\equal{\f}{2}\and \equal{\d}{10}\and \equal{\distrrf}{2}]{numerik_indep.csv}{}%
	{& & &	&$\rho_{\mathcal C}$ & \s & \as&\h&  \ah\\\midrule}%
	\csvreader[head to column names,filter=\equal{\f}{3}\and \equal{\d}{5}\and \equal{\distrrf}{1}]{numerik_indep.csv}{}%
	{$f_{T3}$&$\d$ & $\q$ &$\M$	&$\rho_\G$ & \s & \as&\h&  \ah\\}%
	\csvreader[head to column names,filter=\equal{\f}{3}\and \equal{\d}{5}\and \equal{\distrrf}{2}]{numerik_indep.csv}{}%
	{& & &	&$\rho_{\mathcal C}$ & \s & \as&\h&  \ah\\\cmidrule{2-9}}%
	\csvreader[head to column names,filter=\equal{\f}{3}\and \equal{\d}{10}\and \equal{\distrrf}{1}]{numerik_indep.csv}{}%
	{&$\d$ & $\q$ &$\M$	&$\rho_\G$ & \s & \as&\h&  \ah\\}%
	\csvreader[head to column names,filter=\equal{\f}{3}\and \equal{\d}{10}\and \equal{\distrrf}{2}]{numerik_indep.csv}{}%
	{& & &	&$\rho_{\mathcal C}$ & \s & \as&\h&  \ah \\\bottomrule}%
  \end{tabular}
\caption{Approximation results: MSE on test data for different functions. We compare the performance of the SHRIMP and HARFE algorithm with the ANOVA-boosted algorithms using Algorithm~\ref{alg:1}. The random Fourier features were distributed i.i.d. according to $\rho_\G = \G(\vec 0,\tfrac 1q \I_d)$ or according to $\rho_{\mathcal C} \sim\prod_{i\in [d]}(1+w_i^2)^{-1}$ with variance $\sigma = \tfrac 1q$. In every case we did the approximation $10$ times and show the mean.}
\label{tab:RMSEs}
\end{table}

\subsection{Numerical results for dependent input variables}\label{sec:num2}
We illustrate that even for dependent input variables, it is possible to find non-zero ANOVA terms of the unknown function $f$.
In this example, the ANOVA-boosting method from Algorithm~\ref{alg:2} is tested on a slightly modified Friedmann function, which is
used as benchmark example for certain approximation techniques,
\begin{equation}\label{eq:friedmann}
 f(x_1,\ldots,x_{9}) = 10 \sin(0.1\,\pi \, x_1 x_2) + 20 \, (x_3- \tfrac 12)^2 + 10 \, x_4 + 5 x_5.
\end{equation}
In contrast to the literature we do not use uniform samples on $[0,1]^d$, but we use (partly) dependent Gaussian samples. Due to this sampling we changed the original function slightly to have comparable variances of the non-zero ANOVA terms.
Based on the example in~\cite{Ra142}, the samples $\x\in \X$ are Gaussian random vectors with mean $\E \x = \vec 0$ and with the covariance matrix being one of the following,
\begin{align}
\vec \Sigma_1 &= \vec I_{9}, && \quad\text{uncorrelated}\notag\\
\vec \Sigma_2 &= \tfrac 45 \vec I_{9} + \tfrac 15 \vec 1_{9\times 9},  && \quad\text{equally correlated}\label{eq:mu2}\\
\vec \Sigma_3 &= 
\vec I_{3} \otimes
\begin{pmatrix} 1 & -\tfrac 15 & \tfrac 25 \\
-\tfrac 15 & 1 & -\tfrac 45\\
\tfrac 25 & -\tfrac 45 & 1\\
\end{pmatrix},  && \quad\text{mixed correlated}\label{eq:mu3}
\end{align}
where $M = 500$.
Independent of the drawn samples $\x\in \X$, the function \eqref{eq:friedmann} has non-zero ANOVA terms only for the index set 
\begin{equation}\label{eq:U_fried}
U=\{\varnothing, \{1\}, \{2\},\{3\},\{4\},\{5\},\{1,2\}\}.
\end{equation} 
Furthermore, for independent input variables, $\vec \Sigma = \vec \Sigma_1$ the Sobol indices can be easily calculated analytically, which gives
\begin{align*}
S_{\{3\},\text{var}}&\approx 0.2788\quad& 
S_{\{4\},\text{var}}&\approx 0.3718\\
S_{\{5\},\text{var}}&\approx 0.0929 \quad&
S_{\{1,2\},\text{var}}&\approx 0.2564.
\end{align*}
We use Algorithm \ref{alg:2} to calculate the indices $S^{\text{MC}}_{\vec u,\text{var}}$, where we draw in total $N = 5000$ ANOVA-sparse random Fourier features with $q=2$ and variance $\frac 12$, and choose the regularization parameter $\lambda =1$. The results are plotted in Figure~\ref{fig:pie}.   
Note that the Sobol indices $S^{\text{MC}}_{\vec u,\text{var}}$ can be bigger than one, since for dependent input variables the variances of the ANOVA terms $\sigma^2(f_{\vec u})$ do not sum up to the variance of the function $\sigma^2(f)$. In Figure~\ref{fig:pie} we normalized the indices $S^{\text{MC}}_{\vec u,\text{var}}$ by the sum 
$$\sum_{\vec u \in \{\vec u\mid |\vec u|\leq 2\}} S^{\text{MC}}_{\vec u,\text{var}},$$
which has in total $\binom{9}{1} + \binom{9}{2} = 45 $ summands. It can be clearly seen, that in every case only the terms with non-zero variance in the index set $U$ from~\eqref{eq:U_fried} are significant. In contrast to the case $\vec \Sigma = \vec \Sigma_1$, for dependent variables ($\vec \Sigma_2$ and $\vec \Sigma_3$) the terms $f_1$ and $f_2$ are non-zero, which Algorithm~\ref{alg:2} finds. Using only $2$-sparse random Fourier features and the HARFE algorithm as proposed in~\cite{HARFE23}, leads to results shown in their Fig.4, which clearly blur the importance of the variables $x_1$ to $x_5$ in comparison to the other non-necessary variables, see also Example~\ref{ex:compare_harfe}. Furthermore, they only study the simpler case of independent input variables.

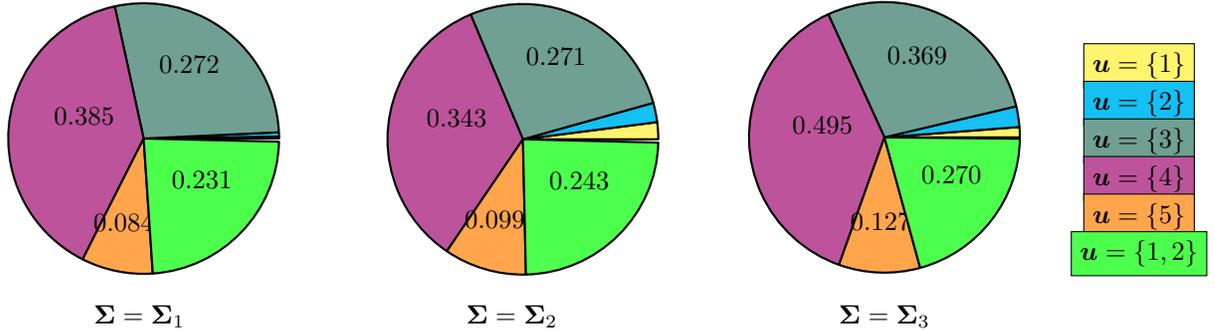
\begin{figure}[htp]
\begin{minipage}[t]{0.3\textwidth}
	\begin{tikzpicture}[scale=0.6]
	\begin{filecontents*}{dependent_fried.csv}
l,Sigmak,gsisA,gsisB,gsisC,gsisD,gsisE,gsisAB,rest
1.0,1,0.0022660288938620484,0.004971198049671522,0.2720236101152294,0.38452219781371966,0.08371513451264831,0.23140144641636679,0.003523035834933075
1.0,2,0.02039604282796453,0.023246702402864795,0.270579808301832,0.34312749586519764,0.09884443017497001,0.2434190869539866,0.0037746902831605755
1.0,3,0.016055091745271848,0.03242411385315624,0.36897737614680964,0.4951222426936563,0.12721467776811837,0.2702035030263621,0.001935747282738376
\end{filecontents*}
\csvreader[head to column names,filter=\equal{\l}{1.0}\and \equal{\Sigmak}{1}]{dependent_fried.csv}{}%
	{\pie[sum=\gsisA+  \gsisB+ \gsisC+ \gsisD+ \gsisE +\gsisAB + \rest, before number=\phantom ,text=inside,after number= ,color={ yellow!70, cyan!70, tuc!70, tucmath!70, orange!70, green!70}]{\gsisA/, \gsisB/ ,\gsisC/\num[round-precision=2]{\gsisC}, \gsisD/\num[round-precision=2]{\gsisD}, \gsisE/\num[round-precision=2]{\gsisE}, \gsisAB/\num[round-precision=2]{\gsisAB}}}
	\end{tikzpicture}
	\centering 
	$$\vec \Sigma=  \vec\Sigma_1$$
\end{minipage}
\begin{minipage}[t]{0.3\textwidth}
\begin{tikzpicture}[scale=0.6]
	\csvreader[head to column names,filter=\equal{\l}{1.0}\and \equal{\Sigmak}{2}]{dependent_fried.csv}{}%
	{\pie[sum=\gsisA+  \gsisB+ \gsisC+ \gsisD+ \gsisE +\gsisAB + \rest,before number=\phantom,text=inside,after number=,color={ yellow!70, cyan!70, tuc!70, tucmath!70, orange!70, green!70} ]{\gsisA/, \gsisB/ ,\gsisC/\num[round-precision=2]{\gsisC}, \gsisD/\num[round-precision=2]{\gsisD}, \gsisE/\num[round-precision=2]{\gsisE}, \gsisAB/\num[round-precision=2]{\gsisAB}}}
	\end{tikzpicture}
		\centering 
	$$\vec \Sigma=  \vec \Sigma_2$$
\end{minipage}
\begin{minipage}[t]{0.3\textwidth}
\begin{tikzpicture}[scale=0.6]
	\csvreader[head to column names,filter=\equal{\l}{1.0}\and \equal{\Sigmak}{3}]{dependent_fried.csv}{}%
	{\pie[sum=\gsisA+  \gsisB+ \gsisC+ \gsisD+ \gsisE +\gsisAB + \rest,before number=\phantom, text=inside, after number=,color={ yellow!70, cyan!70, tuc!70, tucmath!70, orange!70, green!70} ]{\gsisA/, \gsisB/ ,\gsisC/\num[round-precision=2]{\gsisC}, \gsisD/\num[round-precision=2]{\gsisD}, \gsisE/\num[round-precision=2]{\gsisE}, \gsisAB/\num[round-precision=2]{\gsisAB}}}
	\end{tikzpicture}
		\centering 
	$$\vec \Sigma=  \vec \Sigma_3$$
\end{minipage}
\begin{minipage}[t]{0.05\textwidth}
 \begin{tikzpicture}[scale=0.5]
        \node at (0,0) [rectangle,draw, fill =yellow!70]  {$\vec u = \{1\}$};
				 \node at (0,-1) [rectangle,draw, fill = cyan!70]  {$\vec u = \{2\}$};
				 \node at (0,-2) [rectangle,draw, fill =tuc!70]  {$\vec u = \{3\}$};
				 \node at (0,-3) [rectangle,draw, fill =tucmath!70]  {$\vec u = \{4\}$};
					\node at (0,-4) [rectangle,draw, fill =orange!70]  {$\vec u = \{5\}$};
					\node at (0,-5) [rectangle,draw, fill = green!70]  {$\vec u = \{1,2\}$};
    \end{tikzpicture}
\end{minipage}

\caption{The indices $S^{\text{MC}}_{\vec u,\text{var}}$ for Gaussian input samples with different covariance $\Sigma$. The pie charts are normalized to the sum of all indices $S^{\text{MC}}_{\vec u,\text{var}}$ for $|\vec u|\leq 2$, but the numbers represent the actual indices.  }
\label{fig:pie}
\end{figure}


To conclude our numerical section, we want to compare approximation results of random feature algorithms with and without ANOVA-boosting from Algorithm~\ref{alg:2}. 
We study the same test functions \eqref{eq:fT1} to \eqref{eq:fT3} as for the case of independent random variables.
To define the dependence among random variables, it is usual to use copula functions, \cite{JaDuHaRy09, Nelsen06}. Denote the cumulative distribution function of the samples by
$$R_{\X}(\x) = \int_{-\infty}^{\x} \mu(\vec t) \,\dx\vec t,$$
and $R_{1},\ldots, R_{d}$ are the marginal cumulative distribution functions of $x_1,\ldots, x_d$, i.e.,
$$R_i(x_i) = \int_{-\infty}^{x_i} \mu_i(t)\,\dx t.$$
Sklar’s theorem \cite{Skla59} is the building block of the theory of copulas. It states, that for continuous functions $R_{1},\ldots, R_{d}$ there exists a $d$-dimensional Copula $C$, such that for all $\x\in \R^d$, 
$$R_{\X}(\x) = C(R_1(x_1), \cdots, R_d(x_d)).$$
The copula $C$ contains all information on the dependence structure between components of $(x_1,\ldots, x_d)$, whereas the cumulative distribution functions $R_i$ contain all information on the marginal distribution of $x_i$. Especially Archimedean copulas are an important class of multivariate dependence models, since it is very easy to generate random numbers from them. Every Archimedean copula has the simple algebraic form
$$C(y_1,\ldots,y_d) = \psi\left(\psi^{-1}(y_1)+\ldots+ \psi^{-1}(y_d)\right),$$
where $\psi$ is the generator function of the copula.
We test our algorithm for the following well-known copulas with parameter $\theta$:
\begin{align*}
\psi(t) &= \frac{1}{\theta}\left( t^{-\theta}-1\right) &\theta> 0 &&\text{ Clayton copula,}\\
\psi(t) &= \left(-\ln t\right)^{\theta} &\theta\geq 1 &&\text{ Gumbel copula,}\\
\psi(t) &= -\ln \left(\frac{\exp{-\theta t}-1}{\exp{-\theta}-1} \right) &\theta> 0 &&\text{ Frank copula.}
\end{align*}
We use these copulas on the marginal distributions $\mathcal N(0,1)$ ($f_{T1}$), uniform on $[-\pi,\pi]$ ($f_{T2}$) or uniform on $[0,1]$ ($f_{T3}$).

We apply Algorithm~\ref{alg:2} for different settings of $d,M,q$ with fixed parameter $N = 5M$ and Gaussian random Fourier features drawn from $\G(\vec 0,\tfrac 1q \I_d)$. The used input parameters for our algorithm are summarized in Table~\ref{tab:RMSEs2} for the different settings. 
For a better comparison we use the same regularization parameter $\lambda =10^{-6}$ for the SHRIMP steps in both cases.
The resulting MSE on test data are summarized in Table~\ref{tab:RMSEs2}, we did the procedure $10$ times and show the mean. In any case, the ANOVA-boost leads to a clear improvement of the approximation results.\par

\npdecimalsign{.}
\nprounddigits{5}
\begin{table}[hbt]\centering
\begin{tabular}{cccccccc|n{2}{5}n{2}{5}}
\toprule
		 function  & $d$ & $q$& $M $& C&$\theta$& $\lambda$& $\epsilon$& SHRIMP & \text{ANOVA-S} \\
\midrule
	\begin{filecontents*}{numerik_dep.csv}
d, q, M ,c, theta, f,l, s, as
5,2,500,1,2,2,200,0.022868262629629023,0.00033748215126563853
5,2,500,2,2,2,100,0.00282473881683941,0.0004994330908089902
5,2,500,3,5,2,200,0.017317957162812696,0.0034057259431112424

10,2,500,1,2,2,200,0.31271382443331525,0.000751190185653811
10,2,500,2,2,2,100,0.4044195219821976,0.0007465653203607964
10,2,500,3,5,2,200,0.4799786267062009,0.005906373910820427

10,3,200,1,5,3,100,0.13448544242494898,0.009521361970217405
10,3,200,3,4,3,100,0.21361983593148554,0.034270705318858834
10,3,200,2,2,3,100,0.41547447994632974,0.015733654639104096

20,2,500,1,3,3,100,0.02844339923722385,0.01603279492586757
20,2,500,2,3,3,100,0.02651389506239759,0.004938646679092808
20,2,500,3,3,3,100,0.028703059132075347,0.0007092259053420436

10,3,500,2,2,1,100,8.550606050111328e-5,1.1563897985340141e-6
10,3,500,3,4,1,100,0.0002828491969566763,7.444584305327897e-7
10,3,500,1,3,1,100,4.478554214476854e-5,1.0467690962397391e-6

20,2,500,1,2,1,100,0.005043146251568537,1.804383639632081e-6
20,2,500,3,4,1,100,0.0039994130283263745,8.276972134063493e-7
20,2,500,2,1,1,100,0.006341259690392152,1.1523865388360087e-6
\end{filecontents*}
\csvreader[head to column names,filter=\equal{\f}{1}\and \equal{\d}{10} \and  \equal{\c}{1}]{numerik_dep.csv}{}%
	{$f_{T2}$&$\d$ & $\q$ &$\M$	&$\c $&$\theta$ &\l& $0.01$&\s &\as\\}%
		\csvreader[head to column names,filter=\equal{\f}{1}\and \equal{\d}{10} \and  \equal{\c}{2}]{numerik_dep.csv}{}%
	{& & &	&$\c $&$\theta$ &\l& $0.01$&\s &\as\\}%
			\csvreader[head to column names,filter=\equal{\f}{1}\and \equal{\d}{10} \and  \equal{\c}{3}]{numerik_dep.csv}{}%
	{& & &	&$\c $&$\theta$ &\l&$0.01$& \s &\as\\\cmidrule{2-10}}%
		\csvreader[head to column names,filter=\equal{\f}{1}\and \equal{\d}{20} \and  \equal{\c}{1}]{numerik_dep.csv}{}%
	{&$\d$ & $\q$ &$\M$	&$\c $&$\theta$ &\l&$0.01$& \s &\as\\}%
		\csvreader[head to column names,filter=\equal{\f}{1}\and \equal{\d}{20} \and  \equal{\c}{2}]{numerik_dep.csv}{}%
	{& & &	&$\c $&$\theta$ &\l& $0.01$&\s &\as\\}%
			\csvreader[head to column names,filter=\equal{\f}{1}\and \equal{\d}{20} \and  \equal{\c}{3}]{numerik_dep.csv}{}%
	{& & &	&$\c $&$\theta$ &\l& $0.01$&\s &\as\\\midrule}%
	\csvreader[head to column names,filter=\equal{\f}{2}\and \equal{\d}{5} \and  \equal{\c}{1}]{numerik_dep.csv}{}%
	{$f_{T2}$&$\d$ & $\q$ &$\M$	&$\c $&$\theta$ &\l& $0.05$&\s &\as\\}%
		\csvreader[head to column names,filter=\equal{\f}{2}\and \equal{\d}{5} \and  \equal{\c}{2}]{numerik_dep.csv}{}%
	{& & &	&$\c $&$\theta$ &\l& $0.05$&\s &\as\\}%
			\csvreader[head to column names,filter=\equal{\f}{2}\and \equal{\d}{5} \and  \equal{\c}{3}]{numerik_dep.csv}{}%
	{& & &	&$\c $&$\theta$ &\l&$0.05$& \s &\as\\\cmidrule{2-10}}%
		\csvreader[head to column names,filter=\equal{\f}{2}\and \equal{\d}{10} \and  \equal{\c}{1}]{numerik_dep.csv}{}%
	{&$\d$ & $\q$ &$\M$	&$\c $&$\theta$ &\l&$0.05$& \s &\as\\}%
		\csvreader[head to column names,filter=\equal{\f}{2}\and \equal{\d}{10} \and  \equal{\c}{2}]{numerik_dep.csv}{}%
	{& & &	&$\c $&$\theta$ &\l& $0.05$&\s &\as\\}%
			\csvreader[head to column names,filter=\equal{\f}{2}\and \equal{\d}{10} \and  \equal{\c}{3}]{numerik_dep.csv}{}%
	{& & &	&$\c $&$\theta$ &\l& $0.05$&\s &\as\\\midrule}%
			\csvreader[head to column names,filter=\equal{\f}{3}\and \equal{\d}{10} \and  \equal{\c}{1}]{numerik_dep.csv}{}%
	{$f_{T3}$&$\d$ & $\q$ &$\M$	&$\c $&$\theta$ &\l&$0.01$& \s &\as\\}%
		\csvreader[head to column names,filter=\equal{\f}{3}\and \equal{\d}{10} \and  \equal{\c}{2}]{numerik_dep.csv}{}%
	{& & &	&$\c $&$\theta$ &\l& $0.01$&\s &\as\\}%
			\csvreader[head to column names,filter=\equal{\f}{3}\and \equal{\d}{10} \and  \equal{\c}{3}]{numerik_dep.csv}{}%
	{& & &	&$\c $&$\theta$ &\l& $0.01$&\s &\as\\\cmidrule{2-10}}%
				\csvreader[head to column names,filter=\equal{\f}{3}\and \equal{\d}{20} \and  \equal{\c}{1}]{numerik_dep.csv}{}%
	{&$\d$ & $\q$ &$\M$	&$\c $&$\theta$ &\l&$0.01$& \s &\as\\}%
		\csvreader[head to column names,filter=\equal{\f}{3}\and \equal{\d}{20} \and  \equal{\c}{2}]{numerik_dep.csv}{}%
	{& & &	&$\c $&$\theta$ &\l& $0.01$&\s &\as\\}%
			\csvreader[head to column names,filter=\equal{\f}{3}\and \equal{\d}{20} \and  \equal{\c}{3}]{numerik_dep.csv}{}%
	{& & &	&$\c $&$\theta$ &\l& $0.01$&\s &\as\\\bottomrule}%
	
  \end{tabular}
\caption{Approximation results: MSE on test data for different settings. The column C belongs to the copula: $1,2,3$ corresponds to Clayton, Gumbel and Frank copula, respectively. We compare the performance of the SHRIMP and the ANOVA-boosted SHRIMP using Algorithm~\ref{alg:2} for dependent input variables. The random Fourier features were distributed according to $\rho_\G = \G(\vec 0,\tfrac 1q \I_d)$. In every case we did the approximation $10$ times and show the mean. }
\label{tab:RMSEs2}
\end{table}

During the numerical experiments we found that the more the input variables are related, the bigger should be the regularization parameter $\lambda$ in the minimization problem of Algorithm~\ref{alg:2}. 
Notice that, depending on the type of dependence, we do not obtain the same variances for the ANOVA terms, but the non-zero terms are in any case a subset of 
\begin{align}\label{eq:UT1}
U_{T1} &= \{\varnothing,\{1\},\{2\},\{3\},\{4\},\{1,2\},\{2,3\}\},\\
U_{T2} &= \{\varnothing,\{1\},\{2\},\{3\},\{1,3\},\{2,3\}\},\label{eq:UT2}\\
U_{T3} &= \{\varnothing,\{1\},\{2\},\{3\},\{4\},\{5\},\{1,2\}\}\label{eq:UT3},
\end{align}
for the three functions respectively. The variances of the terms $f_{\{3\}}$ and $f_{\{5\}}$ are relatively small for function $f_{T3}$. For that reason, we set $\epsilon = 0.01$ in this case.\par
The numerical results show that even for a small amount of samples in high dimensions our procedure is able to find the correct non-zero terms, which results in much smaller approximation error, compared to the plain SHRIMP algorithm~\cite{Xie22} with fixed effective dimension $q$.

\subsection*{Conclusion and outlook}

We propose a new method, ANOVA-boosting, which exploits sparse structure in the ANOVA terms of a function in a learning problem,
 which often occurs in many domains of interest. 
This method is a possible extension of random Fourier feature algorithms which finds the ANOVA terms with variance above some threshold before the actual approximation. Our algorithms are able to handle independent as well as dependent input variables.\par
Maybe it would be also possible to incorporate the ANOVA-boosting in every step of the iterative algorithm for sparse random Fourier features. 
Another possible future direction is the analysis of the impact of noise and the analysis of a good choice of the regularization parameter $\lambda$ in the boosting step as well as in the random feature algorithm.

\subsubsection*{Acknowledgement}
This work was supported by the BMBF grant 01$\mid$ S20053A (project SA$\ell$E).
Furthermore, we thank the anonymous reviewers
for providing helpful comments and suggestions to improve this article.
\subsubsection*{Declaration of interest}
The authors declare that they have no known competing financial interests or personal relationships
that could have appeared to influence the work reported in this article.
\bibliographystyle{abbrv}

\renewcommand{\d}{\, \mathrm{d}}   
\appendix
\section{Proofs and known inequalities}
\subsection*{Proofs of Section~\ref{sec:ANOVA_independent}}\label{sec:proofs}
First, we need an auxiliary result.
\begin{Lemma}\label{lem:cauchy_double}
Let $g\in L_2(\R^d)$ and $K$ be a symmetric and absolutely integrable kernel function. Then  
$$\left|\int_{\R^d}\int_{\R^d} g(\w) \overline{g(\vec v)} K(\w,\vec v) \d \w \d\vec v \right|\leq \int_{\R^d} |g(\w) |^2 k(\w) \d\w, $$
where $k(\w) = \int_{\R^d} |K(\w,\vec v)|\d \vec v$.
\end{Lemma}
\begin{proof}
According to \cite[Chapter 1]{St04}, the generalization of the Cauchy-Schwarz inequality for double integrals yields
 \begin{align*}
\left| \int_{\R^d}\int_{\R^d} g(\w) \overline{g(\vec v)} K(\w,\vec v) \d\w\d\vec v \right|
&\leq \int_{\R^d}\int_{\R^d} \left|g(\w) \overline{g(\vec v)} K(\w,\vec v) \right|\d\w\d\vec v \\  
&\leq
\left(\int_{\R^d}\int_{\R^d} |g(\w)|^2 \left|K(\w,\vec v) \right|\d\w \d\vec v \right)^{1/2}
\left(\int_{\R^d}\int_{\R^d} |\overline{g(\vec v)}|^2\left|K(\w,\vec v) \right|\d \w \d\vec v \right)^{1/2}\\
 &= \left(\int_{\R^d} |g(\w)|^2 \int_{\R^d}\left|K(\w,\vec v) \right| \d\vec v \d \w  \right)^{1/2}
\left(\int_{\R^d} |g(\vec v)|^2 \int_{\R^d} \left|K(\w,\vec v) \right| \d \w \d\vec v \right)^{1/2}\\
 &= \int_{\R^d} |g(\w)|^2k(\w)  \d\w  .
\end{align*}
This finishes the proof.
\end{proof}
\subsubsection*{Proof of Lemma~\ref{lem:sigma_bound}}
\begin{proof}
Lemma~\ref{lem:ANOVA-terms_f} describes the ANOVA terms of the function $f$. In order to calculate the variance of the ANOVA terms we have
\begin{align*}
\sigma^2(f_{\vec u})
&= \int_{\R^{|\vec u|}}\left|\frac{1}{(2\pi)^d}\int_{\R^d} \hat{f}(\w) E(\x,\w,\mu, \vec u) \d \w\right|^2 \mu_{\vec u}(\vec x_{\vec u})\d \vec x_{\vec u}\\
&=\frac{1}{(2\pi)^{2d}}\int_{\R^{|\vec u|}}\int_{\R^d} \int_{\R^d}\hat{f}(\w)\overline{\hat{f}(\vec v)} E(\x,\w,\mu, \vec u)   \overline{E(\x,\vec v,\mu, \vec u)}\d \w \d \vec v \mu_{\vec u}(\vec x_{\vec u})\d \vec x_{\vec u}\\
&=\frac{1}{(2\pi)^{2d}}\int_{\R^{d}}\int_{\R^d} \int_{\R^{|\vec u|}}\hat{f}(\w)\hat{f}(-\vec v) E(\x,\w,\mu, \vec u)  E(\x,-\vec v,\mu, \vec u)\,\mu_{\vec u}(\vec x_{\vec u})\d \vec x_{\vec u} \d \w \d \vec v \\
&=\frac{1}{(2\pi)^{2d}}\int_{\R^d}\int_{\R^d} \hat{f}(\w)\hat{f}(-\vec v) 
\prod_{i\in \vec u} \left( \hat{\mu_i}(-\omega_i+v_i) -\hat{\mu_i}(-\omega_i)\hat{\mu_i}(v_i) \right)\prod_{i\in \vec u^c} \hat{\mu_i}(-\omega_i)\hat{\mu_i}(v_i)\, \d \w \d \vec v
\end{align*}
To apply Lemma~\ref{lem:cauchy_double} we choose 
$K(\w,\vec v) = \prod_{i\in \vec u} \left( \hat{\mu_i}(-\omega_i+v_i) -\hat{\mu_i}(-\omega_i)\hat{\mu_i}(v_i) \right)\prod_{i\in \vec u^c} \hat{\mu_i}(-\omega_i)\hat{\mu_i}(v_i)$
with 
\begin{align*}
k(\w) &= \int_{\R^d} \left|\prod_{i\in \vec u} \left( \hat{\mu_i}(-\omega_i+v_i) -\hat{\mu_i}(-\omega_i)\hat{\mu_i}(v_i) \right)\prod_{i\in \vec u^c} \hat{\mu_i}(-\omega_i)\hat{\mu_i}(v_i)\right| \d\vec v\\
&= \prod_{i\in [d]} \norm{\hat{\mu_i}}_{L_1(\R)} \prod_{i\in \vec  u}\left|1-\hat{\mu_i}(-\omega_i)\right| \prod_{i\in \vec  u^c}\left|\hat{\mu_i}(-\omega_i)\right|\\
&= \norm{\hat{\mu}}_{L_1(\R^d)}|E(\vec 0,\w,\mu,\vec u)|.
\qedhere 
\end{align*}
\end{proof}

\subsubsection*{Proof of Theorem~\ref{thm:error_f-Tqf}}
\begin{proof}
For this proof we introduce the notation
\begin{equation}\label{eq:def_A}
A(\w,d,q) \coloneqq \left( \prod_{i\in [d]} (1+|\omega_i|^2)^{-s}  \sum_{|\vec u|> q}|E(\vec 0,\w,\mu,\vec u)| \right).
\end{equation}
First, note that since every measure $\mu_i$ is symmetric, $\norm{\hat{\mu_i}}_{L_\infty(\R)} \leq \norm{\mu_i}_{L_1(\R)} = 1$ and $-1\leq\hat{\mu_i}(-\omega_i)\leq 1$. We start with applying Lemma~\ref{lem:sigma_bound}.
\begin{align}
&\norm{f-\TT_{q}f}_{L_2(\R^d,\mu)}^2 = \sum_{|\vec u|> q} \sigma^2(f_{\vec u}) 
\leq \frac{\norm{\hat{\mu}}_{L_1(\R^d)}}{(2\pi)^{2d}} \sum_{|\vec u|\geq q}\int_{\R^d}| \hat{f}(\w)|^2
 \prod_{i\in \vec  u}\left|1-\hat{\mu_i}(-\omega_i)\right| \prod_{i\in \vec  u^c}\left|\hat{\mu_i}(-\omega_i)\right| \d \w \notag\\
&\quad=\frac{\norm{\hat{\mu}}_{L_1(\R^d)}}{(2\pi)^{2d}}\int_{\R^d}| \hat{f}(\w)|^2\prod_{i\in [d]}\frac{(1+|\omega_i|^2)^s}{(1+|\omega_i|^2)^s}
 \sum_{|\vec u|> q}\prod_{i\in \vec  u}|1-\hat{\mu_i}(-\omega_i)| \prod_{i\in \vec  u^c}|\hat{\mu_i}(-\omega_i)| \d \w \notag\\
&\quad=\frac{\norm{\hat{\mu}}_{L_1(\R^d)}}{(2\pi)^{2d}}\max_{\w\in \R^d}\left( \prod_{i\in [d]} (1+|\omega_i|^2)^{-s}  \sum_{|\vec u|> q}\prod_{i\in \vec  u}|1-\hat{\mu_i}(-\omega_i)| \prod_{i\in \vec  u^c}|\hat{\mu_i}(-\omega_i) |\right)
\int_{\R^d}| \hat{f}(\w)|^2\prod_{i\in [d]}(1+|\omega_i|^2)^s\d \w \notag
  \\
	&\quad=\frac{\norm{\hat{\mu}}_{L_1(\R^d)}}{(2\pi)^{2d}}\norm{f}^2_{H^s_{\mix}(\R^d)} \,\max_{\w\in \R^d}\left( \prod_{i\in [d]} (1+|\omega_i|^2)^{-s}  \sum_{|\vec u|> q}\prod_{i\in \vec  u}|1-\hat{\mu_i}(-\omega_i)| \prod_{i\in \vec  u^c}|\hat{\mu_i}(-\omega_i)| \right) \notag \\
	&\quad=\frac{\norm{\hat{\mu}}_{L_1(\R^d)}}{(2\pi)^{2d}}\norm{f}^2_{H^s_{\mix}(\R^d)} \,\max_{\w\in \R^d}A(\w,d,q).
\end{align}
Let us have a closer look at the involved term $A(\w,d,q)$. Let $\vec v$ be the support of the $\w$, which attains the maximum. Since $\hat{\mu_i}(-\omega_i)= 0 $ for every $i\in \vec v^c$, which means that $|\vec v^c|>q$ and $|\vec v|<d-q$, we have
\begin{align*}
\max_{\w\in \R^d}A(\w,d,q) 
&=  \prod_{i\in [d]} (1+|\omega_i|^2)^{-s}  \sum_{|\vec u|> q,\vec v^c\subseteq \vec u }\prod_{i\in \vec  u}|1-\hat{\mu_i}(-\omega_i)| \prod_{i\in \vec  u^c}|\hat{\mu_i}(-\omega_i)| \\
&\leq  \sum_{\vec u' \supseteq \vec v^c } c_{\mu,s}^{|\vec v^c|} \left(\prod_{i\in \vec  u'\backslash\vec v^c}\frac{|1-\hat{\mu_i}(-\omega_i)|}{(1+|\omega_i|^2)^{s}} \prod_{i\in \vec  u'^c}\frac{|\hat{\mu_i}(-\omega_i)|}{ (1+|\omega_i|^2)^{s}} \right)\\
&=  c_{\mu,s}^{|\vec v^c|} \sum_{\vec u' \subseteq\vec v }  \left(\prod_{i\in \vec  u'}\frac{|1-\hat{\mu_i}(-\omega_i)|}{(1+|\omega_i|^2)^{s}} \prod_{i\in \vec  u'^c}\frac{|\hat{\mu_i}(-\omega_i)|}{ (1+|\omega_i|^2)^{s}} \right)\\
&=  c_{\mu,s}^{|\vec v^c|} \prod_{i\in \vec v}\frac{|1-\hat{\mu_i}(-\omega_i)| + |\hat{\mu_i}(-\omega_i)| }{(1+|\omega_i|^2)^{s}}\\
&\leq c_{\mu,s}^{q+1}.
\end{align*}
The last inequality follows by either demanding a symmetric measure $\mu$ with positive Fourier transform or the condition~\eqref{eq:condition}.
The equality~\eqref{eq:constant2} follows by the fact that the maximum of $g(\omega_i) \coloneqq (1+|\omega_i|^2)^{-s}  \left(1- \hat{\mu_i}(-\omega_i)\right)$ is attained where $g'(\omega_i) = 0$, i.e.,
\begin{align*}
0&\stackrel{!}{=} \frac{-2\omega_i s }{(1+|\omega_i|^2)^{s+1}} \left(1- \hat{\mu_i}(\omega_i)\right) - \frac{\mu_i'(\omega_i)}{(1+|\omega_i|^2)^{s}}, \\
\hat{\mu}_i(\omega_i) &= 1+ \hat{\mu}_i'(\omega_i)\,\frac{1+\omega_i^2}{2\omega_i s}.
\end{align*}
Inserting this into~$g(\omega_i)$ yields
\begin{align*}
c_{\mu,s}
&= \sup_{\omega_i\in \R}   (1+|\omega_i|^2)^{-s}  \left(- \hat{\mu}_i'(\omega_i)\,\frac{1+\omega_i^2}{2\omega_i s} \right)\\
&=\sup_{\omega_i\in \R} \frac{1}{2\omega_i s\,(1+|\omega_i|^2)^{s-1}}.
\end{align*}

\end{proof}

\subsection{McDiarmids inequality}
\begin{Theorem}[Mc Diarmids inequality]\label{thm:McDiarmid}
Let a function $v\colon \X_1 \times \X_2 \times \cdots \times \X_N \rightarrow \R$  satisfy the bounded differences property, i.e.~for all $k\in [N]$, and all $x_1\in \X_1,\ldots, x_N\in \X_N,$
$$\sup_{x_k'\in \X_k}|v(x_1,\ldots, x_{k-1},x_k, x_{k+1},\ldots, x_N) - v(x_1,\ldots, x_{k-1},x'_k, x_{k+1},\ldots, x_N)|\leq \Delta_k.$$
Consider independent random variables $X_1,X_2, \ldots, X_N$ where $X_k\in \X_k$ for all $k$. Then, for any $\eps> 0$
$$\P\left(v(X_1,X_2, \ldots, X_N) - \E[v(X_1,X_2, \ldots, X_N)] >\eps\right)\leq  \exp\left(-\frac{2\eps^2}{\sum_{k=1}^N \Delta_k^2}\right).$$
\end{Theorem}

\end{document}